\documentclass[letterpaper,conference,compsoc,10pt]{IEEEtran}
\IEEEoverridecommandlockouts

\usepackage{blindtext}
\usepackage{fancyhdr}

\usepackage[square, comma, sort&compress, numbers]{natbib}
\let\cite\citep
\setcitestyle{comma}
\setcitestyle{open={[},close={]},sep={,}}

\usepackage{amsmath,amssymb,amsfonts}
\usepackage{algorithmic}
\usepackage{graphicx}
\usepackage{textcomp}
\usepackage{setspace}
\usepackage{booktabs}
\usepackage{url}
\usepackage{enumitem}

\usepackage[table,xcdraw]{xcolor}

\usepackage{multirow}
\usepackage{caption}
\usepackage{balance}
\usepackage{subcaption}
\usepackage{array}

\usepackage{diagbox}
\usepackage[normalem]{ulem}
\useunder{\uline}{\ul}{}

\usepackage[colorlinks,
            linkcolor=red,
            anchorcolor=red,
            citecolor=green]{hyperref} %

\usepackage{amsthm} 
\newtheorem{theorem}{Theorem} 
\newtheorem{definition}{Definition} 
\newtheorem{lemma}{Lemma} 
\newtheorem{proposition}{Proposition}
\newtheorem{hypothesis}{Hypothesis}

\usepackage[capitalize]{cleveref}
\Crefname{section}{Sec.}{Secs.}
\Crefname{table}{Tab.}{Tabs.}
\Crefname{equation}{Eq.}{Eqs.}
\Crefname{figure}{Fig.}{Figs.}
\Crefname{lemma}{Lemma}{Lemmas}
\Crefname{theorem}{Theorem}{Theorems}
\Crefname{definition}{Definition}{Definitions}
\Crefname{hypothesis}{Hypothesis}{Hypothesises}

\def\wrt{\textit{w.r.t.} }
\def\ie{\textit{i.e.}}
\def\eg{\textit{e.g.}}
\def\etal{\textit{et al.}}

\def\CO{\textit{CO}}
\def\CM{\textit{CM}}
\def\MU{\textit{MU}}
\def\LS{\textit{LS}}
\def\SAM{\textit{SAM}}
\def\IR{\textit{IR}}
\def\JR{\textit{JR}}
\def\ER{\textit{ER}}
\def\AT{\textit{AT}}
\def\ST{\textit{ST}}

\def\security19{Demontis \textit{et al}. \cite{security19}}
\def\SP22{Mao \textit{et al}. \cite{maoP2022TransferAttacks2022}}
\def\nips21{Yang  \etal ~\cite{yangNeurIPS2021TRS2021}}

\newcommand{\black}[1]{\textcolor{black}{#1}}

\usepackage{lipsum}

\renewcommand{\figurename}{Figure}

\def\BibTeX{{\rm B\kern-.05em{\sc i\kern-.025em b}\kern-.08em
    T\kern-.1667em\lower.7ex\hbox{E}\kern-.125emX}}

\begin{document}
\pagenumbering{arabic} 

\title{Why Does Little Robustness Help?  A Further Step Towards Understanding Adversarial Transferability}

\author{
\IEEEauthorblockN{ 
Yechao Zhang\IEEEauthorrefmark{1}\textsuperscript{1,2,3,4},
Shengshan Hu\IEEEauthorrefmark{1}\textsuperscript{1,2,3,4}, 
Leo Yu Zhang\IEEEauthorrefmark{2},
Junyu Shi\IEEEauthorrefmark{1}\textsuperscript{1,2,3,4}\\
Minghui Li\IEEEauthorrefmark{3},
Xiaogeng Liu\IEEEauthorrefmark{1}\textsuperscript{1,2,3,4}, 
Wei Wan\IEEEauthorrefmark{1}\textsuperscript{1,2,3,4}, 
and Hai Jin\IEEEauthorrefmark{4}\textsuperscript{1,2,5}
}
\IEEEauthorblockA{
\IEEEauthorrefmark{1}
 School of Cyber Science and Engineering,
 Huazhong University of Science and Technology\\
}

\IEEEauthorblockA{
\IEEEauthorrefmark{3}
School of Software Engineering, 
Huazhong University of Science and Technology \
}

\IEEEauthorblockA{
\IEEEauthorrefmark{2}
School of Information and Communication Technology,
Griffith University\
}

\IEEEauthorblockA{
\IEEEauthorrefmark{4}
School of Computer Science and Technology, 
Huazhong University of Science and Technology \\
}
\textsuperscript{1}National Engineering Research Center for Big Data Technology and System\\
\textsuperscript{2}Services Computing Technology and System Lab \
\textsuperscript{3}Hubei Key Laboratory of Distributed System Security \\
\textsuperscript{4}Hubei Engineering Research Center on Big Data Security \ \textsuperscript{5}Cluster and Grid Computing Lab\\
{\tt \footnotesize \{ycz,hushengshan,shijunyu220,minghuili,liuxiaogeng,weiwan\_0303,hjin\}@hust.edu.cn, leo.zhang@griffith.edu.au
}

}

\maketitle
\thispagestyle{fancy}
\pagestyle{plain}

\begin{abstract}
Adversarial examples for \textit{deep neural networks} (DNNs) have been shown to be transferable: examples that successfully fool one white-box surrogate model can also deceive other black-box models with different architectures.
Although a bunch of empirical studies have provided guidance on generating highly transferable adversarial examples, many of these findings fail to be well explained and even lead to confusing or inconsistent advice for practical use.  

In this paper, we take a further step towards understanding adversarial transferability, with a particular focus on surrogate aspects. 
Starting from the intriguing ``little robustness'' phenomenon, where models adversarially trained with mildly perturbed adversarial samples can serve as better surrogates for transfer attacks, we attribute it to a trade-off between two dominant factors: model smoothness and gradient similarity. 
Our research focuses on their joint effects on transferability, rather than demonstrating the separate relationships alone.
Through a combination of theoretical and empirical analyses, we hypothesize that the data distribution shift induced by off-manifold samples in adversarial training  is the reason that impairs gradient similarity. 

Building on these  insights, we further explore the impacts of prevalent data augmentation and gradient regularization  on transferability and analyze how the  trade-off  manifest  in  various training methods, thus building a  comprehensive blueprint for the regulation mechanisms behind  transferability. 
Finally, we provide a general route for constructing superior surrogates to boost transferability, which optimizes both model smoothness and gradient similarity simultaneously, \eg, the combination of input gradient regularization and \textit{sharpness-aware minimization} (\textit{SAM}), validated by extensive experiments.  
In summary, we call for attention to the united impacts of these two factors for launching effective transfer attacks,  rather than optimizing one while ignoring the other, and emphasize the crucial role of  manipulating surrogate models.

\end{abstract}

\section{Introduction} \label{sec:intro}

Adversarial transferability is an intriguing property of adversarial examples (AEs), where examples crafted against a surrogate  DNN could fool other DNNs as well.
Various techniques~\cite{Naseer_2021_ICCV,Dong_2019_CVPR,NEURIPS2020_00e26af6,Dong_2018_CVPR,Li_2020_CVPR,zhao2021on,xie2019improving,zhou2023advencoder,hu2022protecting} 
have been proposed for the generation process of AEs to increase transferability\footnote{\black{We refer to ``adversarial transferability'' as ``transferability''}.}, such as integrating momentum to stabilize the update direction~\cite{Dong_2018_CVPR} or applying transformations at each iteration to create diverse input patterns~\cite{xie2019improving}.
At a high level, all  these research endeavors to find transferable AEs under given surrogate models by proposing complex tricks to be incorporated into the generation pipeline one after another. However, this approach could result in non-trivial computational expenses and low scalability~\cite{qin2022boosting}.

With a different perspective, another line of works~\cite{pmlr-v180-gubri22a,gubriLGV2022,Wu2020Skip,li2020learning} start to examine the role of surrogate models. 
One useful observation is that attacking an ensemble of surrogate models with different architectures can  obtain a  more general update direction for AEs \cite{liu2017delving,chen2023rethinking}.
A recent study \cite{gubriLGV2022} also proposed fine-tuning a well-trained surrogate to obtain a set of intermediate models that can be used for  an ensemble. 
However, it remains unclear which type of surrogate performs the best and should be included in the ensemble. 
Our work taps into this line and tries to manipulate surrogates to launch stronger transfer attacks.

Meanwhile, \security19 started to explain the transferability property and showed that model complexity and gradient alignment  negatively and positively correlate with transferability, respectively. 
Lately, Yang  \etal ~\cite{yangNeurIPS2021TRS2021} established a  transferability lower bound and   theoretically connect transferability to two key factors: model smoothness  and gradient similarity. 
Model smoothness captures the general invariance of the gradient \wrt input features for a given model, while gradient similarity refers to the alignment of the gradient direction between surrogate and target models. 
However, there is still much to be understood and explored regarding these factors. For instance, it is unclear which factor plays a more important role in regulating transferability, how existing empirical findings for improving transferability affect these factors, and how to generally optimize them both simultaneously.
Such a complex situation makes it challenging to completely understand the role of surrogates.
Consequently, there is currently no consensus on how to construct better surrogates to achieve higher adversarial transferability. 
In support of this, a recent work~\cite{maoP2022TransferAttacks2022}, after performing a large-scale empirical study in real-world settings, concluded that surrogate-level factors that affect the transferability highly chaotically interact and there is no effective solution to obtain a good surrogate except through trial and error.

\begin{table*}[htbp!]
\caption{\black{The overview of our interactions with literature in the field. }}
\resizebox{\textwidth}{!}
{
\begin{tabular}{p{8.2cm}p{9.2cm}p{1.6cm}}
\toprule
\textbf{Existing conclusions and viewpoints}     & \textbf{Our observations and inferences} & \textbf{Relation}                                                                                                                                                             \\ \midrule
Stronger regularized (smoother)  models provide  better surrogates on average \cite{security19}.                                    & (1)   \AT ~with large budget yields smoother models that  degrade transferability. In \cref{sec:adv-transfer}.  (2) Stronger regularizations cannot always outperform  less smooth solutions like \SAM. In \cref{trade-off-GR}.  &  Partly \ conflicting \\
\hline
\AT ~and data augmentation do not show strong correlations to  transfer attacks in the ``real-world" environment \cite{maoP2022TransferAttacks2022}.                            & (1)   \AT ~with small budget benefits transfer attack while large budget hinders it.   (2) Data augmentation generally impairs transfer attacks, especially for   stronger augmentations.    In \cref{analysis}, \textbf{Q6}.            &  Conflicting    \\
\hline
Surrogate models with better generalization performance could result in more transferable AEs  \cite{wang2022role}.                                                 & 
Data augmentations that yield surrogates with the best generalization perform the  worst in transfer attacks. In \cref{analysis}, \textbf{Q4}.    &  Conflicting                                                                                      \\
\hline
Attacking multiple surrogates from a sufficiently large geometry vicinity (LGV)  benefits transferability \cite{gubriLGV2022}.                            & 
Attacking  multiple surrogates from arbitrary LGV of a single  superior surrogate may degrade transferability. 
In \cref{surrogate-independent}.   & Partly \ conflicting       \\
\hline
Regularizing pressure transfers from the weight space to the input space.   \cite{Dherin2022why}.   & 
This transfer effect exists, yet is  marginal and  unstable. In \cref{sec:GR-smoothness}.                                                      & Partly \ conflicting              \\
\hline
The poor transferability of ViT is because existing attacks are not strong enough to fully exploit its potential \cite{naseer2022on}.                                             & 
The transferability of ViT may have been restrained by its default training paradigm.   In \cref{analysis}, \textbf{Q5} & Parallel \\                                                                                                                   
\hline
Model complexity (the number of local optima in loss surface) correlates with transferability \cite{security19}. & 
A smoother model is expected to have less and  wider local optima in a finite space. In \cref{analysis}, \textbf{Q3}  & Causal \\
\hline
AEs lie off the underlying manifold of clean data \cite{gilmer2018adversarial}. & Adversarial training causes data distribution shift induced by off-manifold AEs, thus impairing gradient similarity.  & Dependent \\
\hline
(1) Attacking an ensemble of surrogates in the distribution found by \textit{Bayes} learning  improves the transferability \cite{li2023making}. 
(2) \SAM  ~can be seen as a relaxation of Bayes  \cite{mollenhoff2023sam}. 
& 
\SAM ~yields general input gradient alignment towards every training solution. Attacking \SAM ~solution significantly improves transferability.
In \cref{analysis}, \textbf{Q7}. & Matching  \\ 
\bottomrule
\end{tabular}
}
\label{tab:contrary-conclusions}
\end{table*}

In light of this, we aim to deepen our understanding of adversarial transferability and provide concrete guidelines for constructing better surrogates for improving transferability. 
Specifically, we offer the following contributions:

\textbf{Gaining a deeper understanding of the ``little robustness" phenomenon.} 
We begin by exploring the intriguing ``little robustness'' phenomenon~\cite{springer2021little}, where 
models produced by small-budget adversarial training (\AT)
serve as better surrogates than standard ones (see \cref{fig:adv-transfer}). 
Through tailored definitions of gradient similarity and model smoothness, we recognize this phenomenon  comes along with an inherent trade-off effect between these two factors on transferability.
Thus we attribute the “little robustness” appearance to the persistent improvement of model smoothness and deterioration of gradient similarity.
Further, we identify the observed gradient dissimilarity in \AT ~as the result of the data distribution shift caused by off-manifold samples.


\textbf{Investigating the impact of data augmentation on transferability.}
\black{
Beyond \AT, we explore more general distribution shift cases to confirm our hypothesis on data distribution shift impairing gradient similarity.}
Specifically, we exploit four popular data augmentation methods, Mixup (\MU), Cutmix (\CM), Cutout (\CO), and Label smoothing (\LS) to investigate how these different kinds of distribution shifts affect the gradient similarity.
Extensive experiments reveal they unanimously impair gradient  similarity in different levels, and the degradation generally aligns with the augmentation magnitude. 
Moreover, we find different augmentations affect smoothness differently. 
While \LS ~benefits both datasets,  other methods mostly downgrade smoothness in CIFAR-10  and show chaotic performance in ImageNette. 
Despite their different complex trade-off effects, they consistently yield worse transferability.

\textbf{Investigating the impact of gradient regularization on transferability.} 
As the concept of gradient similarity is inherently tangled up with unknown target models, it is difficult to directly optimize it in the real scenario. Therefore, we explore solutions that presumably feature better smoothness without explicitly changing the data distribution, \ie, gradient regularization.
Concretely, we explore the impact of four gradient regularizations  on model smoothness:  input gradient regularization (\textit{IR}) and  input Jacobian regularization (\textit{JR}) in the input space, and    explicit gradient regularization (\textit{ER}) and  sharpness-aware minimization (\textit{SAM}) in the weight space.
Our extensive experiments show that gradient regularizations  universally improve model smoothness, with regularizations in the input space leading to faster and more stable improvement than  regularizations in the weight space. 
However, while input space regularizations (\JR, \IR) produce  better smoothness, they do not necessarily outperform  weight space regularizations like \SAM ~in terms of transferability. 
We \black{also find  this aligns with the fact that} they weigh differently on  gradient similarity, which  plays another crucial role in the overall trade-off.


\textbf{Proposing a general route for generating surrogates.} 
Considering the practical black-box scenario, we propose a general route for the adversary to obtain surrogates with both superior smoothness and  similarity. 
\black{This involves first assessing the smoothness of different training mechanisms and comparing their similarities, then developing generally effective strategies by combining the best elements of each.}
Our experimental results show that input regularization and \textit{SAM} excel at promoting  smoothness and  similarity,  respectively, as well as other complementary properties. 
Thus, we propose \textit{SAM}\&\textit{JR} and \textit{SAM}\&\textit{IR} to obtain better surrogates.
Extensive results show our methods significantly outperform existing solutions in boosting transferability. 
Moreover, we  showcase that the design is a plug-and-play method that can be directly integrated  with surrogate-independent methods to further improve transferability. 
The transfer attack results on  three commercial Machine-Learning-as-a-Service (MLaaS) platforms also demonstrate the effectiveness of our method against real-world deployed DNNs.


In summary, we study  the complex trade-off  between model smoothness and gradient similarity (see \cref{fig:factors}) under various circumstances and identify a general alignment between  the trade-off and transferability.
We call for attention that, for effective surrogates, one  should handle these two factors well.
 \black{Throughout our study, we  also discover a range of observations correlating existing conclusions and viewpoints in the field, \cref{tab:contrary-conclusions} provides a glance at them and report the relations. 
 }

\begin{figure}[!htbp]
 \setlength{\belowcaptionskip}{-0.05cm}  
\centering
\includegraphics[width=0.45\textwidth]{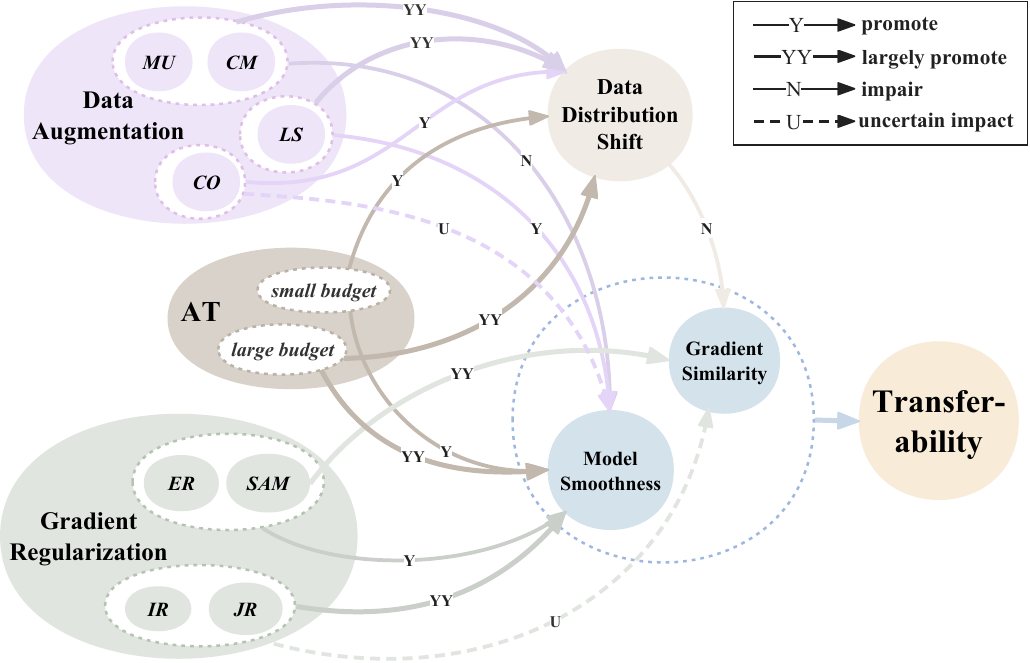}
\caption{A complete overview of the relationship between factors regulating adversarial transferability of our study.}
\label{fig:factors}
\end{figure}


\begin{figure*}[!htbp] 
\centering
  \centering
    \subcaptionbox{CIFAR-10}{\includegraphics[width=0.475 \textwidth]{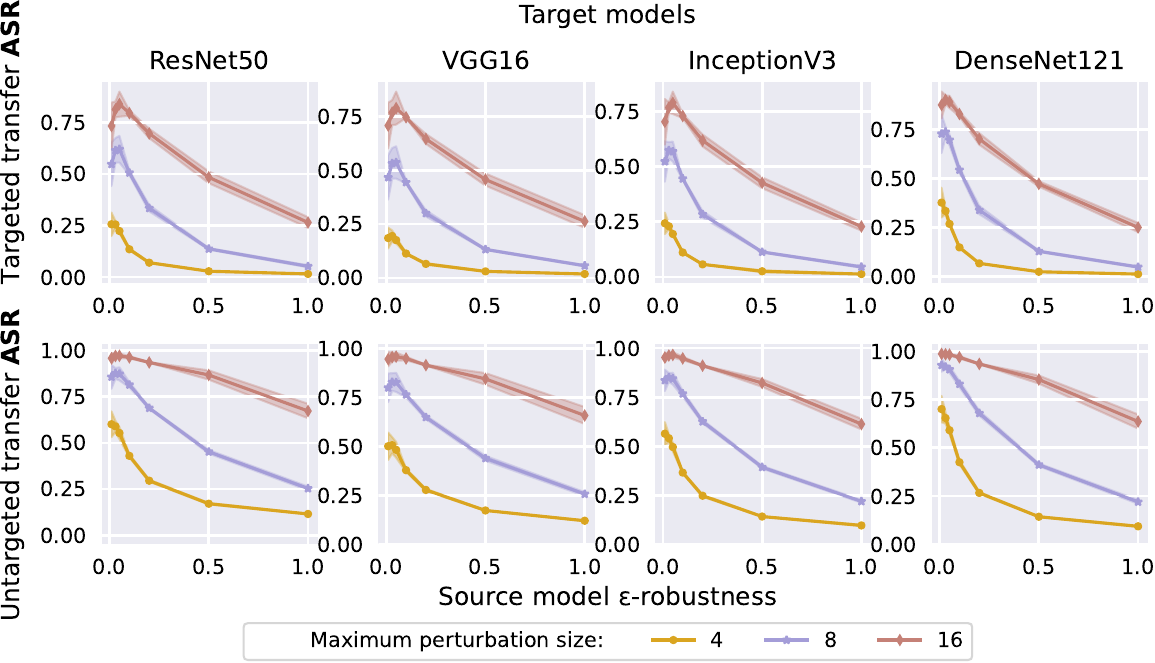}}
    \hfill
      \subcaptionbox{ImageNette}{\includegraphics[width=0.475\textwidth]{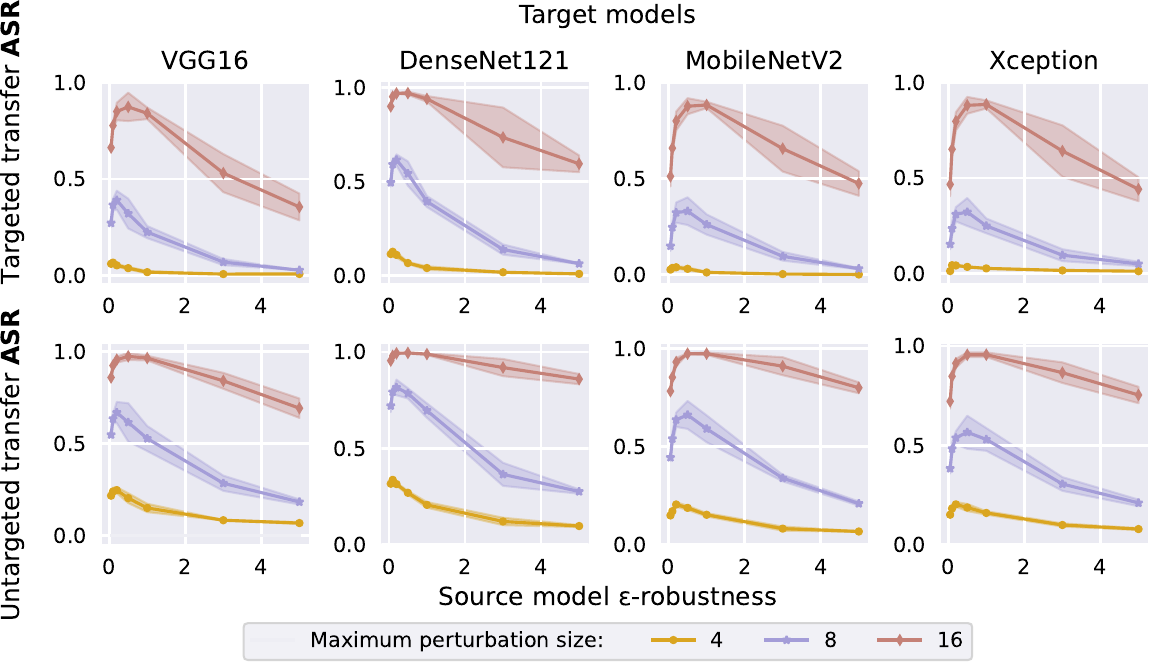}}
      \caption{Transfer \textit{attack success rates} (\textbf{ASRs}) against adversarially trained CIFAR-10   and ImageNette classifiers. \black{We plot the average results of different surrogates obtained by 3 random seeds and the corresponding error bars for each $\epsilon$.}}
\label{fig:adv-transfer}
\vspace{-2mm}
\end{figure*}

\section{Explaining Little Robustness}\label{sec:littlerob}
\black{Understanding adversarial transferability is still a challenging task, and many empirical observations in the literature are rather perplexing.}
As a typical case, \textit{``little robustness''} indicates that a model adversarially trained with small perturbation budgets improves the transferability~\cite{springer2021little,springer2021uncovering}. 
It remains unclear why small-budget adversarially trained models can serve as better surrogates, whereas large-budget ones do not exhibit this benefit.
In this section, we tailor the definitions of gradient similarity and model smoothness, two recently proposed concepts that are believed to be important for  transferability~\cite{yangNeurIPS2021TRS2021}, to analyze this phenomenon. 


\noindent\textbf{Notations.} We denote the input space as $\mathcal{X}= \mathbb{R}^d$, and the output space as $\mathcal{Y}=\mathbb{R}^m$.
We also consider a standard classification training dataset $\mathbf{S}=\left\{\left(x_1, y_1\right), \cdots,\left(x_n, y_n\right)\right\}$, where $x_i \in \mathcal{M} \subset \mathcal{X} $ and $y_i \in \mathcal{L}=\{(0,1)^m\} \subset \mathcal{Y}$ are  \textit{identically and independently distributed} (i.i.d.) drawn from the normal data distribution $\mathcal{D} \in \mathcal{P}_{\mathcal{X \times Y}}$. $\mathcal{M}$ denotes the normal  (image)  feature manifold,  $\mathcal{D}$ is supported on $(\mathcal{M}, \mathcal{L})$ and $\mathcal{P}_{\mathcal{X \times Y}}$ is the set of distributions on $\mathcal{X \times Y}$. 
We further denote the marginal distribution on $\mathcal{X}$ and $\mathcal{Y}$ as $\mathcal{P}_{\mathcal{X}}$ and $\mathcal{P}_{\mathcal{Y}}$, respectively. 
The classification model can be viewed as a mapping function $ \mathcal{F}: \mathcal{X} \rightarrow \mathcal{L}$. Specifically,  given any input $x \in \mathcal{X}$, $\mathcal{F}$ will find an optimal match $\mathcal{F}(x)=\arg \min _{y \in \mathcal{L}} \ell_{\mathcal{F}}(x, y)$ with the hard labels $\mathcal{L}$, where $\ell_{\mathcal{F}}: \mathcal{X} \times \mathcal{Y} \rightarrow \mathbb{R}_{+}$ can be decomposed to a training loss $\ell$\footnote{We use the cross-entropy loss by default in the paper.} and the network's logits output $f(\cdot)$: $\ell_{\mathcal{F}}(x, y):=\ell(f(x), y),(x, y) \in(\mathcal{X}, \mathcal{Y}).$

In this paper, we denote the  surrogate model and target model as $\mathcal{F}$ and $\mathcal{G}$, respectively. 
Generally, both $\mathcal{F}$ and $\mathcal{G}$ can be trained within the data distribution $\mathcal{D}$ using $\mathbf{S}$, or a close yet different distribution $\mathcal{D}'$ obtained by a specific augmentation mechanism  on $\mathbf{S}$. 
Throughout this paper, we use $\mathcal{D}'$  to denote a  data distribution different from $\mathcal{D}$. 
Since this paper focuses on studying the transferability from the surrogate perspective, we assume $\mathcal{G}$ is also well-trained on $\mathcal{D}$ by default. 
We defer the analysis and  results when  $\mathcal{G}$ is trained on  
various $\mathcal{D}'$ to \cref{analysis}, \textbf{Q2}, where the results also support the conclusions drawn from the case of $\mathcal{D}$.

\subsection{Transferability Circuit of Adversarial Training}
\label{sec:adv-transfer}
Adversarial training~\cite{madry2018towards} (\AT) \black{uses adversarial examples generated on $\mathbf{S}$ as augmented data and minimizes} the following adversarial loss: 
\begin{equation} \label{loss_adv}
	L_{adv} = \frac{1}{\left\|\mathbf{S}\right\|}\sum\limits_{i=1}^{\left\|\mathbf{S}\right\|} \max\limits_{\left\|\delta\right\|_{2} < \epsilon} \ell(f(x_i+\delta),y_i),
\end{equation}
where $\delta \in \mathbb{R}^d$ denotes the adversarial perturbation and $\epsilon$ is the adversarial budget.
As evidenced by our experiments in~\cref{fig:adv-transfer}, the attack success rates (ASRs) rise with \black{relatively} small-budget adversarial trained models and start to decrease with the high-budget ones. 
This ``transferability circuit'' is particularly intriguing.
We thus revisit the recently proposed transferability lower bound~\cite{yangNeurIPS2021TRS2021} and  utilize it as a tool  to  analyze  the underlying reason.

\subsection{Lower Bound of Transferability}\label{Transfer-Analysis}
We first re-define model smoothness, gradient similarity, and the transferability between two models, based on which a lower bound of transferability can be obtained. 

\begin{definition}[\textbf{Model smoothness}]
Given a model $\mathcal{F}$ and a data distribution $\mathcal{D}$, the smoothness of $\mathcal{F}$ on $\mathcal{D}$ is defined as $\sigma_{\mathcal{F}, \mathcal{D}} = \mathbb{E}_{(x,y) \sim \mathcal{D}}  [\sigma (\nabla_{x}^2 \ell_{\mathcal{F}}(x, y))]$, where $\sigma (\cdot)$ denotes the \textbf{dominant eigenvalue}, and $\nabla_{x}^2 \ell_{\mathcal{F}}(x, y) $ is the Hessian matrix \wrt $x$. 
We abbreviate $\sigma_{\mathcal{F}, \mathcal{D}}$ as $\sigma_{\mathcal{F}}$ for simplicity and  define the upper smoothness as $\overline{\sigma}_{\mathcal{F}} = \sup\limits_{(x, y) \sim \mathcal{D}} \sigma (\nabla_{x}^2 \ell_{\mathcal{F}}(x, y))$, where $\sup$ is the supremum function.
\end{definition}  

A smoother model on $\mathcal{D}$ is featured with smaller $\sigma_{\mathcal{F}}$ and   $\overline{\sigma}_{\mathcal{F}}$, 
\black{indicating more invariance of loss gradient}.
Different from its original definition  where the global Lipschitz constant $\beta_{\mathcal{F}}$ is used~\cite{yangNeurIPS2021TRS2021}, we use this curvature metric (\ie, the dominant eigenvalue of the Hessian) to define the model smoothness in a local manner. 
With this modified definition, we can  theoretically and empirically quantify the smoothness of  loss surface on $\mathcal{F}$ \black{for a given data distribution} $\mathcal{D}$. 
Moreover, this  provides  a tighter bound for  the model’s loss function gradient, namely, an explicit relation $\overline{\sigma}_{\mathcal{F}} \black{\leq} \beta_{\mathcal{F}}$ holds. 
Although the true distribution $\mathcal{D}$ is unknown, one can still sample a set of data from $\mathbf{S}$ and use the empirical mean as an  approximation to evaluate the model smoothness. 
Note that our definition is correlated to the model complexity concept introduced in \security19 (see \cref{analysis}, \textbf{Q3}).

\begin{definition}[\textbf{Gradient similarity}] \label{Definition-GS}
For two models $\mathcal{F}$ and $\mathcal{G}$ with the loss functions $\ell_\mathcal{F}$ and $\ell_\mathcal{G}$, the gradient similarity over  $(x, y)$ is defined as
$\mathcal{S}\left(\ell_{\mathcal{F}}, \ell_{\mathcal{G}}, x, y\right)= \frac{{\nabla_x \ell_{\mathcal{F}}(x, y)} \cdot \nabla_x \ell_{\mathcal{G}}(x, y) }{\left\|\nabla_x \ell_{\mathcal{F}}(x, y)\right\|_2 \cdot\left\|\nabla_x \ell_{\mathcal{G}}(x, y)\right\|_2}.$
\end{definition}

Given a distribution $\mathcal{D}$, we can further define the infimum of loss gradient similarity on $\mathcal{D}$ as $\underline{\mathcal{S}}\black{_\mathcal{D}}\left( \ell_\mathcal{F}, \ell_{\mathcal{G}}\right)=\inf \limits_{(x, y)\sim \mathcal{D}} \mathcal{S}\left(\ell_{\mathcal{F}}, \ell_{\mathcal{G}}, x, y\right) $. 
Similar to \security19, we  define the expected loss gradient similarity as $\tilde{\mathcal{S}}\black{_\mathcal{D}}(\ell_\mathcal{F}, \ell_\mathcal{G}) = \mathbb{E}_{(x, y) \sim \mathcal{D}}[\mathcal{S}\left(\ell_{\mathcal{F}}, \ell_{\mathcal{G}}, x, y\right)]$, which can capture the similarity between two models on $\mathcal{D}$.
\black{Naturally, a larger gradient similarity indicates a more general alignment between the adversarial directions of two given models.} 
We can sample a set of data from $\mathbf{S}$ to evaluate this alignment between two models on $\mathcal{D}$.

\begin{definition}[\textbf{Transferability}] 
\label{def:transferability}
Given a normal sample $(x,y)\in\mathcal{D}$, and a perturbed version $x+ \delta$   crafted against a surrogate model $\mathcal{F}$. 
The transferability $T_r$ between $\mathcal{F}$ and a target model $\mathcal{G}$ is defined as 
$T_{r}=\mathbb{I}\left[\mathcal{F}(x)=\mathcal{G}(x)=y \wedge \mathcal{F}\left(x+ \delta\right) \neq y \wedge \mathcal{G}\left(x+ \delta\right) \neq y\right]$, where $\mathbb{I}$ denotes the indicator function.
\end{definition}

Here we define transferability  the same way as~\cite{yangNeurIPS2021TRS2021} for untargeted attacks at the instance level. 
A successful untargeted transfer attack requires  both the surrogate  and the target models to give  correct predictions for the unperturbed input and  incorrect predictions for the perturbed one. 
The targeted version is similar. 
\black{Note that we abuse the notations $\delta$ and $\epsilon$ in \cref{def:transferability,theorem1}, while the notations in \cref{loss_adv} have similar meanings under different contexts.}


\begin{theorem}[\textbf{Lower bound of transferability}]
\label{theorem1}
Given any sample $(x, y) \in \mathcal{D}$, let $x+\delta$ denote a perturbed version of $x$ with fooling probability $\operatorname{Pr}(\mathcal{F}(x+\delta) \neq y ) \geq (1-\alpha)$ and perturbation budget $\|\delta\|_{2} \leq \epsilon$.  
Then the transferability $\operatorname{Pr}\left(T_{r}(\mathcal{F}, \mathcal{G}, x, y)=1\right)$ between surrogate model $\mathcal{F}$ and target model $\mathcal{G}$ can be lower bounded by \\
$\operatorname{Pr}\left(T_{r}(\mathcal{F}, \mathcal{G}, x, y)=1\right) \geq(1-\alpha)-\left(\gamma_{\mathcal{F}}+\gamma_{\mathcal{G}}\right)-\frac{\epsilon(1+\alpha)-c_{\mathcal{F}}(1-\alpha)}{\epsilon - c_{\mathcal{G}}}-\frac{\epsilon(1-\alpha)}{\epsilon-c_{\mathcal{G}}} \sqrt{2-2 \underline{\mathcal{S}}_\mathcal{D}\left( \ell_\mathcal{F}, \ell_{\mathcal{G}}\right)}$,
where
\begin{equation}\label{bound-eq}
\begin{aligned}
c_{\mathcal{F}}&=\min_{(x,y) \sim \mathcal{D}} \frac{\min\limits_{y' \in \mathcal{L}:y' \neq y } \ell_{\mathcal{F}}\left(x+\delta, y'\right)-\ell_{\mathcal{F}}(x, y)-\overline{\sigma}_\mathcal{F} \epsilon^{2} / 2}{\left\|\nabla_{x} \ell_{\mathcal{F}}(x, y)\right\|_{2}},\\
c_{\mathcal{G}}&=\max _{(x,y)  \sim  \mathcal{D}} \frac{\min\limits_{y' \in \mathcal{L}:y' \neq y} \ell_{\mathcal{G}}\left(x+\delta, y'\right)-\ell_{\mathcal{G}}(x, y)+\overline{\sigma}_\mathcal{G}\epsilon^{2} / 2}{\left\|\nabla_{x} \ell_{\mathcal{G}}(x, y)\right\|_{2}},
\end{aligned}
\end{equation}
and $\epsilon$ be sufficiently large such that $\epsilon > c_{\mathcal{G}}$ and $\epsilon > c_{\mathcal{F}}$.  $\gamma_{\mathcal{F}}$ and $\gamma_{\mathcal{G}}$ are the natural inaccuracies of $\mathcal{F}$ and $\mathcal{G}$ on $\mathcal{D}$, $\overline{\sigma}_\mathcal{F}$ and $\overline{\sigma}_\mathcal{G}$ are the upper-bound of smoothness w.r.t. $\mathcal{F}$ and $\mathcal{G}$.
\end{theorem}
Note that we deliver the lower bound of transferability for untargeted attacks as an example, and the targeted attack follows a similar form. The conclusions for the targeted attack agree with the untargeted case, as suggested in~\cite{yangNeurIPS2021TRS2021},  so we omit it for space. 
The proof is provided in the Appendix. Although the theorem is initially for $l_2$-norm perturbations, it can be generalized to $l_\infty$-norm as demonstrated in~\cite{yangNeurIPS2021TRS2021}.

\noindent\textbf{Implications.} 
Different from~\cite{yangNeurIPS2021TRS2021}, which seeks to suppress the transferability from both surrogate and target sides, this paper focuses on the surrogate aspect alone to increase transferability. 
Therefore, we next briefly analyze how \black{all the terms} relating to $\mathcal{F}$ influence the lower bound. 

\begin{itemize}[leftmargin=11pt]
\item \textbf{fooling probability $1-\alpha$}: \black{this term captures the likelihood of the  surrogate $\mathcal{F}$ being fooled by the adversarial instance $x + \delta$.
Different AEs generation strategies  result in  different fooling probabilities. 
Intuitively, if $x + \delta$ is unable to fool $\mathcal{F}$, we would expect it cannot fool $\mathcal{G}$ as well. 
The theorem also suggests effective attack strategies  to increase the lower bound. 
Therefore we default to using a strong baseline attack, projected gradient descent (PGD) (\cref{tab:cifar10-ASR,tab:imagenette-ASR}). 
We also report the results of AutoAttack (a gradient-based and gradient-free ensemble attack) in the Appendix (\cref{tab:autoattack}) to ensure the consistency of conclusions.}
\item \textbf{natural inaccuracy $\gamma_{\mathcal{F}}$}: the negative effect of natural inaccuracy on the lower bound of transferability is obvious. A model with low accuracy on $\mathcal{D}$  is certainly undesirable for adversarial attacks. 
Consequently, we minimize  this term by training models with sufficient epochs for CIFAR-10 and fine-tuning pre-trained ImageNet classifiers for  ImageNette.

\begin{figure}[!t] 
\setlength{\belowcaptionskip}{-0.1cm}   
  \centering
    \subcaptionbox{ResNet18 on CIFAR-10}{\includegraphics[width=0.225\textwidth]{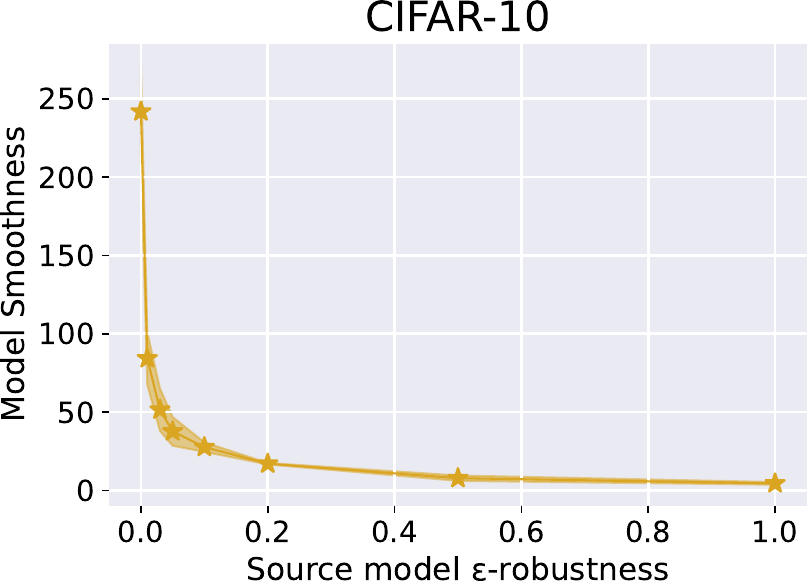}}
    \hfill
      \subcaptionbox{ResNet50 on ImageNette}{\includegraphics[width=0.225\textwidth]{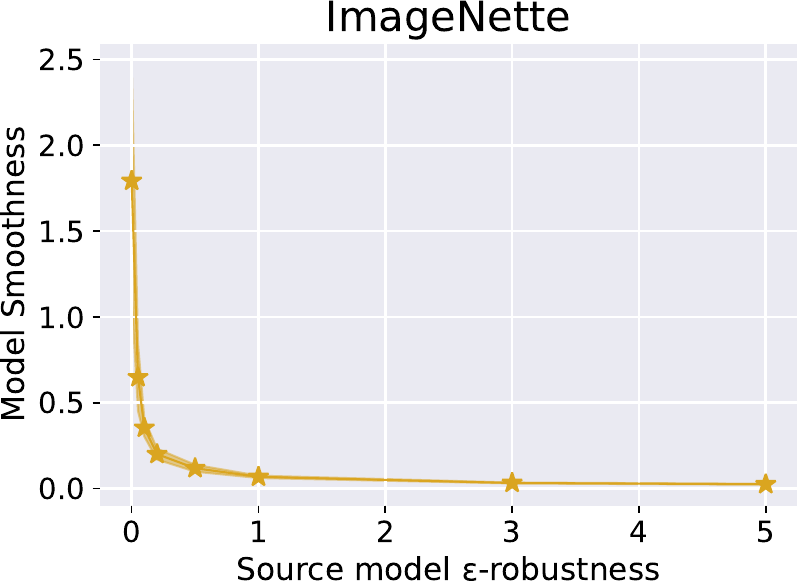}}
      \caption{\black{Average model smoothness  of $\epsilon$-robust models trained over 3 different random seeds  with corresponding error bars at each $\epsilon$.
       Note that the variances are very small.}}
\label{fig:adv-ms}
\vspace{-3mm}
\end{figure}

\item \textbf{gradient similarity $\underline{\mathcal{S}}_\mathcal{D} \left(\ell_\mathcal{F}, \ell_{\mathcal{G}}\right)$}: 
$\overline{\sigma}_{\mathcal{G}}$ is small compared with the perturbation radius $\epsilon$, and the gradient magnitude $\left\|\nabla_{x} \ell_{\mathcal{G}}\right\|_{2}$ is relatively large,  leading to a small $c_{\mathcal{G}}$. 
Additionally, $1-\alpha$ is large since the attack is generally effective against $\mathcal{F}$. 
Thus, the  right side of the inequality has a form of $C- k\sqrt{1- \underline{\mathcal{S}}_\mathcal{D}\left( \ell_\mathcal{F}, \ell_{\mathcal{G}}\right)}$. Since $C$ and $k$ are both positive, there is a positive relationship between gradient  similarity and the lower bound. 

\item \textbf{model smoothness $\overline{\sigma}_{\mathcal{F}}$}:  It's obvious  that a smaller $\overline{\sigma}_{\mathcal{F}}$ generates a larger $c_{\mathcal{F}}$, resulting in a larger lower bound. This indicates smoother models in the input space might serve as better surrogates for transfer attacks. 
\end{itemize}
It is worth noting that  a sensible adversary naturally desires to  use  more precise surrogates and stronger attacks against unknown targets. 
However, smoothness and similarity are  more complex to understand and optimize compared to fooling probability and natural risk.
Therefore, we focus on investigating  the connection between smoothness, similarity, and transferability as well as how the various training mechanisms regulate them under the restriction that the other two factors above  are fairly acceptable.
Prior study \cite{security19} has shown their link with transferability separately, while we explore their joint effect on transferability.



\subsection{Trade-off Under Adversarial Training}
\label{sec:TC-AT}

\begin{figure}[!t] 
  \centering
\subcaptionbox{CIFAR-10}{\includegraphics[width=0.235 \textwidth]{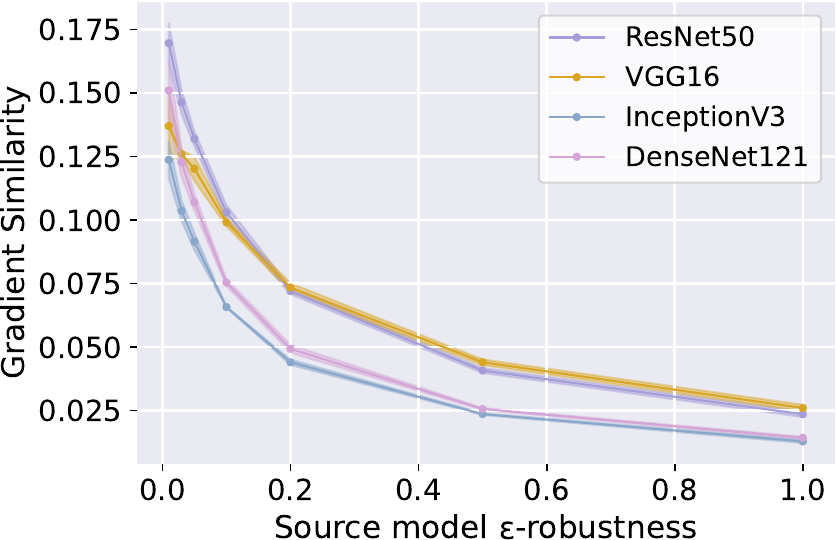}}\label{sub:gs-cifar}
\hfill
\subcaptionbox{ImageNette}{\includegraphics[width=0.235 \textwidth]{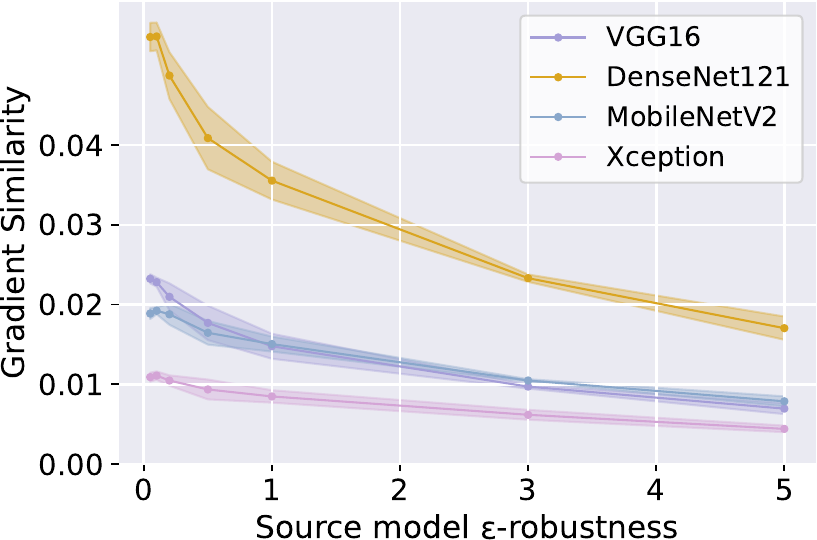}}\label{sub:gs-imagenette}
\caption{\black{Average gradient similarities between $\epsilon$-robust models and different target models with corresponding error bars at each $\epsilon$ over  3  different random seeds.}}
\label{fig:adv-gs}
\vspace{-3mm}
\end{figure}

Through a combination of theoretical analysis and experimental measurements, here we provide our insights into this ``transferability circuit'' in  adversarial training.
Arguably, we make the following two conjectures.

\noindent\textbf{From surrogate's perspective, the trade-off between model smoothness  and gradient similarity  can significantly  indicate adversarial transferability.}
First, through mathematical derivation and empirical results, we demonstrate that \AT ~implicitly improves model smoothness.
We rewrite the adversarial loss in \cref{loss_adv} as:
$\ell(f(x_i), y_i) + [\max\limits_{\left\|\delta\right\|_{2} < \epsilon} \ell(f(x_i+\delta),y_i) - \ell(f(x_i), y_i)]$.
The latter term  reflects the non-smoothness around point $(x_i, y_i)$.
Assuming $x_i$ is a local minimum of $\ell(f(\cdot),y_i)$ and applying a Taylor expansion, we get
$\max\limits_{\left\|\delta\right\|_2 \leq \epsilon} \ell(f(x_i+\delta),y_i) - \ell(f(x_i), y_i) = \frac{1}{2} \sigma(\nabla{x}^2 \ell(f(x_i), y_i)) \cdot \|\delta\|_2^2 + O(\|\delta\|_2^3)$
since the first-order term is 0 at a local minimum.
Therefore, \AT ~implicitly penalizes the curvature of $\ell(f(x_i), y_i)$ with a penalty proportional to the norm of adversarial noise $\|\delta\|_2^2$.

Accordingly, a larger $\epsilon$ strengthens the effect of producing models with greater smoothness. \cref{fig:adv-ms} shows the mean dominant eigenvalues of the Hessian for \AT ~models under various perturbation budgets. It reveals that \AT ~consistently suppresses the dominant eigenvalue, stably producing smoother models.
On the other hand, we empirically find that \AT~impairs gradient similarity. \cref{fig:adv-gs} shows the decay of  similarity between surrogate and target models as  the  budget increases.
Notably, for each dataset, smoothness rapidly improves for small $\epsilon$ (as shown by the steep slope in \cref{fig:adv-ms}), while gradient similarity gradually decreases. 
For larger $\epsilon$, improvements in smoothness become marginal since the models are already very smooth.
As a reminder, we have controlled the other two terms  to be reasonably acceptable. These suggest that:
\begin{itemize}[leftmargin=11pt]
    \item The quick improvement in transferability for small $\epsilon$ occurs could be the cause of the rapid gains in smoothness and small decays in gradient similarity.
    \item The degradation in transferability for large $\epsilon$ occurs  may because the smoothness gains have approached the limit while gradient similarity continues to decrease.
\end{itemize}

\begin{table*}[!t]
\renewcommand \arraystretch{0.945}
\setlength{\abovecaptionskip}{0.05cm}
  \centering
  \caption{Untargeted and targeted transfer ASRs of AEs crafted against baseline (\ST), data augmentations, adversarial training, and gradient regularizations surrogates under the different $L_\infty$ budgets (4/255, 8/255, 16/255) on \textbf{CIFAR-10} using \textbf{PGD}. Results in \colorbox[rgb]{ .949,  .949,  .949}{gray} are (more than 0.2\%) \uline{below the baseline}, results in \colorbox[rgb]{1.0,1.0,1.0}{white} are \uline{close to the baseline} ($\pm0.2\%$). 
  \black{Results in \colorbox[rgb]{.957, .69, .518}{darkest orange} are \uline{the highest} in each respective column, the  \colorbox[rgb]{.973,  .796,  .678}{second darkest} are \uline{the second highest}, and the \colorbox[rgb]{ .988,  .894,  .839}{least dark} are others more than 0.2\% \uline{above the baseline}. }
  We conduct experiments using surrogate models trained from 3 different random seeds in each setting and report the average results and the according error bars. In \%.}
  \tabcolsep=0.05cm
  \resizebox{\textwidth}{!} {
    \begin{tabular}{c|cccc|cccc|cccc}
    \toprule
    \multicolumn{13}{c}{\textbf{Untargeted}} \\
    \hline
     & \multicolumn{4}{c|}{4/255}    & \multicolumn{4}{c|}{8/255}    & \multicolumn{4}{c}{16/255} \\
      \multicolumn{1}{l|}{}     & ResNet50&VGG16& InceptionV3&DenseNet121& ResNet50&VGG16&InceptionV3&DenseNet121&ResNet50&VGG16&InceptionV3& DenseNet121 \\ \cline{1-13}
    \ST    & \multicolumn{1}{c}{54.3$_{\pm2.9}$} & \multicolumn{1}{c}{42.3$_{\pm1.9}$} & \multicolumn{1}{c}{53.6$_{\pm2.3}$} & 70.7$_{\pm2.0}$ & 79.9$_{\pm3.4}$ & 71.4$_{\pm3.5}$ & 80.1$_{\pm3.2}$ & 93.0$_{\pm1.2}$ & 92.8$_{\pm1.6}$ & 90.6$_{\pm1.9}$ & 93.3$_{\pm2.2}$ & 98.6$_{\pm0.5}$ \\
    \hline
    \MU, $\tau=1$ & \cellcolor[rgb]{ .949,  .949,  .949}34.7$_{\pm2.6}$ & \cellcolor[rgb]{ .949,  .949,  .949}26.7$_{\pm1.1}$ & \cellcolor[rgb]{ .949,  .949,  .949}35.8$_{\pm1.8}$ & \cellcolor[rgb]{ .949,  .949,  .949}47.1$_{\pm3.0}$ & \cellcolor[rgb]{ .949,  .949,  .949}57.1$_{\pm3.6}$ & \cellcolor[rgb]{ .949,  .949,  .949}49.1$_{\pm2.2}$ & \cellcolor[rgb]{ .949,  .949,  .949}59.2$_{\pm2.6}$ & \cellcolor[rgb]{ .949,  .949,  .949}74.0$_{\pm2.8}$ & \cellcolor[rgb]{ .949,  .949,  .949}76.6$_{\pm4.1}$ & \cellcolor[rgb]{ .949,  .949,  .949}75.2$_{\pm3.4}$ & \cellcolor[rgb]{ .949,  .949,  .949}79.9$_{\pm2.5}$ & \cellcolor[rgb]{ .949,  .949,  .949}90.2$_{\pm2.4}$ \\
    \MU, $\tau=5$ & \cellcolor[rgb]{ .949,  .949,  .949}18.5$_{\pm0.5}$ & \cellcolor[rgb]{ .949,  .949,  .949}15.5$_{\pm0.2}$ & \cellcolor[rgb]{ .949,  .949,  .949}19.6$_{\pm0.3}$ & \cellcolor[rgb]{ .949,  .949,  .949}22.8$_{\pm0.5}$ & \cellcolor[rgb]{ .949,  .949,  .949}31.3$_{\pm1.7}$ & \cellcolor[rgb]{ .949,  .949,  .949}27.6$_{\pm0.7}$ & \cellcolor[rgb]{ .949,  .949,  .949}34.5$_{\pm0.7}$ & \cellcolor[rgb]{ .949,  .949,  .949}40.7$_{\pm1.2}$ & \cellcolor[rgb]{ .949,  .949,  .949}49.9$_{\pm1.5}$ & \cellcolor[rgb]{ .949,  .949,  .949}49.1$_{\pm0.9}$ & \cellcolor[rgb]{ .949,  .949,  .949}54.4$_{\pm0.3}$ & \cellcolor[rgb]{ .949,  .949,  .949}62.8$_{\pm1.2}$ \\
    \CM, $\tau=1$ & \cellcolor[rgb]{ .949,  .949,  .949}29.3$_{\pm1.2}$ & \cellcolor[rgb]{ .949,  .949,  .949}21.3$_{\pm0.5}$ & \cellcolor[rgb]{ .949,  .949,  .949}28.4$_{\pm2.1}$ & \cellcolor[rgb]{ .949,  .949,  .949}39.8$_{\pm2.0}$ & \cellcolor[rgb]{ .949,  .949,  .949}46.4$_{\pm2.5}$ & \cellcolor[rgb]{ .949,  .949,  .949}37.6$_{\pm1.4}$ & \cellcolor[rgb]{ .949,  .949,  .949}45.6$_{\pm3.2}$ & \cellcolor[rgb]{ .949,  .949,  .949}62.8$_{\pm2.4}$ & \cellcolor[rgb]{ .949,  .949,  .949}63.3$_{\pm2.6}$ & \cellcolor[rgb]{ .949,  .949,  .949}60.2$_{\pm1.4}$ & \cellcolor[rgb]{ .949,  .949,  .949}63.9$_{\pm3.2}$ & \cellcolor[rgb]{ .949,  .949,  .949}81.1$_{\pm1.2}$ \\
    \CM, $\tau=5$ & \cellcolor[rgb]{ .949,  .949,  .949}13.5$_{\pm0.3}$ & \cellcolor[rgb]{ .949,  .949,  .949}10.8$_{\pm0.1}$ & \cellcolor[rgb]{ .949,  .949,  .949}13.1$_{\pm0.4}$ & \cellcolor[rgb]{ .949,  .949,  .949}15.1$_{\pm0.3}$ & \cellcolor[rgb]{ .949,  .949,  .949}19.5$_{\pm0.4}$ & \cellcolor[rgb]{ .949,  .949,  .949}15.7$_{\pm0.2}$ & \cellcolor[rgb]{ .949,  .949,  .949}18.7$_{\pm0.8}$ & \cellcolor[rgb]{ .949,  .949,  .949}22.6$_{\pm0.7}$ & \cellcolor[rgb]{ .949,  .949,  .949}29.4$_{\pm1.4}$ & \cellcolor[rgb]{ .949,  .949,  .949}28.6$_{\pm1.0}$ & \cellcolor[rgb]{ .949,  .949,  .949}29.3$_{\pm1.1}$ & \cellcolor[rgb]{ .949,  .949,  .949}34.3$_{\pm1.5}$ \\
    \CO, $\tau=1$ & \cellcolor[rgb]{ .949,  .949,  .949}51.5$_{\pm2.3}$ & \cellcolor[rgb]{ .949,  .949,  .949}41.3$_{\pm2.4}$ & \cellcolor[rgb]{ .949,  .949,  .949}50.0$_{\pm1.8}$ & \cellcolor[rgb]{ .949,  .949,  .949}67.7$_{\pm1.4}$ & \cellcolor[rgb]{ .949,  .949,  .949}78.7$_{\pm2.4}$ & \cellcolor[rgb]{  .988,  .894,  .839}71.9$_{\pm3.2}$ & \cellcolor[rgb]{ .949,  .949,  .949}78.2$_{\pm0.9}$ & \cellcolor[rgb]{ .949,  .949,  .949}92.3$_{\pm0.9}$ & 93.0$_{\pm2.1}$ & \cellcolor[rgb]{  .988,  .894,  .839}91.8$_{\pm1.5}$ & \cellcolor[rgb]{  .988,  .894,  .839}93.6$_{\pm0.7}$ & \cellcolor[rgb]{  .988,  .894,  .839}98.9$_{\pm0.2}$ \\
    \CO, $\tau=5$ & \cellcolor[rgb]{ .949,  .949,  .949}37.7$_{\pm2.8}$ & \cellcolor[rgb]{ .949,  .949,  .949}26.9$_{\pm1.7}$ & \cellcolor[rgb]{ .949,  .949,  .949}36.9$_{\pm2.0}$ & \cellcolor[rgb]{ .949,  .949,  .949}49.5$_{\pm1.2}$ & \cellcolor[rgb]{ .949,  .949,  .949}63.9$_{\pm4.5}$ & \cellcolor[rgb]{ .949,  .949,  .949}53.1$_{\pm3.6}$ & \cellcolor[rgb]{ .949,  .949,  .949}64.1$_{\pm3.9}$ & \cellcolor[rgb]{ .949,  .949,  .949}80.4$_{\pm1.1}$ & \cellcolor[rgb]{ .949,  .949,  .949}84.7$_{\pm3.3}$ & \cellcolor[rgb]{ .949,  .949,  .949}79.1$_{\pm3.5}$ & \cellcolor[rgb]{ .949,  .949,  .949}85.5$_{\pm3.2}$ & \cellcolor[rgb]{ .949,  .949,  .949}95.3$_{\pm1.0}$ \\
    \LS, $\tau=1$ & \cellcolor[rgb]{ .949,  .949,  .949}33.4$_{\pm1.7}$ & \cellcolor[rgb]{ .949,  .949,  .949}28.6$_{\pm1.4}$ & \cellcolor[rgb]{ .949,  .949,  .949}33.2$_{\pm0.6}$ & \cellcolor[rgb]{ .949,  .949,  .949}43.4$_{\pm1.5}$ & \cellcolor[rgb]{ .949,  .949,  .949}47.5$_{\pm2.0}$ & \cellcolor[rgb]{ .949,  .949,  .949}44.8$_{\pm1.1}$ & \cellcolor[rgb]{ .949,  .949,  .949}47.6$_{\pm0.8}$ & \cellcolor[rgb]{ .949,  .949,  .949}61.8$_{\pm1.9}$ & \cellcolor[rgb]{ .949,  .949,  .949}59.9$_{\pm1.5}$ & \cellcolor[rgb]{ .949,  .949,  .949}63.0$_{\pm1.1}$ & \cellcolor[rgb]{ .949,  .949,  .949}62.9$_{\pm1.0}$ & \cellcolor[rgb]{ .949,  .949,  .949}75.8$_{\pm1.4}$ \\
    \LS, $\tau=5$ & \cellcolor[rgb]{ .949,  .949,  .949}32.8$_{\pm0.5}$ & \cellcolor[rgb]{ .949,  .949,  .949}28.0$_{\pm1.3}$ & \cellcolor[rgb]{ .949,  .949,  .949}31.9$_{\pm0.7}$ & \cellcolor[rgb]{ .949,  .949,  .949}40.8$_{\pm1.8}$ & \cellcolor[rgb]{ .949,  .949,  .949}50.7$_{\pm1.4}$ & \cellcolor[rgb]{ .949,  .949,  .949}46.8$_{\pm2.4}$ & \cellcolor[rgb]{ .949,  .949,  .949}50.6$_{\pm0.6}$ & \cellcolor[rgb]{ .949,  .949,  .949}63.4$_{\pm0.9}$ & \cellcolor[rgb]{ .949,  .949,  .949}65.3$_{\pm2.1}$ & \cellcolor[rgb]{ .949,  .949,  .949}66.1$_{\pm1.9}$ & \cellcolor[rgb]{ .949,  .949,  .949}66.7$_{\pm0.5}$ & \cellcolor[rgb]{ .949,  .949,  .949}78.0$_{\pm0.7}$ \\
    \hline
    \AT & \cellcolor[rgb]{ .988,  .894,  .839}59.0$_{\pm3.0}$ & \cellcolor[rgb]{ .988,  .894,  .839}50.3$_{\pm4.4}$ & \cellcolor[rgb]{ .929,  .929,  .929}53.3$_{\pm2.9}$ & \cellcolor[rgb]{ .929,  .929,  .929}65.3$_{\pm2.7}$ & \cellcolor[rgb]{ .988,  .894,  .839}87.8$_{\pm2.8}$ & \cellcolor[rgb]{ .988,  .894,  .839}82.7$_{\pm4.5}$ & \cellcolor[rgb]{ .988,  .894,  .839}84.7$_{\pm2.7}$ & \cellcolor[rgb]{ .929,  .929,  .929}92.2$_{\pm0.8}$ & \cellcolor[rgb]{ .988,  .894,  .839}96.9$_{\pm1.3}$ & \cellcolor[rgb]{ .988,  .894,  .839}95.8$_{\pm2.0}$ & \cellcolor[rgb]{ .988,  .894,  .839}96.0$_{\pm1.4}$ & 98.6$_{\pm0.3}$ \\
    \hline
    \IR & \cellcolor[rgb]{ .988,  .894,  .839}62.3$_{\pm1.2}$ & \cellcolor[rgb]{ .988,  .894,  .839}55.6$_{\pm0.5}$ & \cellcolor[rgb]{ .988,  .894,  .839}56.5$_{\pm2.1}$ & \cellcolor[rgb]{ .929,  .929,  .929}63.7$_{\pm1.9}$ & \cellcolor[rgb]{ .988,  .894,  .839}96.0$_{\pm0.4}$ & \cellcolor[rgb]{ .988,  .894,  .839}93.2$_{\pm0.2}$ & \cellcolor[rgb]{ .988,  .894,  .839}93.8$_{\pm0.6}$ & \cellcolor[rgb]{ .988,  .894,  .839}96.9$_{\pm0.6}$ & \cellcolor[rgb]{ .973,  .796,  .678}99.9$_{\pm0.0}$ & \cellcolor[rgb]{ .973,  .796,  .678}99.7$_{\pm0.0}$ & \cellcolor[rgb]{ .988,  .894,  .839}99.8$_{\pm0.1}$ & \cellcolor[rgb]{ .973,  .796,  .678}99.9$_{\pm0.1}$ \\
    \JR & \cellcolor[rgb]{ .988,  .894,  .839}74.2$_{\pm2.0}$ & \cellcolor[rgb]{ .973,  .796,  .678}65.8$_{\pm0.8}$ & \cellcolor[rgb]{ .988,  .894,  .839}72.6$_{\pm0.3}$ & \cellcolor[rgb]{ .988,  .894,  .839}84.1$_{\pm0.5}$ & \cellcolor[rgb]{ .988,  .894,  .839}94.1$_{\pm0.7}$ & \cellcolor[rgb]{ .988,  .894,  .839}91.3$_{\pm0.3}$ & \cellcolor[rgb]{ .988,  .894,  .839}93.4$_{\pm0.6}$ & \cellcolor[rgb]{ .988,  .894,  .839}97.7$_{\pm0.5}$ & \cellcolor[rgb]{ .988,  .894,  .839}98.3$_{\pm0.3}$ & \cellcolor[rgb]{ .988,  .894,  .839}98.3$_{\pm0.2}$ & \cellcolor[rgb]{ .988,  .894,  .839}98.5$_{\pm0.4}$ & \cellcolor[rgb]{ .988,  .894,  .839}99.7$_{\pm0.5}$ \\
    \ER & \cellcolor[rgb]{ .988,  .894,  .839}64.9$_{\pm3.3}$ & \cellcolor[rgb]{ .988,  .894,  .839}52.5$_{\pm1.9}$ & \cellcolor[rgb]{ .988,  .894,  .839}54.2$_{\pm6.7}$ & \cellcolor[rgb]{ .929,  .929,  .929}62.8$_{\pm10.9}$ & \cellcolor[rgb]{ .988,  .894,  .839}94.2$_{\pm1.4}$ & \cellcolor[rgb]{ .988,  .894,  .839}87.1$_{\pm5.3}$ & \cellcolor[rgb]{ .988,  .894,  .839}87.9$_{\pm4.7}$ & 92.8$_{\pm5.3}$ & \cellcolor[rgb]{ .988,  .894,  .839}99.6$_{\pm0.2}$ & \cellcolor[rgb]{ .988,  .894,  .839}98.7$_{\pm4.0}$ & \cellcolor[rgb]{ .988,  .894,  .839}98.6$_{\pm0.8}$ & \cellcolor[rgb]{ .988,  .894,  .839}99.5$_{\pm0.6}$ \\
    \SAM & \cellcolor[rgb]{ .973,  .796,  .678}76.1$_{\pm1.5}$ & \cellcolor[rgb]{ .988,  .894,  .839}64.3$_{\pm2.1}$ & \cellcolor[rgb]{ .973,  .796,  .678}74.5$_{\pm0.7}$ & \cellcolor[rgb]{ .973,  .796,  .678}88.5$_{\pm0.4}$ & \cellcolor[rgb]{ .988,  .894,  .839}97.5$_{\pm0.5}$ & \cellcolor[rgb]{ .988,  .894,  .839}94.1$_{\pm0.8}$ & \cellcolor[rgb]{ .973,  .796,  .678}97.1$_{\pm0.1}$ & \cellcolor[rgb]{ .973,  .796,  .678}99.6$_{\pm0.1}$ & \cellcolor[rgb]{ .988,  .894,  .839}99.8$_{\pm0.0}$ & \cellcolor[rgb]{ .973,  .796,  .678}99.7$_{\pm0.1}$ & \cellcolor[rgb]{ .973,  .796,  .678}99.9$_{\pm0.0}$ & \cellcolor[rgb]{ .957,  .69,  .518}100.0$_{\pm0.0}$ \\ \hline
    \SAM\&\IR & \cellcolor[rgb]{ .988,  .894,  .839}64.7$_{\pm0.3}$ & \cellcolor[rgb]{ .988,  .894,  .839}58.2$_{\pm0.4}$ & \cellcolor[rgb]{ .988,  .894,  .839}58.4$_{\pm1.6}$ & \cellcolor[rgb]{ .929,  .929,  .929}65.4$_{\pm1.0}$ & \cellcolor[rgb]{ .973,  .796,  .678}97.6$_{\pm0.0}$ & \cellcolor[rgb]{ .973,  .796,  .678}95.5$_{\pm0.1}$ & \cellcolor[rgb]{ .988,  .894,  .839}95.9$_{\pm0.4}$ & \cellcolor[rgb]{ .988,  .894,  .839}97.9$_{\pm0.3}$ & \cellcolor[rgb]{ .957,  .69,  .518}100.0$_{\pm0.0}$ & \cellcolor[rgb]{ .957,  .69,  .518}99.9$_{\pm0.0}$ & \cellcolor[rgb]{ .957,  .69,  .518}100.0$_{\pm0.0}$ & \cellcolor[rgb]{ .957,  .69,  .518}100.0$_{\pm0.0}$ \\
    \SAM\&\JR & \cellcolor[rgb]{ .957,  .69,  .518}83.1$_{\pm0.6}$ & \cellcolor[rgb]{ .957,  .69,  .518}70.1$_{\pm0.9}$ & \cellcolor[rgb]{ .957,  .69,  .518}79.4$_{\pm0.6}$ & \cellcolor[rgb]{ .957,  .69,  .518}91.5$_{\pm0.3}$ & \cellcolor[rgb]{ .957,  .69,  .518}98.6$_{\pm0.2}$ & \cellcolor[rgb]{ .957,  .69,  .518}96.6$_{\pm0.2}$ & \cellcolor[rgb]{ .957,  .69,  .518}98.3$_{\pm0.1}$ & \cellcolor[rgb]{ .957,  .69,  .518}99.9$_{\pm0.0}$ & \cellcolor[rgb]{ .973,  .796,  .678}99.9$_{\pm0.0}$ & \cellcolor[rgb]{ .957,  .69,  .518}99.9$_{\pm0.0}$ & \cellcolor[rgb]{ .973,  .796,  .678}99.9$_{\pm0.0}$ & \cellcolor[rgb]{ .957,  .69,  .518}100.0$_{\pm0.0}$ \\
    \toprule
    \multicolumn{13}{c}{\textbf{Targeted}} \\

    \hline
    \ST   & 18.3$_{\pm1.4}$ & 12.7$_{\pm0.6}$ & 20.4$_{\pm1.1}$ & 34.9$_{\pm1.9}$ & 39.4$_{\pm4.9}$ & 32.7$_{\pm3.6}$ & 43.6$_{\pm4.0}$ & 68.8$_{\pm4.0}$ & 55.7$_{\pm6.0}$ & 53.5$_{\pm5.5}$ & 60.0$_{\pm5.3}$ & 85.4$_{\pm4.2}$ \\
    \hline
   \MU, $\tau=1$ & \cellcolor[rgb]{ .949,  .949,  .949}8.2$_{\pm0.8}$ & \cellcolor[rgb]{ .949,  .949,  .949}5.9$_{\pm0.5}$ & \cellcolor[rgb]{ .949,  .949,  .949}10.5$_{\pm0.6}$ & \cellcolor[rgb]{ .949,  .949,  .949}15.9$_{\pm1.0}$ & \cellcolor[rgb]{ .949,  .949,  .949}17.0$_{\pm2.7}$ & \cellcolor[rgb]{ .949,  .949,  .949}13.8$_{\pm1.4}$ & \cellcolor[rgb]{ .949,  .949,  .949}20.7$_{\pm2.1}$ & \cellcolor[rgb]{ .949,  .949,  .949}33.3$_{\pm3.5}$ & \cellcolor[rgb]{ .949,  .949,  .949}28.0$_{\pm5.1}$ & \cellcolor[rgb]{ .949,  .949,  .949}26.8$_{\pm2.9}$ & \cellcolor[rgb]{ .949,  .949,  .949}33.1$_{\pm3.5}$ & \cellcolor[rgb]{ .949,  .949,  .949}50.9$_{\pm6.2}$ \\
    \MU, $\tau=5$ & \cellcolor[rgb]{ .949,  .949,  .949}3.0$_{\pm0.1}$ & \cellcolor[rgb]{ .949,  .949,  .949}2.4$_{\pm0.1}$ & \cellcolor[rgb]{ .949,  .949,  .949}4.1$_{\pm0.1}$ & \cellcolor[rgb]{ .949,  .949,  .949}5.0$_{\pm0.2}$ & \cellcolor[rgb]{ .949,  .949,  .949}6.1$_{\pm0.5}$ & \cellcolor[rgb]{ .949,  .949,  .949}5.0$_{\pm0.4}$ & \cellcolor[rgb]{ .949,  .949,  .949}7.9$_{\pm0.5}$ & \cellcolor[rgb]{ .949,  .949,  .949}10.4$_{\pm0.5}$ & \cellcolor[rgb]{ .949,  .949,  .949}10.8$_{\pm0.5}$ & \cellcolor[rgb]{ .949,  .949,  .949}10.7$_{\pm0.6}$ & \cellcolor[rgb]{ .949,  .949,  .949}13.5$_{\pm0.6}$ & \cellcolor[rgb]{ .949,  .949,  .949}16.9$_{\pm0.9}$ \\
    \CM, $\tau=1$  & \cellcolor[rgb]{ .949,  .949,  .949}6.0$_{\pm0.4}$ & \cellcolor[rgb]{ .949,  .949,  .949}4.2$_{\pm0.2}$ & \cellcolor[rgb]{ .949,  .949,  .949}6.8$_{\pm1.0}$ & \cellcolor[rgb]{ .949,  .949,  .949}11.4$_{\pm1.1}$ & \cellcolor[rgb]{ .949,  .949,  .949}10.7$_{\pm0.7}$ & \cellcolor[rgb]{ .949,  .949,  .949}8.3$_{\pm0.5}$ & \cellcolor[rgb]{ .949,  .949,  .949}11.6$_{\pm1.8}$ & \cellcolor[rgb]{ .949,  .949,  .949}20.8$_{\pm1.9}$ & \cellcolor[rgb]{ .949,  .949,  .949}16.4$_{\pm0.9}$ & \cellcolor[rgb]{ .949,  .949,  .949}14.9$_{\pm0.8}$ & \cellcolor[rgb]{ .949,  .949,  .949}17.3$_{\pm2.5}$ & \cellcolor[rgb]{ .949,  .949,  .949}29.6$_{\pm1.7}$ \\
    \CM, $\tau = 5$ & \cellcolor[rgb]{ .949,  .949,  .949}2.0$_{\pm0.1}$ & \cellcolor[rgb]{ .949,  .949,  .949}1.6$_{\pm0.1}$ & \cellcolor[rgb]{ .949,  .949,  .949}2.4$_{\pm0.2}$ & \cellcolor[rgb]{ .949,  .949,  .949}2.8$_{\pm0.2}$ & \cellcolor[rgb]{ .949,  .949,  .949}2.8$_{\pm0.1}$ & \cellcolor[rgb]{ .949,  .949,  .949}2.3$_{\pm0.0}$ & \cellcolor[rgb]{ .949,  .949,  .949}3.2$_{\pm0.1}$ & \cellcolor[rgb]{ .949,  .949,  .949}3.7$_{\pm0.1}$ & \cellcolor[rgb]{ .949,  .949,  .949}4.2$_{\pm0.2}$ & \cellcolor[rgb]{ .949,  .949,  .949}3.9$_{\pm0.4}$ & \cellcolor[rgb]{ .949,  .949,  .949}4.4$_{\pm0.2}$ & \cellcolor[rgb]{ .949,  .949,  .949}4.7$_{\pm0.3}$ \\
    \CO, $\tau=1$ & \cellcolor[rgb]{ .949,  .949,  .949}16.0$_{\pm1.3}$ & \cellcolor[rgb]{ .949,  .949,  .949}11.3$_{\pm1.2}$ & \cellcolor[rgb]{ .949,  .949,  .949}17.6$_{\pm1.1}$ & \cellcolor[rgb]{ .949,  .949,  .949}30.6$_{\pm1.7}$ & \cellcolor[rgb]{ .949,  .949,  .949}35.8$_{\pm3.0}$ & \cellcolor[rgb]{ .949,  .949,  .949}30.7$_{\pm2.5}$ & \cellcolor[rgb]{ .949,  .949,  .949}38.6$_{\pm2.7}$ & \cellcolor[rgb]{ .949,  .949,  .949}63.4$_{\pm2.2}$ & \cellcolor[rgb]{ .949,  .949,  .949}52.8$_{\pm4.3}$ & \cellcolor[rgb]{ .949,  .949,  .949}51.5$_{\pm3.0}$ & \cellcolor[rgb]{ .949,  .949,  .949}55.9$_{\pm3.1}$ & \cellcolor[rgb]{ .949,  .949,  .949}82.1$_{\pm1.6}$ \\
    \CO, $\tau=5$ & \cellcolor[rgb]{ .949,  .949,  .949}7.9$_{\pm1.1}$ & \cellcolor[rgb]{ .949,  .949,  .949}5.4$_{\pm0.5}$ & \cellcolor[rgb]{ .949,  .949,  .949}9.5$_{\pm0.7}$ & \cellcolor[rgb]{ .949,  .949,  .949}14.2$_{\pm0.7}$ & \cellcolor[rgb]{ .949,  .949,  .949}17.4$_{\pm2.9}$ & \cellcolor[rgb]{ .949,  .949,  .949}13.6$_{\pm1.3}$ & \cellcolor[rgb]{ .949,  .949,  .949}20.5$_{\pm2.1}$ & \cellcolor[rgb]{ .949,  .949,  .949}31.1$_{\pm0.9}$ & \cellcolor[rgb]{ .949,  .949,  .949}29.1$_{\pm5.0}$ & \cellcolor[rgb]{ .949,  .949,  .949}24.9$_{\pm2.8}$ & \cellcolor[rgb]{ .949,  .949,  .949}31.9$_{\pm4.1}$ & \cellcolor[rgb]{ .949,  .949,  .949}48.5$_{\pm1.8}$ \\
    \LS, $\tau=1$  & \cellcolor[rgb]{ .949,  .949,  .949}6.8$_{\pm0.6}$ & \cellcolor[rgb]{ .949,  .949,  .949}5.4$_{\pm0.4}$ & \cellcolor[rgb]{ .949,  .949,  .949}7.0$_{\pm0.3}$ & \cellcolor[rgb]{ .949,  .949,  .949}10.7$_{\pm1.3}$ & \cellcolor[rgb]{ .949,  .949,  .949}9.9$_{\pm3.8}$ & \cellcolor[rgb]{ .949,  .949,  .949}9.0$_{\pm4.0}$ & \cellcolor[rgb]{ .949,  .949,  .949}10.0$_{\pm3.3}$ & \cellcolor[rgb]{ .949,  .949,  .949}15.6$_{\pm6.2}$ & \cellcolor[rgb]{ .949,  .949,  .949}15.4$_{\pm1.9}$ & \cellcolor[rgb]{ .949,  .949,  .949}16.7$_{\pm1.8}$ & \cellcolor[rgb]{ .949,  .949,  .949}16.3$_{\pm1.2}$ & \cellcolor[rgb]{ .929,  .929,  .929}23.9$_{\pm2.9}$ \\
    \LS, $\tau=5$  & \cellcolor[rgb]{ .949,  .949,  .949}6.7$_{\pm0.4}$ & \cellcolor[rgb]{ .949,  .949,  .949}5.7$_{\pm0.1}$ & \cellcolor[rgb]{ .949,  .949,  .949}7.1$_{\pm0.8}$ & \cellcolor[rgb]{ .949,  .949,  .949}10.5$_{\pm1.3}$ & \cellcolor[rgb]{ .949,  .949,  .949}13.5$_{\pm1.1}$ & \cellcolor[rgb]{ .949,  .949,  .949}12.8$_{\pm0.7}$ & \cellcolor[rgb]{ .949,  .949,  .949}14.0$_{\pm1.5}$ & \cellcolor[rgb]{ .949,  .949,  .949}20.1$_{\pm2.0}$ & \cellcolor[rgb]{ .949,  .949,  .949}19.4$_{\pm0.9}$ & \cellcolor[rgb]{ .949,  .949,  .949}20.0$_{\pm0.3}$ & \cellcolor[rgb]{ .949,  .949,  .949}20.0$_{\pm1.4}$ & \cellcolor[rgb]{ .949,  .949,  .949}27.0$_{\pm1.8}$ \\
    \hline 
    \AT  & \cellcolor[rgb]{ .988,  .894,  .839}24.9$_{\pm2.2}$ & \cellcolor[rgb]{ .988,  .894,  .839}18.6$_{\pm3.4}$ & \cellcolor[rgb]{ .988,  .894,  .839}21.9$_{\pm2.0}$ & \cellcolor[rgb]{ .949,  .949,  .949}32.5$_{\pm2.5}$ & \cellcolor[rgb]{ .988,  .894,  .839}60.7$_{\pm6.7}$ & \cellcolor[rgb]{ .988,  .894,  .839}52.4$_{\pm7.9}$ & \cellcolor[rgb]{ .988,  .894,  .839}55.6$_{\pm5.5}$ & \cellcolor[rgb]{ .988,  .894,  .839}72.6$_{\pm4.0}$ & \cellcolor[rgb]{ .988,  .894,  .839}77.4$_{\pm7.7}$ & \cellcolor[rgb]{ .988,  .894,  .839}74.6$_{\pm7.9}$ & \cellcolor[rgb]{ .988,  .894,  .839}72.4$_{\pm6.7}$ & \cellcolor[rgb]{ .988,  .894,  .839}88.4$_{\pm4.1}$ \\
    \hline
    \IR  & \cellcolor[rgb]{ .988,  .894,  .839}25.7$_{\pm1.1}$ & \cellcolor[rgb]{ .988,  .894,  .839}20.6$_{\pm0.2}$ & \cellcolor[rgb]{ .988,  .894,  .839}22.5$_{\pm1.1}$ & \cellcolor[rgb]{ .949,  .949,  .949}28.8$_{\pm1.3}$ & \cellcolor[rgb]{ .988,  .894,  .839}70.8$_{\pm0.5}$ & \cellcolor[rgb]{ .988,  .894,  .839}62.8$_{\pm0.7}$ & \cellcolor[rgb]{ .988,  .894,  .839}65.0$_{\pm1.5}$ & \cellcolor[rgb]{ .988,  .894,  .839}75.2$_{\pm0.6}$ & \cellcolor[rgb]{ .988,  .894,  .839}89.9$_{\pm0.7}$ & \cellcolor[rgb]{ .988,  .894,  .839}86.7$_{\pm0.8}$ & \cellcolor[rgb]{ .988,  .894,  .839}86.3$_{\pm1.3}$ & \cellcolor[rgb]{ .988,  .894,  .839}92.7$_{\pm0.8}$ \\
    \JR & \cellcolor[rgb]{ .973,  .796,  .678}38.6$_{\pm2.6}$ & \cellcolor[rgb]{ .973,  .796,  .678}30.3$_{\pm1.2}$ & \cellcolor[rgb]{ .973,  .796,  .678}39.1$_{\pm0.1}$ & \cellcolor[rgb]{ .988,  .894,  .839}54.0$_{\pm1.2}$ & \cellcolor[rgb]{ .988,  .894,  .839}75.2$_{\pm3.9}$ & \cellcolor[rgb]{ .988,  .894,  .839}69.6$_{\pm2.5}$ & \cellcolor[rgb]{ .988,  .894,  .839}76.0$_{\pm0.9}$ & \cellcolor[rgb]{ .988,  .894,  .839}90.0$_{\pm0.8}$ & \cellcolor[rgb]{ .988,  .894,  .839}87.3$_{\pm2.8}$ & \cellcolor[rgb]{ .988,  .894,  .839}88.9$_{\pm1.8}$ & \cellcolor[rgb]{ .988,  .894,  .839}88.4$_{\pm1.0}$ & \cellcolor[rgb]{ .988,  .894,  .839}97.6$_{\pm0.3}$ \\
    \ER &\cellcolor[rgb]{ .988,  .894,  .839}22.7$_{\pm1.6}$ & \cellcolor[rgb]{ .988,  .894,  .839}15.5$_{\pm2.8}$ & \cellcolor[rgb]{ .949,  .949,  .949}17.8$_{\pm3.4}$ & \cellcolor[rgb]{ .949,  .949,  .949}24.1$_{\pm8.1}$ & \cellcolor[rgb]{ .988,  .894,  .839}54.6$_{\pm3.3}$ & \cellcolor[rgb]{ .988,  .894,  .839}42.7$_{\pm6.9}$ & \cellcolor[rgb]{ .988,  .894,  .839}44.1$_{\pm8.0}$ & \cellcolor[rgb]{ .949,  .949,  .949}55.0$_{\pm15.9}$ & \cellcolor[rgb]{ .988,  .894,  .839}77.2$_{\pm3.0}$ & \cellcolor[rgb]{ .988,  .894,  .839}67.2$_{\pm8.3}$ & \cellcolor[rgb]{ .988,  .894,  .839}66.5$_{\pm9.6}$ & \cellcolor[rgb]{ .929,  .929,  .929}76.6$_{\pm16.7}$ \\
     \SAM & \cellcolor[rgb]{ .988,  .894,  .839}35.0$_{\pm1.9}$ & \cellcolor[rgb]{ .988,  .894,  .839}25.2$_{\pm1.8}$ & \cellcolor[rgb]{ .988,  .894,  .839}36.9$_{\pm1.3}$ & \cellcolor[rgb]{ .973,  .796,  .678}56.7$_{\pm0.2}$ & \cellcolor[rgb]{ .988,  .894,  .839}74.7$_{\pm3.1}$ & \cellcolor[rgb]{ .988,  .894,  .839}65.7$_{\pm3.5}$ & \cellcolor[rgb]{ .988,  .894,  .839}76.5$_{\pm1.9}$ & \cellcolor[rgb]{ .973,  .796,  .678}94.0$_{\pm0.4}$ & \cellcolor[rgb]{ .988,  .894,  .839}92.3$_{\pm1.5}$ & \cellcolor[rgb]{ .988,  .894,  .839}90.9$_{\pm2.2}$ & \cellcolor[rgb]{ .973,  .796,  .678}93.4$_{\pm0.3}$ & \cellcolor[rgb]{ .988,  .894,  .839}99.5$_{\pm0.2}$ \\ \hline
     \SAM\&\IR & \cellcolor[rgb]{ .988,  .894,  .839}28.6$_{\pm0.2}$ & \cellcolor[rgb]{ .988,  .894,  .839}23.5$_{\pm0.2}$ & \cellcolor[rgb]{ .988,  .894,  .839}25.0$_{\pm1.2}$ & \cellcolor[rgb]{ .949,  .949,  .949}31.4$_{\pm0.7}$ & \cellcolor[rgb]{ .973,  .796,  .678}81.8$_{\pm0.0}$ & \cellcolor[rgb]{ .973,  .796,  .678}75.7$_{\pm0.2}$ & \cellcolor[rgb]{ .973,  .796,  .678}77.6$_{\pm1.1}$ & \cellcolor[rgb]{ .988,  .894,  .839}84.8$_{\pm0.6}$ & \cellcolor[rgb]{ .957,  .69,  .518}99.2$_{\pm0.1}$ & \cellcolor[rgb]{ .957,  .69,  .518}98.7$_{\pm0.1}$ & \cellcolor[rgb]{ .957,  .69,  .518}98.6$_{\pm0.4}$ & \cellcolor[rgb]{ .973,  .796,  .678}99.6$_{\pm0.1}$ \\     
     \SAM\&\JR & \cellcolor[rgb]{ .957,  .69,  .518}44.2$_{\pm1.5}$ & \cellcolor[rgb]{ .957,  .69,  .518}32.9$_{\pm1.1}$ & \cellcolor[rgb]{ .957,  .69,  .518}46.1$_{\pm1.3}$ & \cellcolor[rgb]{ .957,  .69,  .518}65.7$_{\pm0.3}$ & \cellcolor[rgb]{ .957,  .69,  .518}85.5$_{\pm1.4}$ & \cellcolor[rgb]{ .957,  .69,  .518}77.6$_{\pm1.7}$ & \cellcolor[rgb]{ .957,  .69,  .518}87.3$_{\pm1.4}$ & \cellcolor[rgb]{ .957,  .69,  .518}97.6$_{\pm0.1}$ & \cellcolor[rgb]{ .973,  .796,  .678}98.3$_{\pm0.5}$ & \cellcolor[rgb]{ .973,  .796,  .678}97.5$_{\pm0.4}$ & \cellcolor[rgb]{ .957,  .69,  .518}98.6$_{\pm0.4}$ & \cellcolor[rgb]{ .957,  .69,  .518}99.9$_{\pm0.0}$ \\
    \bottomrule
    \end{tabular}%
    }
\label{tab:cifar10-ASR}
\vspace{-1mm}
\end{table*}%

Based on the above analyses,  we can infer that the ``transferability circuit'' of \AT ~may primarily arise from  the trade-off effect between smoothness and  similarity, given the restriction of  surrogates with  relatively good  accuracies  and  a highly effective attack strategy.
This inference is intuitively inspired by the implications regarding all the surrogate-related factors in \cref{theorem1}.

Meanwhile, we observe that the  ``transferability circuit'' of \AT ~generally exists, though it varies under different non-surrogate circumstances. For example, the optimal $\epsilon$ AT to achieve the best ASRs varies among datasets and perturbation budgets in \cref{fig:adv-transfer}. 
This suggests that for different non-surrogate factors, the extent to which model smoothness and gradient similarity contribute to transferability may vary.

Consequently, we make the conjecture that when generating  AEs against accurate surrogate models using effective strategies, the trade-off between  model smoothness and gradient similarity is strongly correlated with the transferability of those examples from the surrogate models.
To test this conjecture, we will investigate changes in model smoothness and gradient similarity alongside the corresponding transfer ASRs under various surrogate training mechanisms in \cref{sec:trade-off:da,sec:trade-off:gr}.
Furthermore, in \cref{analysis}, \textbf{Q1},  we provide a statistical analysis of the correlations between model smoothness, gradient similarity and transferability.
Moving forward, the primary concern at hand is to ascertain the rationale of this similarity deterioration in \AT, as well as how to generally avoid this and reach a better balance of smoothness and similarity for stronger transfer attacks.

\begin{table*}[!tp]
\setlength{\abovecaptionskip}{0cm}
\setlength{\belowcaptionskip}{-0cm}
\renewcommand\arraystretch{0.64}
  \centering
  \caption{Transfer ASRs of AEs on \textbf{ImageNette} using \textbf{PGD}.
  We plot this table in the same way as \cref{tab:cifar10-ASR}. In \%.
  }
  \tabcolsep=0.03cm
  \resizebox{\textwidth}{!}
  {
    \begin{tabular}{c|cccc|cccc|cccc}
    \toprule    \multicolumn{13}{c}{\textbf{\scriptsize Untargeted}} \\
\hline
\multicolumn{1}{l|}{} & \multicolumn{4}{c|}{\scriptsize 4/255}                                                                                                                & \multicolumn{4}{c|}{\scriptsize 8/255}                                                                                                                  & \multicolumn{4}{c}{\scriptsize 16/255}                                                                                                                 \\
\multicolumn{1}{l|}{} & \scriptsize VGG16                            & \scriptsize DenseNet121                      & \scriptsize MobileNetV2                      & \scriptsize Xception                         & \scriptsize VGG16                             & \scriptsize DenseNet121                      & \scriptsize MobileNetV2                       & \scriptsize Xception                         & \scriptsize VGG16                             & \scriptsize DenseNet121                      & \scriptsize MobileNetV2                      & \scriptsize Xception                         \\ \hline
\scriptsize \textit{ST}          & \scriptsize \cellcolor[HTML]{FFFFFF}$\textrm{11.2}_{\pm1.0}$& \scriptsize \cellcolor[HTML]{FFFFFF}$\textrm{17.5}_{\pm1.1}$& \scriptsize \cellcolor[HTML]{FFFFFF}$\textrm{7.6}_{\pm0.4}$ & \scriptsize \cellcolor[HTML]{FFFFFF}$\textrm{8.7}_{\pm0.5}$ & \scriptsize \cellcolor[HTML]{FFFFFF}$\textrm{28.7}_{\pm1.1}$ & \scriptsize \cellcolor[HTML]{FFFFFF}$\textrm{42.7}_{\pm3.5}$& \scriptsize \cellcolor[HTML]{FFFFFF}$\textrm{18.3}_{\pm0.7}$ & \scriptsize \cellcolor[HTML]{FFFFFF}$\textrm{18.5}_{\pm1.3}$& \scriptsize \cellcolor[HTML]{FFFFFF}$\textrm{61.7}_{\pm4.7}$ & \scriptsize \cellcolor[HTML]{FFFFFF}$\textrm{81.4}_{\pm4.6}$& \scriptsize \cellcolor[HTML]{FFFFFF}$\textrm{49.3}_{\pm2.7}$& \scriptsize \cellcolor[HTML]{FFFFFF}$\textrm{44.1}_{\pm5.1}$ \\ \hline
\scriptsize \textit{MU}, $\tau=1$     & \scriptsize \cellcolor[HTML]{F2F2F2}$\textrm{9.2}_{\pm1.6}$ & \scriptsize \cellcolor[HTML]{F2F2F2}$\textrm{12.7}_{\pm0.9}$& \scriptsize \cellcolor[HTML]{F2F2F2}$\textrm{7.2}_{\pm0.6}$ & \scriptsize \cellcolor[HTML]{F2F2F2}$\textrm{7.2}_{\pm0.2}$ & \scriptsize \cellcolor[HTML]{F2F2F2}$\textrm{20.7}_{\pm1.3}$ & \scriptsize \cellcolor[HTML]{F2F2F2}$\textrm{22.7}_{\pm1.3}$& \scriptsize \cellcolor[HTML]{F2F2F2}$\textrm{13.7}_{\pm1.5}$ & \scriptsize \cellcolor[HTML]{F2F2F2}$\textrm{12.7}_{\pm1.1}$& \scriptsize \cellcolor[HTML]{F2F2F2}$\textrm{43.9}_{\pm2.3}$ & \scriptsize \cellcolor[HTML]{F2F2F2}$\textrm{53.3}_{\pm4.1}$& \scriptsize \cellcolor[HTML]{F2F2F2}$\textrm{38.3}_{\pm3.1}$& \scriptsize \cellcolor[HTML]{F2F2F2}$\textrm{29.3}_{\pm2.9}$ \\
\scriptsize \textit{MU}, $\tau=5$     & \scriptsize \cellcolor[HTML]{F2F2F2}$\textrm{6.7}_{\pm1.1}$ & \scriptsize \cellcolor[HTML]{F2F2F2}$\textrm{6.7}_{\pm0.9}$ & \scriptsize \cellcolor[HTML]{F2F2F2}$\textrm{5.1}_{\pm0.5}$ & \scriptsize \cellcolor[HTML]{F2F2F2}$\textrm{5.7}_{\pm0.5}$ & \scriptsize \cellcolor[HTML]{F2F2F2}$\textrm{11.5}_{\pm0.5}$ & \scriptsize \cellcolor[HTML]{F2F2F2}$\textrm{12.0}_{\pm0.6}$& \scriptsize \cellcolor[HTML]{F2F2F2}$\textrm{8.1}_{\pm0.7}$  & \scriptsize \cellcolor[HTML]{F2F2F2}$\textrm{8.1}_{\pm0.1}$ & \scriptsize \cellcolor[HTML]{F2F2F2}$\textrm{29.3}_{\pm0.7}$ & \scriptsize \cellcolor[HTML]{F2F2F2}$\textrm{29.2}_{\pm2.4}$& \scriptsize \cellcolor[HTML]{F2F2F2}$\textrm{24.1}_{\pm1.3}$& \scriptsize \cellcolor[HTML]{F2F2F2}$\textrm{18.8}_{\pm0.2}$ \\
\scriptsize \textit{CM}, $\tau=1$     & \scriptsize \cellcolor[HTML]{F2F2F2}$\textrm{7.8}_{\pm0.6}$ & \scriptsize \cellcolor[HTML]{F2F2F2}$\textrm{10.6}_{\pm1.0}$  & \scriptsize \cellcolor[HTML]{F2F2F2}$\textrm{6.1}_{\pm0.7}$ & \scriptsize \cellcolor[HTML]{F2F2F2}$\textrm{7.1}_{\pm0.5}$ & \scriptsize \cellcolor[HTML]{F2F2F2}$\textrm{16.5}_{\pm1.3}$ & \scriptsize \cellcolor[HTML]{F2F2F2}$\textrm{18.1}_{\pm1.7}$& \scriptsize \cellcolor[HTML]{F2F2F2}$\textrm{10.3}_{\pm0.3}$ & \scriptsize \cellcolor[HTML]{F2F2F2}$\textrm{10.3}_{\pm0.1}$& \scriptsize \cellcolor[HTML]{F2F2F2}$\textrm{35.8}_{\pm1.0}$  & \scriptsize \cellcolor[HTML]{F2F2F2}$\textrm{42.6}_{\pm2.4}$& \scriptsize \cellcolor[HTML]{F2F2F2}$\textrm{29.6}_{\pm2.0}$  & \scriptsize \cellcolor[HTML]{F2F2F2}$\textrm{21.5}_{\pm0.5}$ \\
\scriptsize \textit{CM}, $\tau=5$     & \scriptsize \cellcolor[HTML]{F2F2F2}$\textrm{5.5}_{\pm0.5}$ & \scriptsize \cellcolor[HTML]{F2F2F2}$\textrm{6.3}_{\pm0.1}$ & \scriptsize \cellcolor[HTML]{F2F2F2}$\textrm{4.2}_{\pm0.4}$ & \scriptsize \cellcolor[HTML]{F2F2F2}$\textrm{4.6}_{\pm0.4}$ & \scriptsize \cellcolor[HTML]{F2F2F2}$\textrm{8.9}_{\pm0.9}$  & \scriptsize \cellcolor[HTML]{F2F2F2}$\textrm{9.3}_{\pm0.3}$ & \scriptsize \cellcolor[HTML]{F2F2F2}$\textrm{6.7}_{\pm0.3}$  & \scriptsize \cellcolor[HTML]{F2F2F2}$\textrm{6.7}_{\pm0.5}$ & \scriptsize \cellcolor[HTML]{F2F2F2}$\textrm{19.3}_{\pm0.5}$ & \scriptsize \cellcolor[HTML]{F2F2F2}$\textrm{18.3}_{\pm1.1}$& \scriptsize \cellcolor[HTML]{F2F2F2}$\textrm{17.0}_{\pm0.6}$  & \scriptsize \cellcolor[HTML]{F2F2F2}$\textrm{11.7}_{\pm0.5}$ \\
\scriptsize \textit{CO}, $\tau=1$     & \scriptsize \cellcolor[HTML]{FCE4D6}$\textrm{11.8}_{\pm0.4}$& \scriptsize \cellcolor[HTML]{F2F2F2}$\textrm{16.7}_{\pm1.1}$& \scriptsize \cellcolor[HTML]{FCE4D6}$\textrm{8.1}_{\pm0.5}$ & \scriptsize $\textrm{8.7}_{\pm0.3}$                         & \scriptsize \cellcolor[HTML]{F2F2F2}$\textrm{28.3}_{\pm1.1}$                         & \scriptsize \cellcolor[HTML]{F2F2F2}$\textrm{40.5}_{\pm2.9}$& \scriptsize $\textrm{18.5}_{\pm0.7}$ & \scriptsize $\textrm{18.4}_{\pm1.4}$                        & \scriptsize $\textrm{61.8}_{\pm1.4}$ & \scriptsize \cellcolor[HTML]{F2F2F2}$\textrm{79.8}_{\pm2.4}$& \scriptsize \cellcolor[HTML]{FCE4D6}$\textrm{50.9}_{\pm1.9}$& \scriptsize \cellcolor[HTML]{F2F2F2}$\textrm{43.8}_{\pm0.6}$ \\
\scriptsize \textit{CO}, $\tau=5$     & \scriptsize \cellcolor[HTML]{F2F2F2}$\textrm{10.3}_{\pm0.7}$& \scriptsize \cellcolor[HTML]{F2F2F2}$\textrm{14.3}_{\pm0.3}$& \scriptsize \cellcolor[HTML]{F2F2F2}$\textrm{7.3}_{\pm0.9}$ & \scriptsize \cellcolor[HTML]{F2F2F2}$\textrm{7.7}_{\pm0.7}$ & \scriptsize \cellcolor[HTML]{F2F2F2}$\textrm{25.1}_{\pm1.7}$ & \scriptsize \cellcolor[HTML]{F2F2F2}$\textrm{34.5}_{\pm2.3}$& \scriptsize \cellcolor[HTML]{F2F2F2}$\textrm{17.3}_{\pm1.9}$ & \scriptsize \cellcolor[HTML]{F2F2F2}$\textrm{15.9}_{\pm1.1}$& \scriptsize \cellcolor[HTML]{F2F2F2}$\textrm{56.4}_{\pm4.2}$ & \scriptsize \cellcolor[HTML]{F2F2F2}$\textrm{75.0}_{\pm3.6}$  & \scriptsize \cellcolor[HTML]{F2F2F2}$\textrm{46.3}_{\pm0.9}$& \scriptsize \cellcolor[HTML]{F2F2F2}$\textrm{40.1}_{\pm2.1}$ \\
\scriptsize \textit{LS}, $\tau=1$     & \scriptsize \cellcolor[HTML]{F2F2F2}$\textrm{6.8}_{\pm0.6}$ & \scriptsize \cellcolor[HTML]{F2F2F2}$\textrm{8.7}_{\pm0.3}$ & \scriptsize \cellcolor[HTML]{F2F2F2}$\textrm{5.2}_{\pm0.4}$ & \scriptsize \cellcolor[HTML]{F2F2F2}$\textrm{5.9}_{\pm0.9}$ & \scriptsize \cellcolor[HTML]{F2F2F2}$\textrm{13.2}_{\pm0.6}$ & \scriptsize \cellcolor[HTML]{F2F2F2}$\textrm{14.7}_{\pm0.1}$& \scriptsize \cellcolor[HTML]{F2F2F2}$\textrm{9.5}_{\pm0.1}$  & \scriptsize \cellcolor[HTML]{F2F2F2}$\textrm{8.5}_{\pm0.5}$ & \scriptsize \cellcolor[HTML]{F2F2F2}$\textrm{27.6}_{\pm0.2}$ & \scriptsize \cellcolor[HTML]{F2F2F2}$\textrm{35.5}_{\pm0.7}$& \scriptsize \cellcolor[HTML]{F2F2F2}$\textrm{24.8}_{\pm0.8}$& \scriptsize \cellcolor[HTML]{F2F2F2}$\textrm{18.0}_{\pm0.6}$ \\
\scriptsize \textit{LS}, $\tau=5$     & \scriptsize \cellcolor[HTML]{F2F2F2}$\textrm{6.1}_{\pm0.3}$ & \scriptsize \cellcolor[HTML]{F2F2F2}$\textrm{7.1}_{\pm0.5}$ & \scriptsize \cellcolor[HTML]{F2F2F2}$\textrm{4.7}_{\pm0.5}$ & \scriptsize \cellcolor[HTML]{F2F2F2}$\textrm{5.5}_{\pm0.5}$ & \scriptsize \cellcolor[HTML]{F2F2F2}$\textrm{9.1}_{\pm1.1}$  & \scriptsize \cellcolor[HTML]{F2F2F2}$\textrm{10.2}_{\pm0.8}$& \scriptsize \cellcolor[HTML]{F2F2F2}$\textrm{7.2}_{\pm0.2}$  & \scriptsize \cellcolor[HTML]{F2F2F2}$\textrm{7.8}_{\pm0.2}$ & \scriptsize \cellcolor[HTML]{F2F2F2}$\textrm{20.0}_{\pm1.2}$   & \scriptsize \cellcolor[HTML]{F2F2F2}$\textrm{23.1}_{\pm1.3}$& \scriptsize \cellcolor[HTML]{F2F2F2}$\textrm{18.6}_{\pm1.6}$& \scriptsize \cellcolor[HTML]{F2F2F2}$\textrm{14.4}_{\pm0.8}$ \\ \hline
\scriptsize \textit{AT}          & \scriptsize \cellcolor[HTML]{FCE4D6}$\textrm{14.7}_{\pm2.1}$& \scriptsize \cellcolor[HTML]{FCE4D6}$\textrm{20.5}_{\pm1.5}$& \scriptsize \cellcolor[HTML]{FCE4D6}$\textrm{15.3}_{\pm1.3}$& \scriptsize \cellcolor[HTML]{FCE4D6}$\textrm{15.7}_{\pm0.9}$& \scriptsize \cellcolor[HTML]{FCE4D6}$\textrm{52.1}_{\pm6.3}$ & \scriptsize \cellcolor[HTML]{FCE4D6}$\textrm{68.7}_{\pm2.9}$& \scriptsize \cellcolor[HTML]{FCE4D6}$\textrm{59.3}_{\pm7.1}$ & \scriptsize \cellcolor[HTML]{FCE4D6}$\textrm{52.3}_{\pm5.1}$& \scriptsize \cellcolor[HTML]{FCE4D6}$\textrm{96.5}_{\pm1.5}$ & \scriptsize \cellcolor[HTML]{FCE4D6}$\textrm{98.8}_{\pm0.4}$& \scriptsize \cellcolor[HTML]{F8CBAD}$\textrm{97.5}_{\pm0.5}$& \scriptsize \cellcolor[HTML]{F8CBAD}$\textrm{95.7}_{\pm0.9}$ \\ \hline
\scriptsize \textit{IR}          & \scriptsize \cellcolor[HTML]{FCE4D6}$\textrm{21.0}_{\pm6.6}$& \scriptsize \cellcolor[HTML]{FCE4D6}$\textrm{31.7}_{\pm6.9}$& \scriptsize \cellcolor[HTML]{FCE4D6}$\textrm{18.5}_{\pm5.1}$& \scriptsize \cellcolor[HTML]{FCE4D6}$\textrm{19.5}_{\pm4.3}$& \scriptsize \cellcolor[HTML]{FCE4D6}$\textrm{62.5}_{\pm14.3}$& \scriptsize \cellcolor[HTML]{FCE4D6}$\textrm{85.6}_{\pm9.0}$& \scriptsize \cellcolor[HTML]{FCE4D6}$\textrm{65.2}_{\pm16.4}$& \scriptsize \cellcolor[HTML]{FCE4D6}$\textrm{58.4}_{\pm14.0}$ & \scriptsize \cellcolor[HTML]{FCE4D6}$\textrm{94.7}_{\pm6.9}$ & \scriptsize \cellcolor[HTML]{F8CBAD}$\textrm{99.5}_{\pm0.3}$& \scriptsize \cellcolor[HTML]{FCE4D6}$\textrm{95.7}_{\pm5.3}$& \scriptsize \cellcolor[HTML]{FCE4D6}$\textrm{94.6}_{\pm5.0}$   \\
\scriptsize \textit{JR}          & \scriptsize \cellcolor[HTML]{FCE4D6}$\textrm{26.3}_{\pm2.1}$& \scriptsize \cellcolor[HTML]{FCE4D6}$\textrm{37.3}_{\pm0.7}$& \scriptsize \cellcolor[HTML]{FCE4D6}$\textrm{21.7}_{\pm1.3}$& \scriptsize \cellcolor[HTML]{FCE4D6}$\textrm{20.5}_{\pm1.3}$& \scriptsize \cellcolor[HTML]{FCE4D6}$\textrm{68.5}_{\pm2.1}$ & \scriptsize \cellcolor[HTML]{FCE4D6}$\textrm{87.0}_{\pm0.2}$& \scriptsize \cellcolor[HTML]{FCE4D6}$\textrm{67.6}_{\pm3.6}$ & \scriptsize \cellcolor[HTML]{FCE4D6}$\textrm{56.0}_{\pm5.6}$& \scriptsize \cellcolor[HTML]{FCE4D6}$\textrm{95.0}_{\pm1.4}$ & \scriptsize \cellcolor[HTML]{F8CBAD}$\textrm{99.5}_{\pm0.1}$& \scriptsize \cellcolor[HTML]{FCE4D6}$\textrm{94.7}_{\pm1.1}$& \scriptsize \cellcolor[HTML]{FCE4D6}$\textrm{89.8}_{\pm3.0}$   \\
\scriptsize \textit{ER}          & \scriptsize \cellcolor[HTML]{FCE4D6}$\textrm{13.2}_{\pm2.8}$& \scriptsize \cellcolor[HTML]{FCE4D6}$\textrm{20.9}_{\pm1.7}$& \scriptsize \cellcolor[HTML]{FCE4D6}$\textrm{8.9}_{\pm1.5}$ & \scriptsize \cellcolor[HTML]{FCE4D6}$\textrm{9.6}_{\pm1.0}$ & \scriptsize \cellcolor[HTML]{FCE4D6}$\textrm{33.3}_{\pm9.9}$ & \scriptsize \cellcolor[HTML]{FCE4D6}$\textrm{51.1}_{\pm8.7}$& \scriptsize \cellcolor[HTML]{FCE4D6}$\textrm{20.7}_{\pm3.7}$ & \scriptsize \cellcolor[HTML]{FCE4D6}$\textrm{20.5}_{\pm3.3}$& \scriptsize \cellcolor[HTML]{FCE4D6}$\textrm{67.5}_{\pm15.9}$& \scriptsize \cellcolor[HTML]{FCE4D6}$\textrm{89.3}_{\pm6.7}$& \scriptsize \cellcolor[HTML]{FCE4D6}$\textrm{54.5}_{\pm9.1}$& \scriptsize \cellcolor[HTML]{FCE4D6}$\textrm{50.1}_{\pm8.7}$ \\
\scriptsize \textit{SAM}         & \scriptsize \cellcolor[HTML]{FCE4D6}$\textrm{22.8}_{\pm1.4}$& \scriptsize \cellcolor[HTML]{FCE4D6}$\textrm{29.7}_{\pm1.5}$& \scriptsize \cellcolor[HTML]{FCE4D6}$\textrm{12.8}_{\pm0.8}$& \scriptsize \cellcolor[HTML]{FCE4D6}$\textrm{12.8}_{\pm0.0}$& \scriptsize \cellcolor[HTML]{FCE4D6}$\textrm{58.7}_{\pm2.7}$ & \scriptsize \cellcolor[HTML]{FCE4D6}$\textrm{72.8}_{\pm2.0}$  & \scriptsize \cellcolor[HTML]{FCE4D6}$\textrm{42.3}_{\pm1.1}$ & \scriptsize \cellcolor[HTML]{FCE4D6}$\textrm{34.7}_{\pm1.7}$& \scriptsize \cellcolor[HTML]{FCE4D6}$\textrm{91.4}_{\pm2.0}$   & \scriptsize \cellcolor[HTML]{FCE4D6}$\textrm{97.1}_{\pm0.3}$& \scriptsize \cellcolor[HTML]{FCE4D6}$\textrm{81.0}_{\pm1.8}$& \scriptsize \cellcolor[HTML]{FCE4D6}$\textrm{75.3}_{\pm1.7}$ \\ \hline
\scriptsize \textit{SAM}\&\textit{IR}         & \scriptsize \cellcolor[HTML]{F8CBAD}$\textrm{26.9}_{\pm5.5}$& \scriptsize \cellcolor[HTML]{F8CBAD}$\textrm{39.8}_{\pm4.8}$& \scriptsize \cellcolor[HTML]{F8CBAD}$\textrm{24.3}_{\pm4.7}$& \scriptsize \cellcolor[HTML]{F8CBAD}$\textrm{23.2}_{\pm4.0}$& \scriptsize \cellcolor[HTML]{F8CBAD}$\textrm{72.5}_{\pm10.9}$ & \scriptsize \cellcolor[HTML]{F8CBAD}$\textrm{92.6}_{\pm2.8}$  & \scriptsize \cellcolor[HTML]{F4B084}$\textrm{78.5}_{\pm9.5}$ & \scriptsize \cellcolor[HTML]{F4B084}$\textrm{68.9}_{\pm10.3}$& \scriptsize \cellcolor[HTML]{F4B084}$\textrm{98.1}_{\pm2.1}$   & \scriptsize \cellcolor[HTML]{F4B084}$\textrm{99.8}_{\pm0.2}$& \scriptsize \cellcolor[HTML]{F4B084}$\textrm{98.7}_{\pm1.5}$& \scriptsize \cellcolor[HTML]{F4B084}$\textrm{97.7}_{\pm2.3}$ \\
\scriptsize \textit{SAM}\&\textit{JR}         & \scriptsize \cellcolor[HTML]{F4B084}$\textrm{32.7}_{\pm4.3}$& \scriptsize \cellcolor[HTML]{F4B084}$\textrm{46.3}_{\pm3.7}$& \scriptsize \cellcolor[HTML]{F4B084}$\textrm{28.1}_{\pm4.5}$& \scriptsize \cellcolor[HTML]{F4B084}$\textrm{24.9}_{\pm4.1}$& \scriptsize \cellcolor[HTML]{F4B084}$\textrm{76.6}_{\pm5.6}$ & \scriptsize \cellcolor[HTML]{F4B084}$\textrm{93.6}_{\pm1.2}$  & \scriptsize \cellcolor[HTML]{F8CBAD}$\textrm{73.1}_{\pm7.7}$ & \scriptsize \cellcolor[HTML]{F8CBAD}$\textrm{65.9}_{\pm7.7}$& \scriptsize \cellcolor[HTML]{F8CBAD}$\textrm{96.7}_{\pm7.3}$   & \scriptsize \cellcolor[HTML]{F4B084}$\textrm{99.8}_{\pm0.2}$& \scriptsize \cellcolor[HTML]{FCE4D6}$\textrm{96.6}_{\pm1.6}$& \scriptsize \cellcolor[HTML]{FCE4D6}$\textrm{93.9}_{\pm4.1}$ \\
\toprule   \multicolumn{13}{c}{\textbf{\scriptsize Targeted}} \\
\hline
\scriptsize \textit{ST}          & \scriptsize \cellcolor[HTML]{FFFFFF}2.1$_{\pm0.3}$& \scriptsize \cellcolor[HTML]{FFFFFF}4.1$_{\pm0.3}$ & \scriptsize \cellcolor[HTML]{FFFFFF}0.6$_{\pm0.4}$& \scriptsize \cellcolor[HTML]{FFFFFF}2.0$_{\pm0.2}$& \scriptsize \cellcolor[HTML]{FFFFFF}10.4$_{\pm0.8}$ & \scriptsize \cellcolor[HTML]{FFFFFF}19.6$_{\pm2.8}$ & \scriptsize \cellcolor[HTML]{FFFFFF}4.2$_{\pm0.4}$  & \scriptsize \cellcolor[HTML]{FFFFFF}5.7$_{\pm1.7}$  & \scriptsize \cellcolor[HTML]{FFFFFF}33.6$_{\pm3.8}$ & \scriptsize \cellcolor[HTML]{FFFFFF}60.3$_{\pm7.5}$ & \scriptsize \cellcolor[HTML]{FFFFFF}19.5$_{\pm2.1}$ & \scriptsize \cellcolor[HTML]{FFFFFF}19.7$_{\pm3.1}$  \\ \hline
\scriptsize \textit{MU}, $\tau=1$     & \scriptsize \cellcolor[HTML]{F2F2F2}1.3$_{\pm0.3}$& \scriptsize \cellcolor[HTML]{F2F2F2}2.3$_{\pm0.3}$ & \scriptsize \cellcolor[HTML]{FFFFFF}0.4$_{\pm0.2}$& \scriptsize \cellcolor[HTML]{F2F2F2}1.0$_{\pm0.2}$& \scriptsize \cellcolor[HTML]{F2F2F2}5.1$_{\pm0.7}$  & \scriptsize \cellcolor[HTML]{F2F2F2}7.3$_{\pm1.1}$  & \scriptsize \cellcolor[HTML]{F2F2F2}1.5$_{\pm0.3}$  & \scriptsize \cellcolor[HTML]{F2F2F2}2.5$_{\pm0.1}$  & \scriptsize \cellcolor[HTML]{F2F2F2}15.7$_{\pm1.9}$ & \scriptsize \cellcolor[HTML]{F2F2F2}25.1$_{\pm4.1}$ & \scriptsize \cellcolor[HTML]{F2F2F2}9.3$_{\pm1.5}$  & \scriptsize \cellcolor[HTML]{F2F2F2}8.8$_{\pm2.4}$   \\
\scriptsize \textit{MU}, $\tau=5$     & \scriptsize \cellcolor[HTML]{F2F2F2}0.9$_{\pm0.1}$& \scriptsize \cellcolor[HTML]{F2F2F2}0.6$_{\pm0.0}$ & \scriptsize \cellcolor[HTML]{F2F2F2}0.1$_{\pm0.1}$& \scriptsize \cellcolor[HTML]{F2F2F2}0.5$_{\pm0.1}$& \scriptsize \cellcolor[HTML]{F2F2F2}1.3$_{\pm0.3}$  & \scriptsize \cellcolor[HTML]{F2F2F2}1.1$_{\pm0.1}$  & \scriptsize \cellcolor[HTML]{F2F2F2}0.4$_{\pm0.4}$  & \scriptsize \cellcolor[HTML]{F2F2F2}1.0$_{\pm0.2}$  & \scriptsize \cellcolor[HTML]{F2F2F2}3.8$_{\pm0.8}$  & \scriptsize \cellcolor[HTML]{F2F2F2}3.7$_{\pm0.7}$  & \scriptsize \cellcolor[HTML]{F2F2F2}1.7$_{\pm0.1}$  & \scriptsize \cellcolor[HTML]{F2F2F2}2.8$_{\pm0.6}$   \\
\scriptsize \textit{CM}, $\tau=1$     & \scriptsize \cellcolor[HTML]{F2F2F2}1.1$_{\pm0.5}$& \scriptsize \cellcolor[HTML]{F2F2F2}1.4$_{\pm0.2}$ & \scriptsize \cellcolor[HTML]{F2F2F2}0.2$_{\pm0.2}$& \scriptsize \cellcolor[HTML]{F2F2F2}0.7$_{\pm0.3}$& \scriptsize \cellcolor[HTML]{F2F2F2}2.3$_{\pm0.3}$  & \scriptsize \cellcolor[HTML]{F2F2F2}3.1$_{\pm0.5}$  & \scriptsize \cellcolor[HTML]{F2F2F2}0.7$_{\pm0.1}$  & \scriptsize \cellcolor[HTML]{F2F2F2}1.7$_{\pm0.5}$ & \scriptsize \cellcolor[HTML]{F2F2F2}6.5$_{\pm0.5}$  & \scriptsize \cellcolor[HTML]{F2F2F2}9.4$_{\pm0.4}$  & \scriptsize \cellcolor[HTML]{F2F2F2}3.7$_{\pm0.3}$  & \scriptsize \cellcolor[HTML]{F2F2F2}4.9$_{\pm0.3}$   \\
\scriptsize \textit{CM}, $\tau=5$     & \scriptsize \cellcolor[HTML]{F2F2F2}0.7$_{\pm0.1}$& \scriptsize \cellcolor[HTML]{F2F2F2}0.8$_{\pm0.0}$ & \scriptsize \cellcolor[HTML]{F2F2F2}0.1$_{\pm0.1}$& \scriptsize \cellcolor[HTML]{F2F2F2}0.4$_{\pm0.2}$& \scriptsize \cellcolor[HTML]{F2F2F2}0.9$_{\pm0.3}$  & \scriptsize \cellcolor[HTML]{F2F2F2}1.0$_{\pm0.2}$  & \scriptsize \cellcolor[HTML]{F2F2F2}0.2$_{\pm0.2}$  & \scriptsize \cellcolor[HTML]{F2F2F2}0.9$_{\pm0.1}$  & \scriptsize \cellcolor[HTML]{F2F2F2}2.3$_{\pm0.3}$  & \scriptsize \cellcolor[HTML]{F2F2F2}2.5$_{\pm0.1}$  & \scriptsize \cellcolor[HTML]{F2F2F2}1.4$_{\pm0.4}$  & \scriptsize \cellcolor[HTML]{F2F2F2}1.6$_{\pm0.2}$   \\
\scriptsize \textit{CO}, $\tau=1$     & \scriptsize \cellcolor[HTML]{FFFFFF}2.3$_{\pm0.5}$& \scriptsize \cellcolor[HTML]{F2F2F2}3.4$_{\pm0.6}$ & \scriptsize \cellcolor[HTML]{FFFFFF}0.4$_{\pm0.2}$& \scriptsize \cellcolor[HTML]{F2F2F2}1.7$_{\pm0.3}$& \scriptsize \cellcolor[HTML]{F2F2F2}9.3$_{\pm1.3}$  & \scriptsize \cellcolor[HTML]{F2F2F2}18.2$_{\pm3.0}$ & \scriptsize \cellcolor[HTML]{F2F2F2}3.2$_{\pm1.0}$  & \scriptsize \cellcolor[HTML]{F2F2F2}4.9$_{\pm0.5}$  & \scriptsize \cellcolor[HTML]{F2F2F2}31.0$_{\pm2.4}$ & \scriptsize \cellcolor[HTML]{FCE4D6}62.3$_{\pm6.9}$ & \scriptsize \cellcolor[HTML]{F2F2F2}18.2$_{\pm1.6}$ & \scriptsize \cellcolor[HTML]{F2F2F2}17.4$_{\pm1.8}$  \\
\scriptsize \textit{CO}, $\tau=5$     & \scriptsize \cellcolor[HTML]{FFFFFF}1.9$_{\pm0.1}$& \scriptsize \cellcolor[HTML]{F2F2F2}3.1$_{\pm0.9}$ & \scriptsize 0.6$_{\pm0.2}$                        & \scriptsize \cellcolor[HTML]{F2F2F2}1.2$_{\pm0.2}$& \scriptsize \cellcolor[HTML]{F2F2F2}8.8$_{\pm1.4}$  & \scriptsize \cellcolor[HTML]{F2F2F2}14.9$_{\pm0.9}$ & \scriptsize \cellcolor[HTML]{F2F2F2}3.8$_{\pm1.0}$  & \scriptsize \cellcolor[HTML]{F2F2F2}4.3$_{\pm1.1}$  & \scriptsize \cellcolor[HTML]{F2F2F2}28.3$_{\pm2.9}$ & \scriptsize \cellcolor[HTML]{F2F2F2}52.5$_{\pm2.1}$ & \scriptsize \cellcolor[HTML]{F2F2F2}16.2$_{\pm3.0}$   & \scriptsize \cellcolor[HTML]{F2F2F2}16.9$_{\pm1.3}$  \\
\scriptsize \textit{LS}, $\tau=1$     & \scriptsize \cellcolor[HTML]{F2F2F2}0.9$_{\pm0.1}$& \scriptsize \cellcolor[HTML]{F2F2F2}1.5$_{\pm0.5}$ & \scriptsize \cellcolor[HTML]{FFFFFF}0.4$_{\pm0.0}$& \scriptsize \cellcolor[HTML]{F2F2F2}0.7$_{\pm0.1}$& \scriptsize \cellcolor[HTML]{F2F2F2}1.6$_{\pm0.0}$  & \scriptsize \cellcolor[HTML]{F2F2F2}2.5$_{\pm0.1}$  & \scriptsize \cellcolor[HTML]{F2F2F2}0.5$_{\pm0.1}$  & \scriptsize \cellcolor[HTML]{F2F2F2}1.4$_{\pm0.2}$  & \scriptsize \cellcolor[HTML]{F2F2F2}4.0$_{\pm0.2}$  & \scriptsize \cellcolor[HTML]{F2F2F2}6.3$_{\pm0.7}$  & \scriptsize \cellcolor[HTML]{F2F2F2}3.2$_{\pm0.6}$  & \scriptsize \cellcolor[HTML]{F2F2F2}2.8$_{\pm0.2}$   \\
\scriptsize \textit{LS}, $\tau=5$     & \scriptsize \cellcolor[HTML]{F2F2F2}0.7$_{\pm0.1}$& \scriptsize \cellcolor[HTML]{F2F2F2}1.2$_{\pm0.0}$ & \scriptsize \cellcolor[HTML]{F2F2F2}0.3$_{\pm0.1}$& \scriptsize \cellcolor[HTML]{F2F2F2}0.5$_{\pm0.3}$& \scriptsize \cellcolor[HTML]{F2F2F2}1.2$_{\pm0.0}$  & \scriptsize \cellcolor[HTML]{F2F2F2}1.9$_{\pm0.5}$  & \scriptsize \cellcolor[HTML]{F2F2F2}0.3$_{\pm0.1}$  & \scriptsize \cellcolor[HTML]{F2F2F2}1.0$_{\pm0.2}$  & \scriptsize \cellcolor[HTML]{F2F2F2}2.7$_{\pm0.9}$  & \scriptsize \cellcolor[HTML]{F2F2F2}3.7$_{\pm0.9}$  & \scriptsize \cellcolor[HTML]{F2F2F2}2.2$_{\pm0.4}$  & \scriptsize \cellcolor[HTML]{F2F2F2}2.3$_{\pm0.7}$   \\ \hline
\scriptsize \textit{AT}          & \scriptsize 2.1$_{\pm0.9}$                        & \scriptsize 4.0$_{\pm0.8}$                         & \scriptsize \cellcolor[HTML]{FCE4D6}1.1$_{\pm0.3}$& \scriptsize \cellcolor[HTML]{FCE4D6}2.6$_{\pm0.2}$& \scriptsize \cellcolor[HTML]{FCE4D6}22.5$_{\pm3.1}$ & \scriptsize \cellcolor[HTML]{FCE4D6}39.3$_{\pm2.9}$ & \scriptsize \cellcolor[HTML]{FCE4D6}26.2$_{\pm5.4}$ & \scriptsize \cellcolor[HTML]{FCE4D6}25.1$_{\pm4.1}$ & \scriptsize \cellcolor[HTML]{FCE4D6}84.7$_{\pm3.7}$ & \scriptsize \cellcolor[HTML]{FCE4D6}94.1$_{\pm1.5}$ & \scriptsize \cellcolor[HTML]{F8CBAD}88.3$_{\pm2.2}$ & \scriptsize \cellcolor[HTML]{F8CBAD}89.3$_{\pm2.5}$  \\ \hline
\scriptsize \textit{IR}          & \scriptsize \cellcolor[HTML]{FCE4D6}4.2$_{\pm2.0}$& \scriptsize \cellcolor[HTML]{FCE4D6}9.9$_{\pm3.7}$ & \scriptsize \cellcolor[HTML]{FCE4D6}3.4$_{\pm1.6}$& \scriptsize \cellcolor[HTML]{FCE4D6}3.3$_{\pm1.1}$& \scriptsize \cellcolor[HTML]{FCE4D6}33.6$_{\pm13.4}$& \scriptsize \cellcolor[HTML]{FCE4D6}65.3$_{\pm14.3}$& \scriptsize \cellcolor[HTML]{FCE4D6}34.1$_{\pm12.7}$& \scriptsize \cellcolor[HTML]{FCE4D6}31.9$_{\pm13.5}$& \scriptsize \cellcolor[HTML]{FCE4D6}84.5$_{\pm14.3}$& \scriptsize \cellcolor[HTML]{FCE4D6}98.2$_{\pm2.2}$ & \scriptsize \cellcolor[HTML]{FCE4D6}86.9$_{\pm12.3}$& \scriptsize \cellcolor[HTML]{FCE4D6}84.6$_{\pm14.6}$ \\
\scriptsize \textit{JR}          & \scriptsize \cellcolor[HTML]{F8CBAD}6.5$_{\pm0.5}$& \scriptsize \cellcolor[HTML]{F8CBAD}14.0$_{\pm0.8}$& \scriptsize \cellcolor[HTML]{FCE4D6}4.6$_{\pm0.4}$& \scriptsize \cellcolor[HTML]{F8CBAD}4.8$_{\pm0.6}$& \scriptsize \cellcolor[HTML]{FCE4D6}41.4$_{\pm5.2}$ & \scriptsize \cellcolor[HTML]{FCE4D6}71.7$_{\pm2.5}$ & \scriptsize \cellcolor[HTML]{FCE4D6}36.5$_{\pm5.7}$ & \scriptsize \cellcolor[HTML]{FCE4D6}33.6$_{\pm5.8}$ & \scriptsize \cellcolor[HTML]{FCE4D6}82.5$_{\pm7.3}$ & \scriptsize \cellcolor[HTML]{FCE4D6}98.1$_{\pm1.3}$ & \scriptsize \cellcolor[HTML]{FCE4D6}79.5$_{\pm6.7}$ & \scriptsize \cellcolor[HTML]{FCE4D6}76.0$_{\pm9.4}$  \\
\scriptsize \textit{ER}          & \scriptsize \cellcolor[HTML]{FCE4D6}2.6$_{\pm0.8}$& \scriptsize \cellcolor[HTML]{FCE4D6}5.9$_{\pm1.5}$ & \scriptsize 0.7$_{\pm0.3}$& \scriptsize \cellcolor[HTML]{FFFFFF}1.9$_{\pm0.3}$& \scriptsize \cellcolor[HTML]{FCE4D6}12.3$_{\pm5.7}$ & \scriptsize \cellcolor[HTML]{FCE4D6}26.7$_{\pm8.1}$ & \scriptsize \cellcolor[HTML]{FCE4D6}4.7$_{\pm1.5}$  & \scriptsize \cellcolor[HTML]{FCE4D6}6.7$_{\pm2.5}$  & \scriptsize \cellcolor[HTML]{FCE4D6}40.4$_{\pm20.8}$& \scriptsize \cellcolor[HTML]{FCE4D6}75.0$_{\pm12.8}$& \scriptsize \cellcolor[HTML]{FCE4D6}23.5$_{\pm9.9}$ & \scriptsize \cellcolor[HTML]{FCE4D6}24.6$_{\pm10.8}$ \\
\scriptsize \textit{SAM}         & \scriptsize \cellcolor[HTML]{FCE4D6}6.1$_{\pm1.1}$& \scriptsize \cellcolor[HTML]{FCE4D6}8.9$_{\pm0.9}$ & \scriptsize \cellcolor[HTML]{FCE4D6}2.3$_{\pm0.5}$& \scriptsize \cellcolor[HTML]{FCE4D6}3.3$_{\pm0.1}$& \scriptsize \cellcolor[HTML]{FCE4D6}28.1$_{\pm1.1}$ & \scriptsize \cellcolor[HTML]{FCE4D6}46.5$_{\pm1.3}$ & \scriptsize \cellcolor[HTML]{FCE4D6}12.6$_{\pm2.0}$ & \scriptsize \cellcolor[HTML]{FCE4D6}11.6$_{\pm0.8}$ & \scriptsize \cellcolor[HTML]{FCE4D6}76.0$_{\pm0.4}$ & \scriptsize \cellcolor[HTML]{FCE4D6}92.1$_{\pm1.1}$ & \scriptsize \cellcolor[HTML]{FCE4D6}52.4$_{\pm1.0}$ & \scriptsize \cellcolor[HTML]{FCE4D6}53.7$_{\pm1.3}$  \\ \hline
\scriptsize \textit{SAM}\&\textit{IR}         & \scriptsize \cellcolor[HTML]{FCE4D6}5.9$_{\pm2.1}$& \scriptsize \cellcolor[HTML]{FCE4D6}13.2$_{\pm2.8}$ & \scriptsize \cellcolor[HTML]{F8CBAD}5.2$_{\pm1.6}$& \scriptsize \cellcolor[HTML]{FCE4D6}4.3$_{\pm0.9}$& \scriptsize \cellcolor[HTML]{F8CBAD}44.2$_{\pm12.8}$ & \scriptsize \cellcolor[HTML]{F8CBAD}76.3$_{\pm4.5}$ & \scriptsize \cellcolor[HTML]{F4B084}46.5$_{\pm12.1}$ & \scriptsize \cellcolor[HTML]{F4B084}43.6$_{\pm11.8}$ & \scriptsize \cellcolor[HTML]{F4B084}92.5$_{\pm6.7}$ & \scriptsize \cellcolor[HTML]{F4B084}99.5$_{\pm0.1}$ & \scriptsize \cellcolor[HTML]{F4B084}93.5$_{\pm4.5}$ & \scriptsize \cellcolor[HTML]{F4B084}92.4$_{\pm4.8}$  \\
\scriptsize \textit{SAM}\&\textit{JR}         & \scriptsize \cellcolor[HTML]{F4B084}9.9$_{\pm1.3}$& \scriptsize \cellcolor[HTML]{F4B084}19.5$_{\pm1.3}$ & \scriptsize \cellcolor[HTML]{F4B084}6.4$_{\pm0.6}$& \scriptsize \cellcolor[HTML]{F4B084}5.9$_{\pm1.1}$& \scriptsize \cellcolor[HTML]{F4B084}50.0$_{\pm6.4}$ & \scriptsize \cellcolor[HTML]{F4B084}78.7$_{\pm1.7}$ & \scriptsize \cellcolor[HTML]{F8CBAD}41.4$_{\pm9.0}$ & \scriptsize \cellcolor[HTML]{F8CBAD}38.4$_{\pm7.2}$ & \scriptsize \cellcolor[HTML]{F8CBAD}88.7$_{\pm7.5}$ & \scriptsize \cellcolor[HTML]{F8CBAD}99.3$_{\pm0.1}$ & \scriptsize \cellcolor[HTML]{FCE4D6}83.1$_{\pm5.7}$ & \scriptsize \cellcolor[HTML]{FCE4D6}84.1$_{\pm7.1}$  \\ 
 \hline
\end{tabular}
}
\label{tab:imagenette-ASR}
\vspace{-1mm}
\end{table*}

\vspace{1mm}
\noindent\textbf{Data distribution shift  impairs gradient similarity.}
Different from model smoothness, it is difficult to understand how exactly \AT ~degrades the gradient similarity. 
In this work, we attribute this  degradation  to the data distribution shift induced by the off-manifold examples in \AT.
Since the emerging of adversarial examples~\cite{szegedy2013intriguing}, there is a long-held belief that:
\textit{Clean data lies in a low-dimensional manifold. 
Even though the adversarial examples are close to the clean data, they lie off the underlying data manifold}~\cite{ma2018characterizing,song2018pixeldefend,khoury2019on,gilmer2018adversarial,liu2022towards}. 
Nevertheless, recent researches demonstrate that  AEs can also be on-manifold \cite{kwon2021improving,patel2021manifold,Stutz_2019_CVPR,NEURIPS2020_23937b42,liu2022towards}, 
and on-manifold and off-manifold AEs may co-exist~\cite{xiao2022understanding}. 
With this in mind, we analyze how  \AT ~enlarges the data distribution shift along with the increment of  adversarial budget $\epsilon$.

Given a regular loss function $\ell(f(x),y)$  on the low-dimensional manifold $\mathcal{M}$, the adversarial perturbation $\delta \in \mathbb{R}^d$ could make the loss $\ell(f(x+\delta),y)$ sufficiently high with $x + \delta \in \mathcal{M}$ or $x + \delta \notin \mathcal{M}$~\cite{xuICML2022Adversarially2022}. 
At a high level, \AT ~augments the dataset by adding adversarial examples during each iteration, obtaining an augmented data distribution on $\mathcal{P}_{\mathcal{X}}$, and the adversarial budget $\epsilon$ can be regarded as a parameter that controls the augmentation magnitude. 
Intuitively, a larger $\epsilon$  induces more space for off-manifold samples, resulting in a larger distribution shift.
On the other hand, it is also widely acknowledged that adversarial training with larger perturbation budgets leads to \textit{robust overfitting} \cite{wu2020adversarial,dong2022exploring,stutz2021relating,pmlr-v119-rice20a}.
Accordingly, a larger $\epsilon$ may cause more disbenefit of underfitting to the normal distribution $\mathcal{D}$.
Formally, we formalize our hypothesis  as follows:
\vspace{-1mm}
\begin{hypothesis}[\textbf{Distribution shift impairs gradient similarity}]
\label{hypothesis1}
Given two data distribution $\mathcal{D}, \mathcal{D}' \in \mathcal{P}_{\mathcal{X}\times\mathcal{Y}}$, and two source models $\mathcal{F}_{\mathcal{D}}$ and $\mathcal{F}_{\mathcal{D}'}$ trained on $\mathcal{D}$ and  $\mathcal{D}'$,  respectively, 
supposing the target model $\mathcal{G}_{\mathcal{D}}$ is also trained on $\mathcal{D}$  and they all share a joint training loss $\ell$,
if the distance between $\mathcal{D}$ and $\mathcal{D}'$ is large enough, then \black{
$\tilde{\mathcal{S}}_\mathcal{D}\left(\ell_{\mathcal{F}_{\mathcal{D}'}}, \ell_{\mathcal{G}_{\mathcal{D}}}\right)
<\tilde{\mathcal{S}}_\mathcal{D}\left(\ell_{\mathcal{F}_{\mathcal{D}}}, \ell_{\mathcal{G}_{\mathcal{D}}}\right)$} is likely to stand.
\end{hypothesis}

The intuition behind this hypothesis is that, on average, the two models trained on the same distribution should align better in the gradient direction than those trained with different distributions. 
Note that, we relax the distribution change of \AT ~on $\mathcal{P}_{\mathcal{X}}$ to $\mathcal{P}_{\mathcal{X}\times\mathcal{Y}}$ here, hypothesizing the general change on $\mathcal{P}_{\mathcal{X}\times\mathcal{Y}}$ impairs gradient similarity.
In \cref{analysis}, we  extend this hypothesis to a general version where target models are also trained over $\mathcal{P}_{\mathcal{X} \times \mathcal{Y}}$ (see \cref{hypothesis2}).

\section{Investigating Data Augmentation}
\label{sec:trade-off:da}

To extend the distribution shifts of \AT ~to more general cases and further verify \cref{hypothesis1}, we investigate how data augmentations, the popular training paradigm that explicitly changes data distribution, influence  gradient similarity. 
Additionally, we investigate the trade-off effect between similarity and smoothness  under data augmentations and observe how they reflect on the  transfer attack  accordingly.

\subsection{Data Augmentation Mechanisms}

We explore 4  popular data augmentation mechanisms and choose proper augmentation magnitude parameters to control the \black{degree of} data distribution shift from $\mathcal{D}$ to $\mathcal{D}'$, aligning the augmentation parameter $\epsilon$ in \AT. 

\vspace{0.5mm}
\noindent\textbf{Mixup (\textit{MU}) \cite{zhang2018mixup}} trains a model on the convex combination of randomly selected sample pairs, \ie, the mixed image $\tilde{x}_{i, j} = b x_i+(1-b )x_j$ and the corresponding label $\tilde{y}_{i, j}=b  y_i+(1-b )y_j$, where $(x_i, y_i), (x_j, y_j) \in \mathbf{S}$. 
Here $b \in [0,1]$ is a random variable drawn from the Beta distribution $\operatorname{Beta}(1,1)$. 
Additionally, we use a probability parameter $\mathbf{p}$ to control how much  $\mathcal{D}'$ shifts away from $\mathcal{D}$. 
Specifically, as $\mathbf{p}\rightarrow 0$, the augmented dataset will be identical to $\mathbf{S}$; 
as $\mathbf{p}\rightarrow 1$, all the samples will be interpolated. 
Thus, at a high level, $\mathbf{p}$ controls how much  $\mathcal{D}'$ shifts away from $\mathcal{D}$.
We train augmented models for \MU~with $\mathbf{p} \in [0.1, 0.3, 0.5, 0.7, 0.9]$.

\begin{figure}[t!]
\centering
\includegraphics[width=0.42\textwidth]{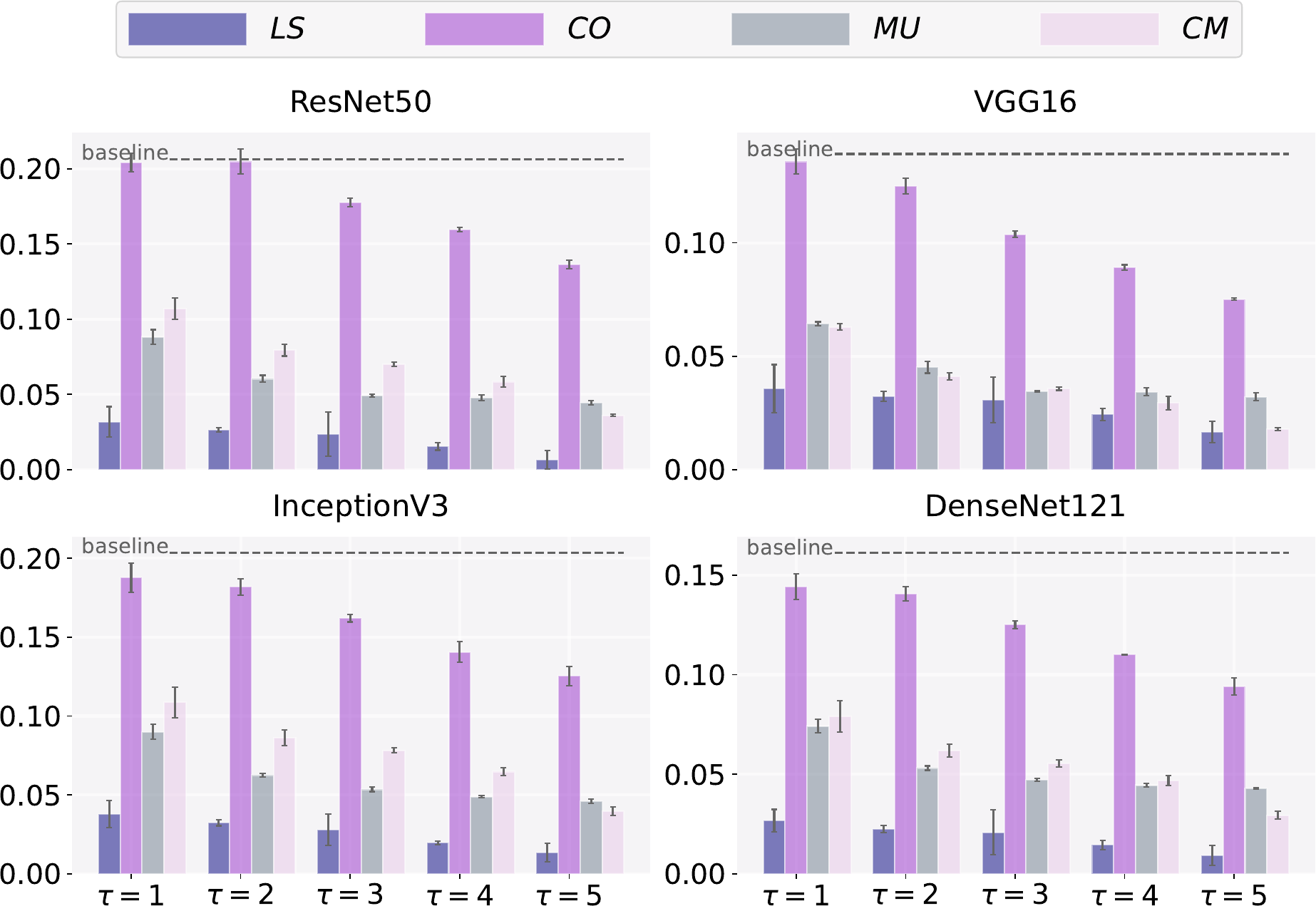}
\caption{Gradient similarities between  augmented surrogate and  \ST ~target  CIFAR-10 models with error bars at each $\tau$.
}
\label{fig:GS-DA-cifar10}
\end{figure}
\vspace{0.5mm}
\noindent\textbf{Cutmix (\textit{CM})~\cite{yun2019cutmix}} also augments both images and labels by mixing samples. 
The difference is patches are cut and pasted among training samples. 
A mixed sample $\tilde{x}_{i,j} = \mathbf{M_b } \odot x_i + (\mathbf{1}-\mathbf{M_b }) \odot x_j$ and its  label $\tilde{y}_{i,j}$ are also obtained from two samples $(x_i, y_i), (x_j, y_j) \in \mathbf{S}$, where $\mathbf{M_b}$ is a  binary matrix with the same size of $x_i$ to indicate the location for cutting and pasting, $\mathbf{1}$ is a matrix filled with  1, and $b$ is drawn from $\operatorname{Beta}(1,1)$.
$\tilde{y}_{i,j}$ is obtained similarly as \MU. 
Similar to in \MU, we use a probability parameter $\mathbf{p}$ to control the augmentation magnitude and train augmented models for \CM~with $\mathbf{p} \in [0.1, 0.3, 0.5, 0.7, 0.9]$.  

\vspace{0.5mm}
\noindent\textbf{Cutout (\textit{CO}) \cite{devries2017improved}} augments the dataset by masking out random regions of images with a size less than $M \times M$.   
In the real implementation, the mask-out region and its size for each image are not deterministic. A bigger $M$ value indicates it is more likely to mask out a bigger region. 
Hence, we utilize the value of $M$ only to regulate the extent of augmentation since it probabilistically governs the augmentation magnitude, similar to the role of $\mathbf{p}$ in \MU ~and \CM.
We select  $M \in [8, 12, 16, 20, 24]$ and $M \in [80, 100, 120, 140, 160]$ for CIFAR-10 and ImageNette, respectively.

\vspace{0.5mm}
\noindent\textbf{{Label smoothing (\textit{LS}) \cite{szegedy2016rethinking}}} augments the dataset by replacing the hard labels $\mathcal{L}$ with soft continuous labels in probability distribution. 
Specifically, the probability of the ground-truth class is $1-p$, and the probability of other classes will be uniformly assigned to $p/(m-1)$, where $p \in [0,1)$ and $m$ denotes the number of classes. 
We train models for \LS ~with $p \in [0.1,0.2,0.3,0,4,0.5]$.

These augmentations along with \AT ~change  data distributions in distinct manners. 
For a normal sample $(x, y)\sim\mathcal{D}$,  \AT ~associates input feature $x$ with ground-truth $y$ ``inaccurately'' under a noise $\delta$, while \CO ~associates part of ``right'' $x$ with ground-truth $y$, thus they explicitly change the distribution on $\mathcal{P}_\mathcal{X}$. 
Contrarily, both  \MU ~and \CM ~correlate parts of $x$ with $y$, causing shift in  $\mathcal{P}_{\mathcal{X}\times\mathcal{Y}}$. 
Whereas \LS ~relates the full $x$ to a ``wrong'' $y'$, causing shift in  $\mathcal{P}_{\mathcal{Y}}$.
Intuitively, these augmentation methods will cause different impacts on the gradient $\nabla\ell_\mathcal{F}(x, y)$.
As we  select 5 augmentation magnitudes for each augmentation, we use a general parameter $\tau \in [1,2,3,4,5] $ to simplify the notations, \eg,  $\tau=1$ means  $\mathbf{p}= 0.1$ for \MU~and $M=8$ for \CO~on CIFAR-10.

\subsection{Data Augmentation Impairs Similarity}
To verify \cref{hypothesis1}, we  explore how    different data distribution shifts influence  gradient similarity. For all the  experiments, we  train models with 3 different random seeds in each setting.
Besides, multiple target models are considered to ensure generality. 
The  results over 3 standard training (\ST) models without  augmentations are used as baselines.

\begin{figure}[t!]
\centering
\includegraphics[width=0.43\textwidth]{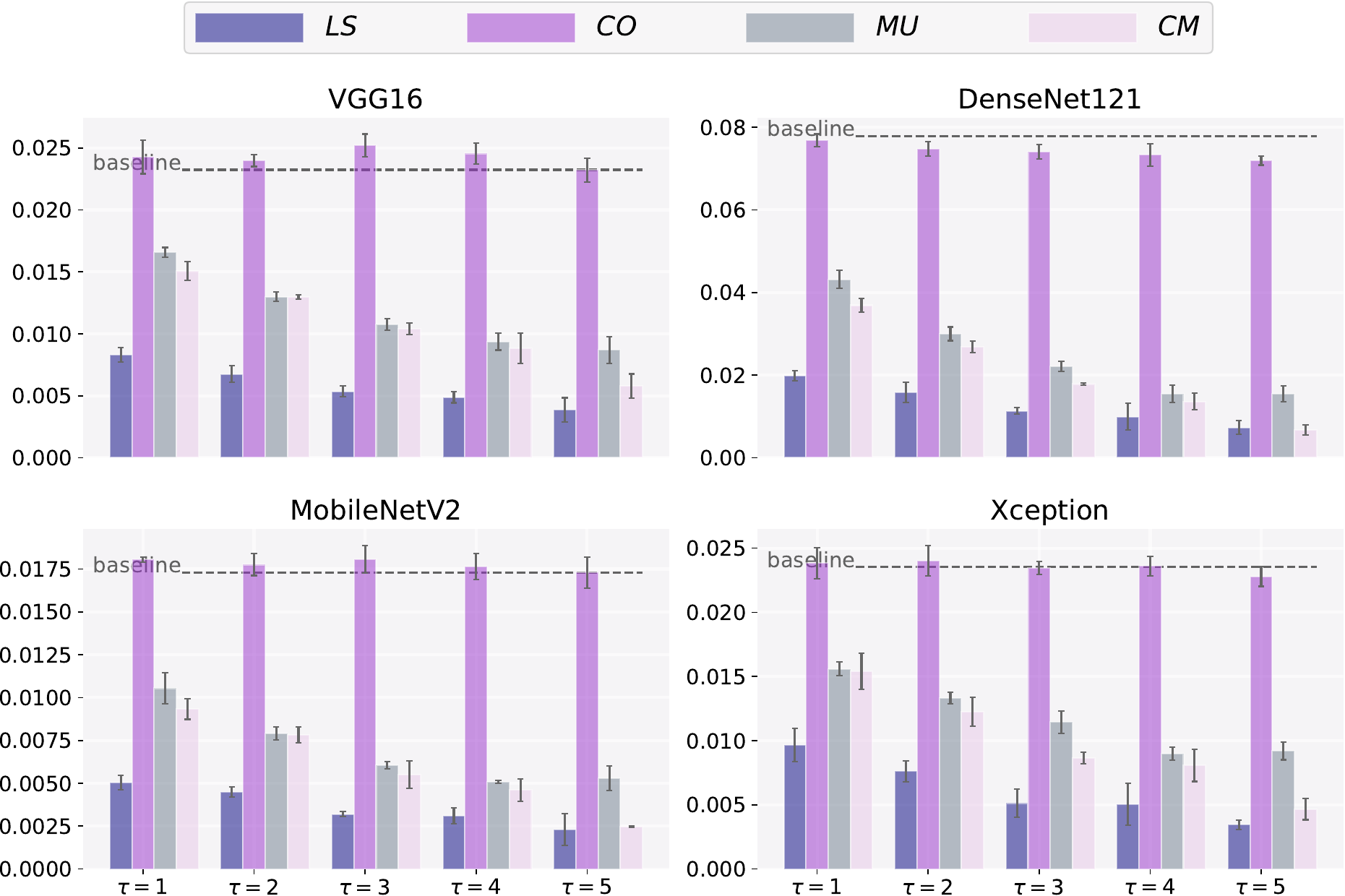}
\caption{Gradient similarities between  augmented surrogate and  \ST ~target  ImageNette models with error bars at each $\tau$.
}
\label{fig:GS-DA-imagenette}
\vspace{-2mm}
\end{figure}

\cref{fig:GS-DA-cifar10,fig:GS-DA-imagenette}  depict the alignments towards multiple target models for CIFAR-10 and ImageNette.
First,  we observe that \CO~harms gradient similarity the least, but its effect varies across different datasets.
Specifically, \CO ~slightly degrades  gradient similarity in CIFAR-10 while appearing comparable to baselines in ImageNette.
The difference in training settings between the two datasets could be the reason for this gap; CIFAR-10 models are trained from scratch, while ImageNette models are obtained by fine-tuning well-trained ImageNet classifiers. These different outcomes suggest that simply removing some pixels is not enough to make the model forget what it has previously ``seen",  since ImageNette is a 10-class subset of ImageNet.
On the other hand, \LS, \MU, and \CM~degrade similarity largely over both datasets.
Note that the similarity degradations  of these augmentations are in line with their  manipulations on  $\mathcal{P}_\mathcal{Y}$, 
suggesting modifying the ground-truth $y$ during training results in more impacts on the direction of gradient $\nabla\ell_\mathcal{F}(x, y)$.

In summary, despite subtle differences, \textbf{the results on both datasets support the hypothesis  that data distribution shift impairs gradient similarity.
Further, the distribution shift in $\mathcal{P}_\mathcal{Y}$ may have a greater  negative impact on  similarity than that in $\mathcal{P}_\mathcal{X}$}.

\subsection{Trade-off Under Data Augmentation}
To figure out how these augmentations influence smoothness and  transferability, we measure smoothness and the  ASRs of AEs \wrt these augmented models.
To save space,  we merely report the $\tau=1, 5$ cases for PGD attack in \cref{tab:cifar10-ASR,tab:imagenette-ASR}, as they represent the smallest and largest augmentation magnitudes.
We report  the $\tau=1, 3, 5$ cases for AutoAttack in  \cref{tab:autoattack}. 
\cref{fig:MS-DA-cifar10,fig:MS-DA-imagenette} plots the smoothness for augmented models on two datasets.

First,  \LS ~exhibits a monotonous benefit on smoothness on both datasets due to its implicit effect on reducing the gradient norm during training \cite{szegedy2016rethinking}.
However,   the tables show that the ASRs of \LS ~are always lower than those of \ST, implying this benefit may not completely offset the negative effect of largely degraded similarity in \LS. 
Additionally,  there is no monotonicity in ASRs in \LS, aligning with  the  always combating negative and positive effects.
In contrast, other augmentations do not exhibit a consistent tendency for smoothness on both datasets. 
Specifically, in CIFAR-10, \CM, \MU, and \CO ~mostly degrade the smoothness while the degradation is not strictly monotonous. 
The situation is even more chaotic in ImageNette, where the variances of smoothness are relatively large. However, the ASRs do reveal some patterns. 
For \CM ~and \MU,  their ASRs degrade from $\tau=1$ to $\tau=5$, and the ASRs of \CO ~are generally better than \MU ~and \CM. 
Notably, \CO ~performs slightly better than \ST ~under a few cases, especially when $\tau=1$. 
This agrees with the fact that \CO ~has  slightly better smoothness and less severely degraded similarity.
Conclusively, the  \textbf{trade-off under data augmentations is quite complex, and no single augmentation can  produce good surrogates}.
\begin{figure}[!t]
\centering
\subcaptionbox{\MU}{\includegraphics[width=0.215 \textwidth]{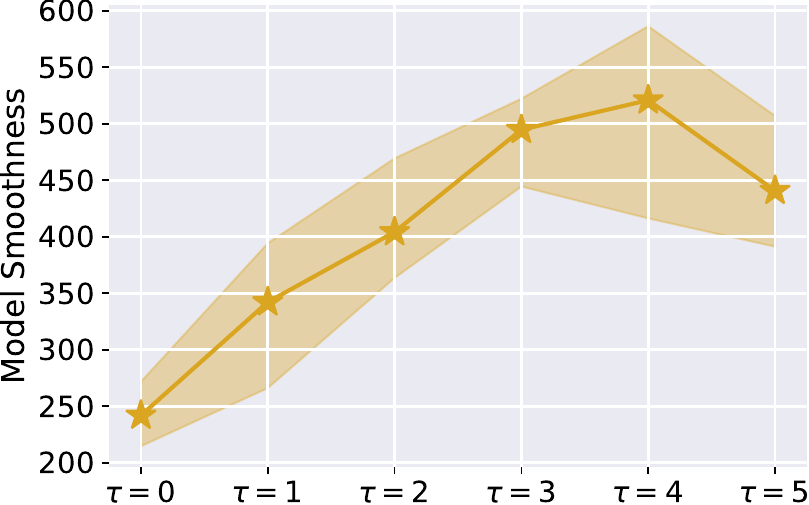}}
\hfill 
\subcaptionbox{\CM}{\includegraphics[width=0.215 \textwidth]{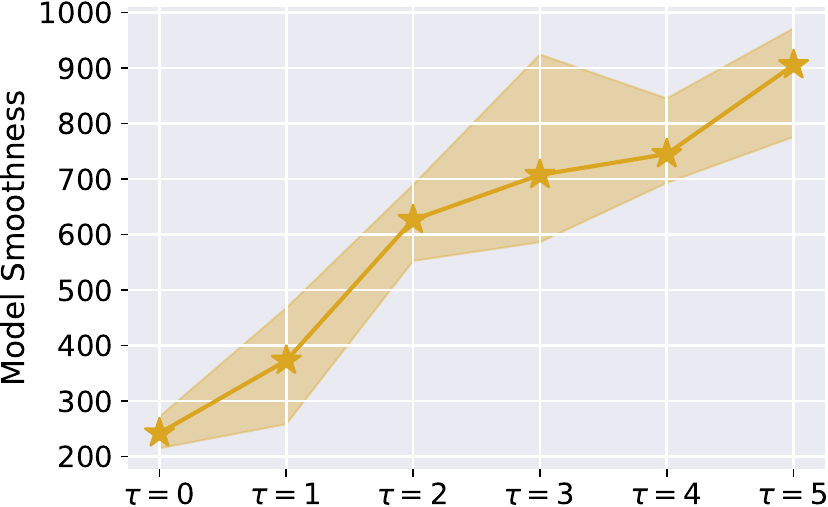}}
\subcaptionbox{\CO}{\includegraphics[width=0.215 \textwidth]{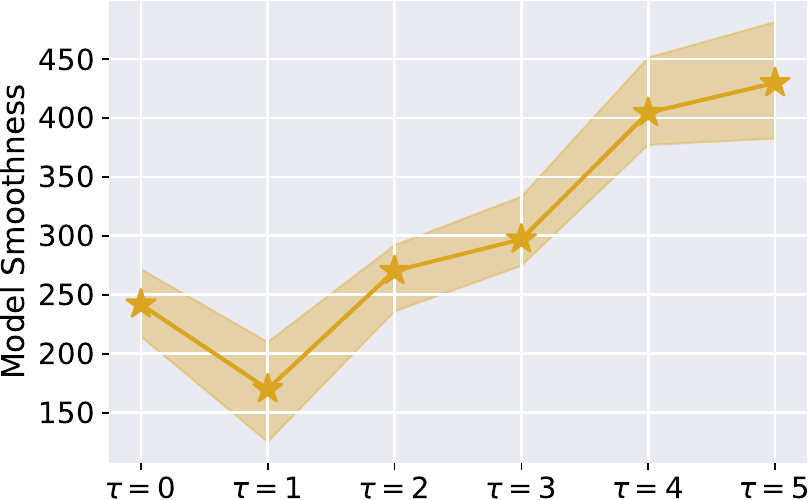}}
\hfill 
\subcaptionbox{\LS}{\includegraphics[width=0.215 \textwidth]{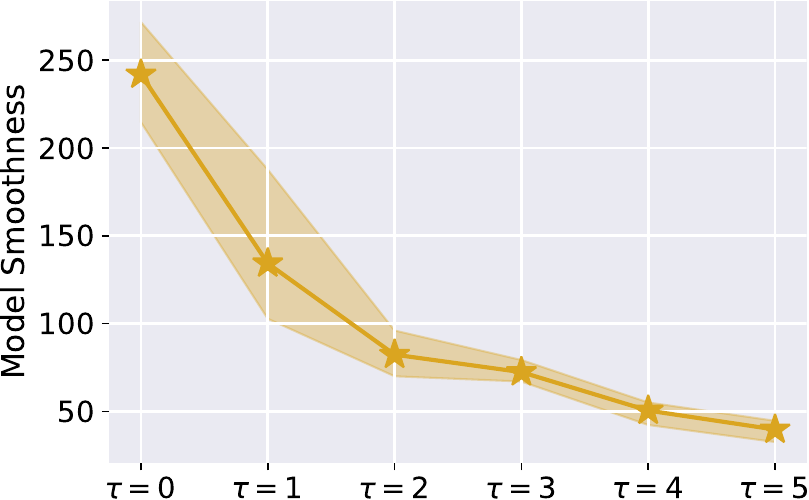}}
\caption{Average model smoothness of augmented models on CIFAR-10 with respective error bars at each $\tau$.}
\label{fig:MS-DA-cifar10}
\vspace{-2mm}
\end{figure}



\section{Investigating Gradient Regularization}
\label{sec:trade-off:gr}
At present, we  recognize what may  harm gradient similarity, but have no idea of how to generally improve it. 
Additionally, it is challenging to regulate  similarity independently without affecting other factors, especially in black-box scenarios where  target  models are unknown. 
In contrast, the standalone model smoothness can be regulated and measured  with the surrogate itself. 
Therefore, we will explore how gradient regularization, which is expected to improve smoothness  without changing the data distribution, can be used to enhance surrogates.

\subsection{Gradient Regularization Mechanisms}
\vspace{-2mm}
We first formally define  gradient regularization and deliver the intuitions on why they are chosen.
To sharpen the analysis, we introduce the notation $\theta$ to parameterize a neural network and reform  $f(\cdot)$  as $f_\theta(\cdot)$ if needed. For concise illustrations, we sometimes abbreviate  $l(f(x), y))$ as $l(f(x))$  in this section.

\vspace{-2mm}
\subsubsection{Regularization in the Input Space}
\label{sec:input-space}
Based on the \black{implication that regularizing input space smoothness benefits  transferability shown in Sec.~\ref{Transfer-Analysis}, one should directly optimize the loss surface curvature.}
However, it is challenging to realize it since the computationally expensive second-order derivative is involved, \ie, $\nabla_{x}^2 \ell(f_\theta(x))$. 
Fortunately, deep learning and optimization theory~\cite{drucker1992improving} allow us to use  the first-order derivative to approximate it. 
We thus choose  two well-known  first-order gradient regularizations as follows:

\begin{figure}[!t]
\centering
\subcaptionbox{\MU}{\includegraphics[width=0.211\textwidth]{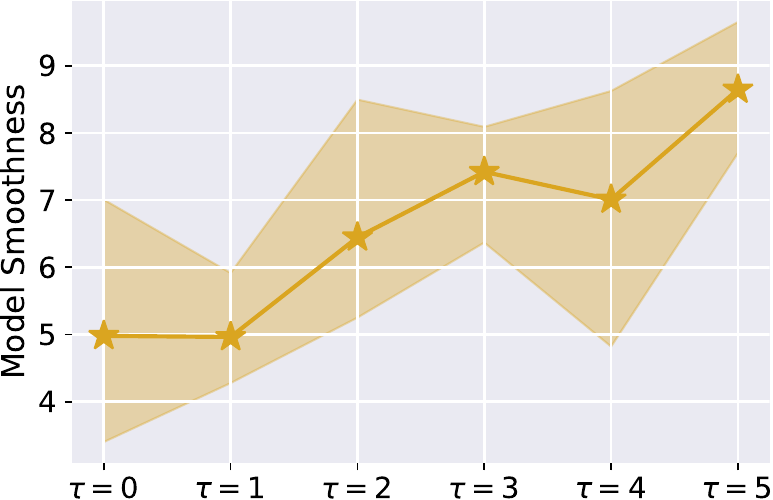}}
\hfill 
\subcaptionbox{\CM}{\includegraphics[width=0.211 \textwidth]{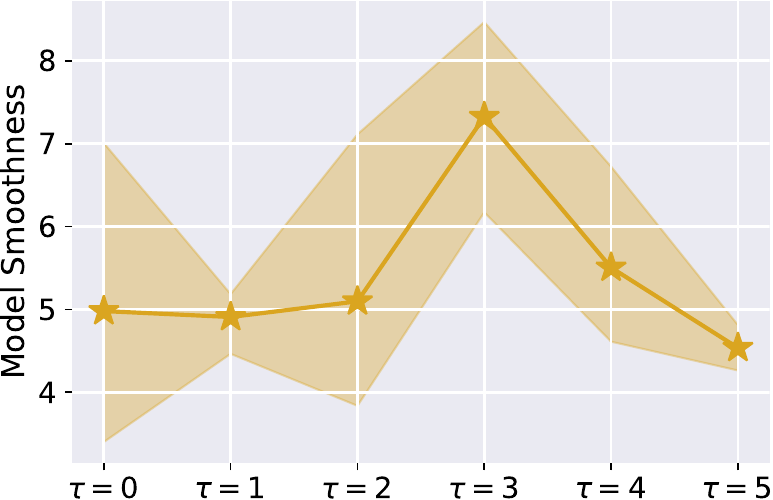}}
\subcaptionbox{\CO}{\includegraphics[width=0.211 \textwidth]{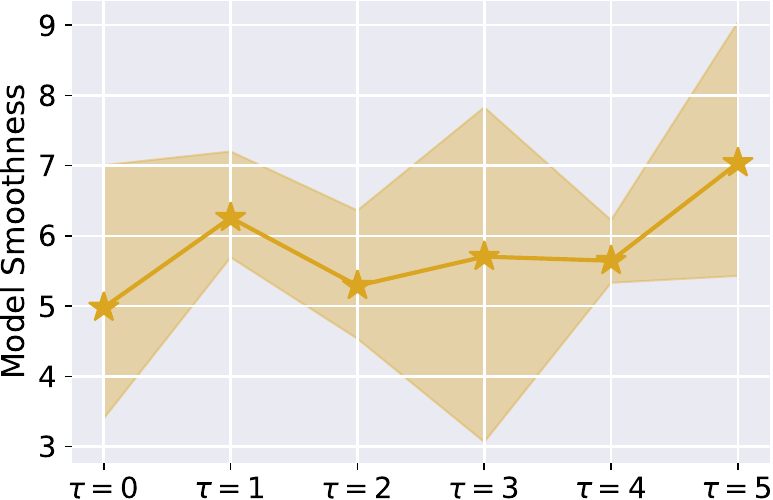}}
\hfill
\subcaptionbox{\LS}{\includegraphics[width=0.211 \textwidth]{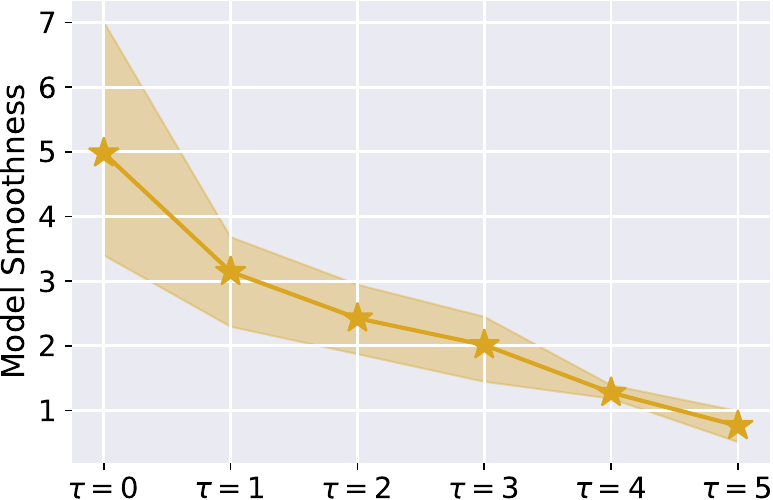}}
\caption{Average model smoothness of augmented models on ImageNette with respective error bars at each $\tau$.}
\label{fig:MS-DA-imagenette}
\vspace{-3mm}
\end{figure}

\noindent \textbf{Input gradient regularization (\textit{IR})}: \IR ~directly add the Euclidean norm of gradient  \wrt input  to loss function as: 
$$
L_{ir} = \frac{1}{\left\|\mathbf{S}\right\|}\sum_{i=1}^{\left\|\mathbf{S}\right\|} [ \ell(f(x_i))+\lambda_{ir}\|\nabla_x\ell(f(x_i))\|], 
$$
where $\lambda_{ir}$ controls the regularization magnitude. \IR ~has  been demonstrated  improving the generalization and interpretability of DNNs~\cite{drucker1992improving,ross2018improving}. 
Here we  show that \IR ~can also improve transferability.  
Regularizing $\|\nabla_x\ell(f(x))\|_2$  leads to a smaller curvature $\sigma (\nabla_{x}^2 \ell(x))$ because the spectral norm of a matrix is upper-bounded by the Frobenius norm, \ie, $\sigma (\nabla_{x}^2 \ell(x)) \leq \|\nabla_{x}^2 \ell(x)\|_F \approx \|\nabla_{x} \ell(f(x))^T \nabla_{x} \ell(f(x))\|_F \leq \|\nabla_x\ell(f(x))\|^2$. 
Here  we approximate the second-order derivative using the first-order derivative as suggested in~\cite{DBLP:conf/icml/MartensSS12,jakubovitz2018improving}.
Consequently, smaller $\sigma (\nabla_{x}^2 \ell(x))$ indicates better smoothness, 
thereby improving transferability.\\
\noindent \textbf{Input Jacobian regularization (\textit{JR})}: \JR ~regularizes gradients on  the logits output of the network as:
$$
L_{jr} = \frac{1}{\left\|\mathbf{S}\right\|}\sum_{i=1}^{\left\|\mathbf{S}\right\|} [ \ell(f_\theta(x_i)+\lambda_{jr}\|\nabla_x f_\theta(x_i)\|_F].
$$
\JR~has been also correlated with the generalization and robustness~\cite{jakubovitz2018improving,yoshida2017spectral,chan2019jacobian,hoffman2019robust,varga2017gradient}. 
Since \JR~has a similar form with \IR, it is reasonable to presume  that  they have a similar effect. 
Formally, we prove that when $\|\nabla_x f_\theta(x_i)\|_F \rightarrow 0$, we have $\|\nabla_x \ell(f_\theta(x_i))\| \rightarrow 0$ using cross-entropy loss as an example (see \cref{proposition1}). 
This implies that \JR~also has a positive effect on  model smoothness.

\begin{figure}[!t]
\centering
\subcaptionbox{\IR}{\includegraphics[width=0.215 \textwidth]{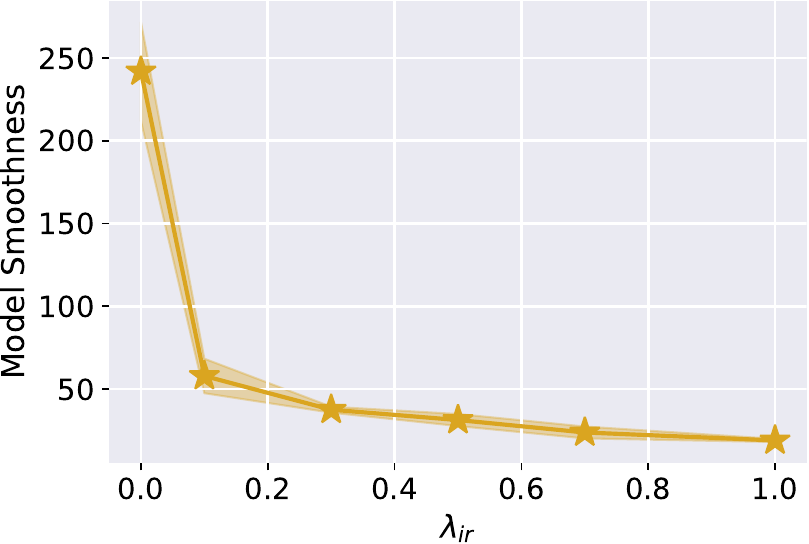}}
\hfill 
\subcaptionbox{\JR}{\includegraphics[width=0.215 \textwidth]{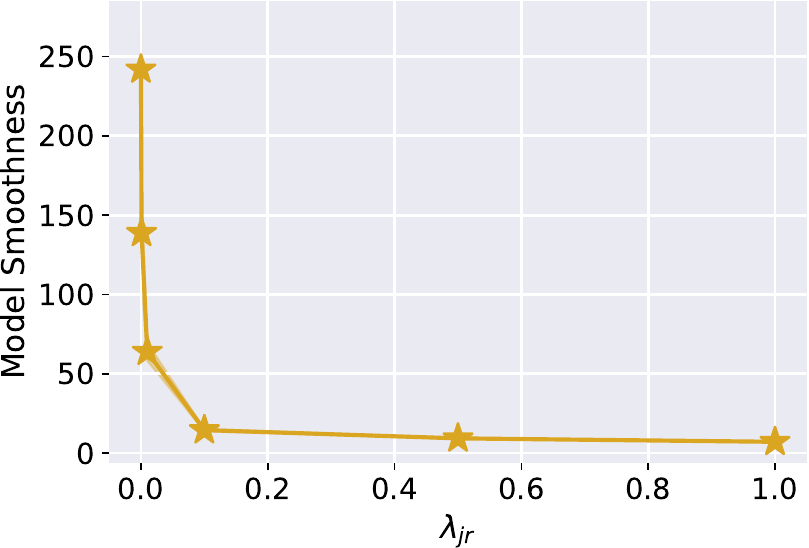}}
 \subcaptionbox{\ER}{\includegraphics[width=0.215 \textwidth]{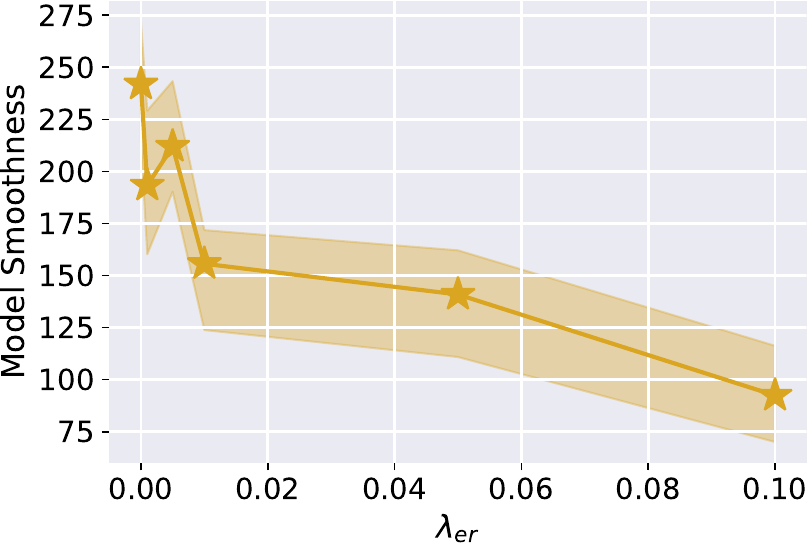}}
 \hfill 
\subcaptionbox{\SAM}{\includegraphics[width=0.215 \textwidth]{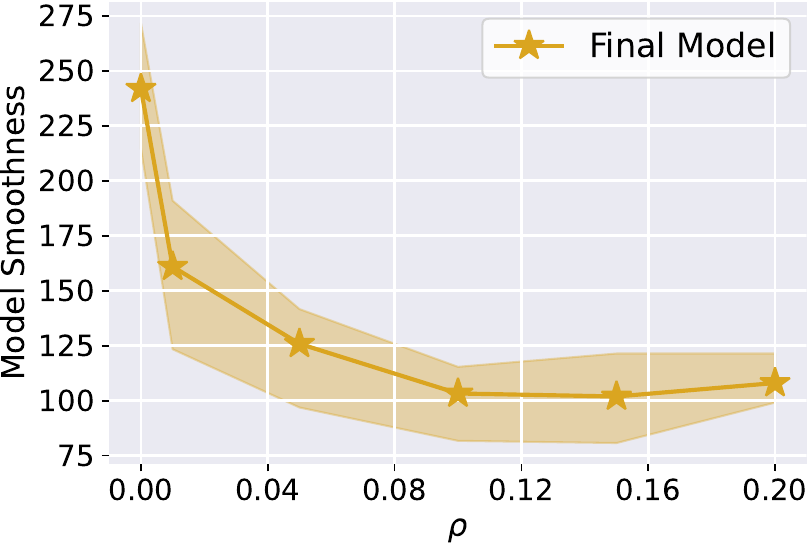}}
\caption{Model smoothness of regularized  models on CIFAR-10 with respective error bars at each $\tau$. The standard accuracy of these models ranges from 90.12\% to 95.38\%.}
\label{fig:MS-GR-cifar10}
\vspace{-2mm}
\end{figure}
\subsubsection{Regularization in the Weight Space}
 
Recently, researchers paid attention to the gradient regularization on $\theta$, \ie, the parameters (weights) of the network.
A recent study~\cite{Dherin2022why}  proved that the gradient regularizing pressure on the weight space $\|\nabla_\theta f_\theta(x)\|_F$ can  transfer to the input space $\|\nabla_x f_\theta(x)\|_F$. 
Consequently, regularizing  gradients \wrt the weight space could also implicitly promote the  smoothness in the input space. 
We next explore two representative  weight gradient regularization methods as follows:\\
\noindent \textbf{Explicit gradient regularization (\textit{ER})}: 
In \ER, the Euclidean norm of gradient \wrt $\theta$ is  added to the training loss $L_{er}$ to promote flatness  as:
$$
L_{er} (\theta) = L(\theta) + \frac{\lambda_{er}}{2} \|\nabla_\theta L(\theta)\|^2,
$$
where $L$ is the original objective function.
Similar to \IR~and \JR, we can directly apply the double-propagation technique to implement the regularization term, using the automatic differentiation framework provided by  modern ML libraries like PyTorch~\cite{paszke2017automatic}.
\ER~has been empirically studied on  improving  generalization~\cite{barrett2021implicit,smith2021on,geiping2022stochastic}, while we investigate it from the transferability perspective.\\
\noindent \textbf{Sharpness-aware minimization (\textit{SAM})}: \black{\SAM ~also seeks flat solution in the weight space.
Specifically, it intends to minimize the original training loss and the worst-case sharpness $L_{sam}(\theta) =  L(\theta) + [\max_{\|\boldsymbol{\hat{\theta} -\theta}\|_2 \leq \rho} L(\hat{\theta}) - L(\theta)]= \max_{\|\boldsymbol{\hat{\theta} -\theta}\|_2 \leq \rho} L(\hat{\theta})$, where $\rho$ denotes the radius to search for the worst neighbor $\hat{\theta}$. 
As  the exact worst neighbor is difficult to track,
 \SAM ~uses the  gradient of ascent direction neighbor for the update  after twice approximations:}
$$
\nabla_{\theta} L_{sam}(\theta) \approx  \nabla_{\theta}L(\theta +\rho\frac{\nabla_{\theta}L(\theta)}{||\nabla_{\theta}L(\theta)||}).
$$
Recent researches establish \SAM ~as a special kind of gradient normalization~\cite{zhao2022penalizing,karakida2022understanding}, where $\rho$ represents the regularization magnitude.
\SAM~has been extensively studied recently for its remarkable performance on  generalization~\cite{foret2021sharpnessaware,andriushchenkoICML22a,chenICLR2022When2022,zhuang2022surrogate,wen2022does,du2022efficient,pmlr-v162-abbas22b}.
In this paper, for the first time, we substantiate that \SAM ~also features outstanding superiority in boosting transferability.

\subsection{Gradient Regularization Promotes  Smoothness}
\label{sec:GR-smoothness}    
We train multiple regularized models for each regularization and use coarse-grained parameter intervals \cite{ross2018improving,foret2021sharpnessaware,smith2021on} and  restrict the largest regularized magnitude to avoid excessive regularization such that the accuracy of these  models are  above 90\%. 
This allows for a fair comparison of each regularization method in an acceptable accuracy range.  
We choose 5 parameters in each interval through random grad and binary choices (See \cref{tab:training-settings}).

\begin{figure}[t!]
\centering
\subcaptionbox{\IR}{\includegraphics[width=0.213 \textwidth]{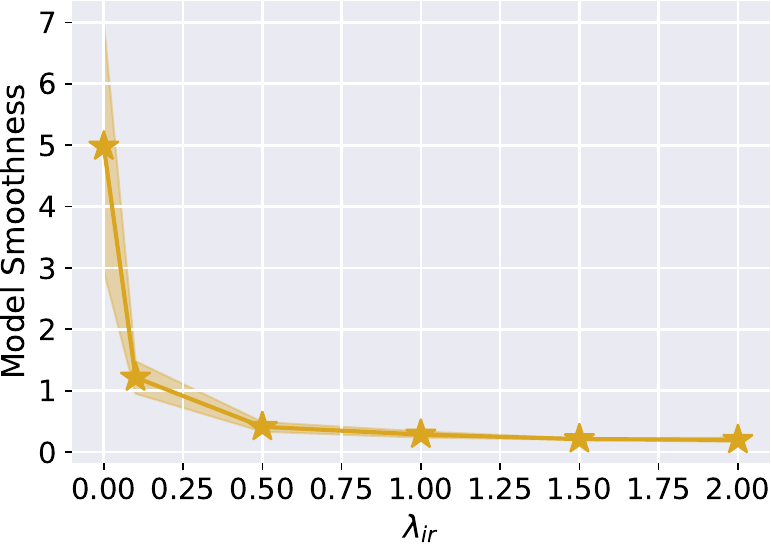}}
\hfill 
\subcaptionbox{\JR}{\includegraphics[width=0.213 \textwidth]{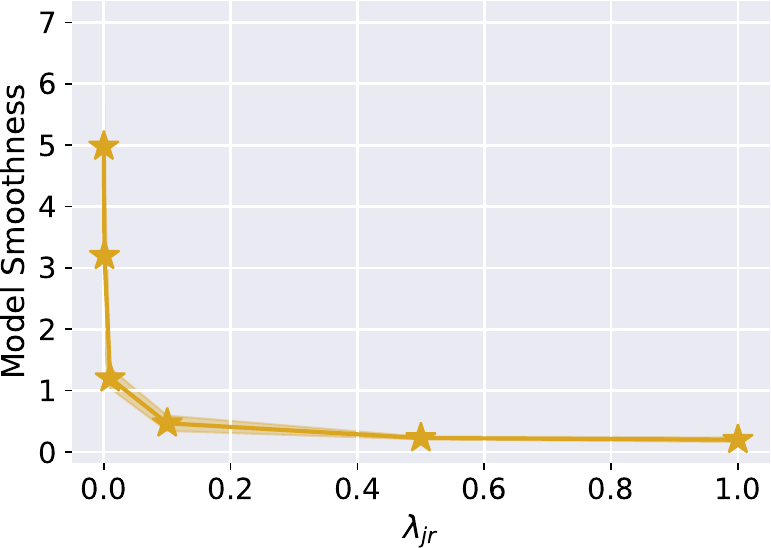}}
 \subcaptionbox{\ER}{\includegraphics[width=0.213 \textwidth]{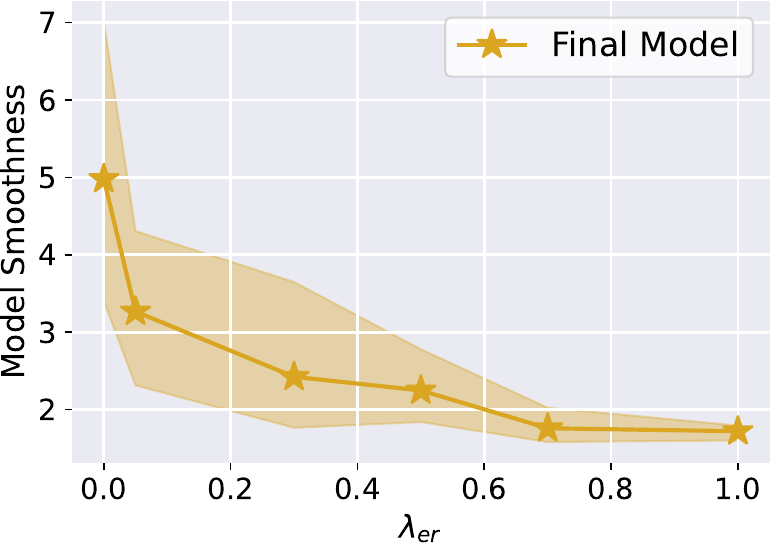}}
 \hfill 
\subcaptionbox{\SAM}{\includegraphics[width=0.213 \textwidth]{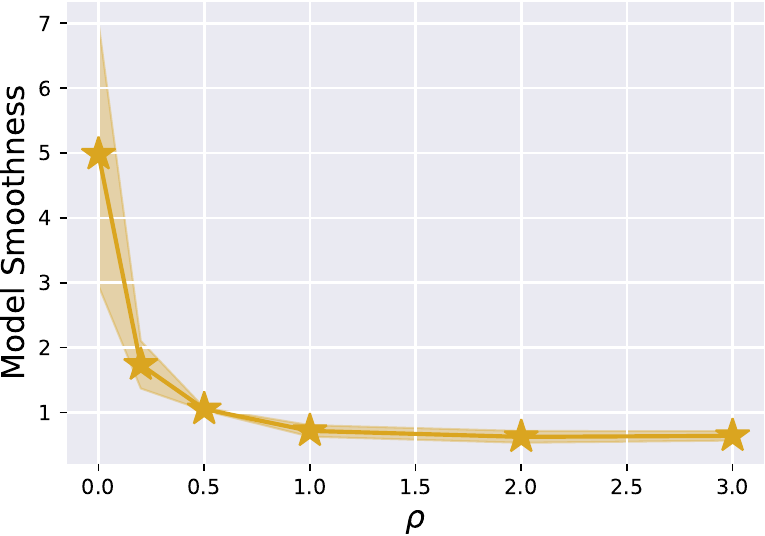}}
\caption{Model smoothness of regularized models on ImageNette with respective error bars at each $\tau$. The standard accuracy of these models ranges from 90.62\% to 96.91\%.}
\label{fig:MS-GR-imagenette}
\vspace{-3mm}
\end{figure}

\begin{figure*}[!t] 

\centering
\subcaptionbox{ResNet18 trained over CIFAR-10}{\includegraphics[width=0.45\textwidth]{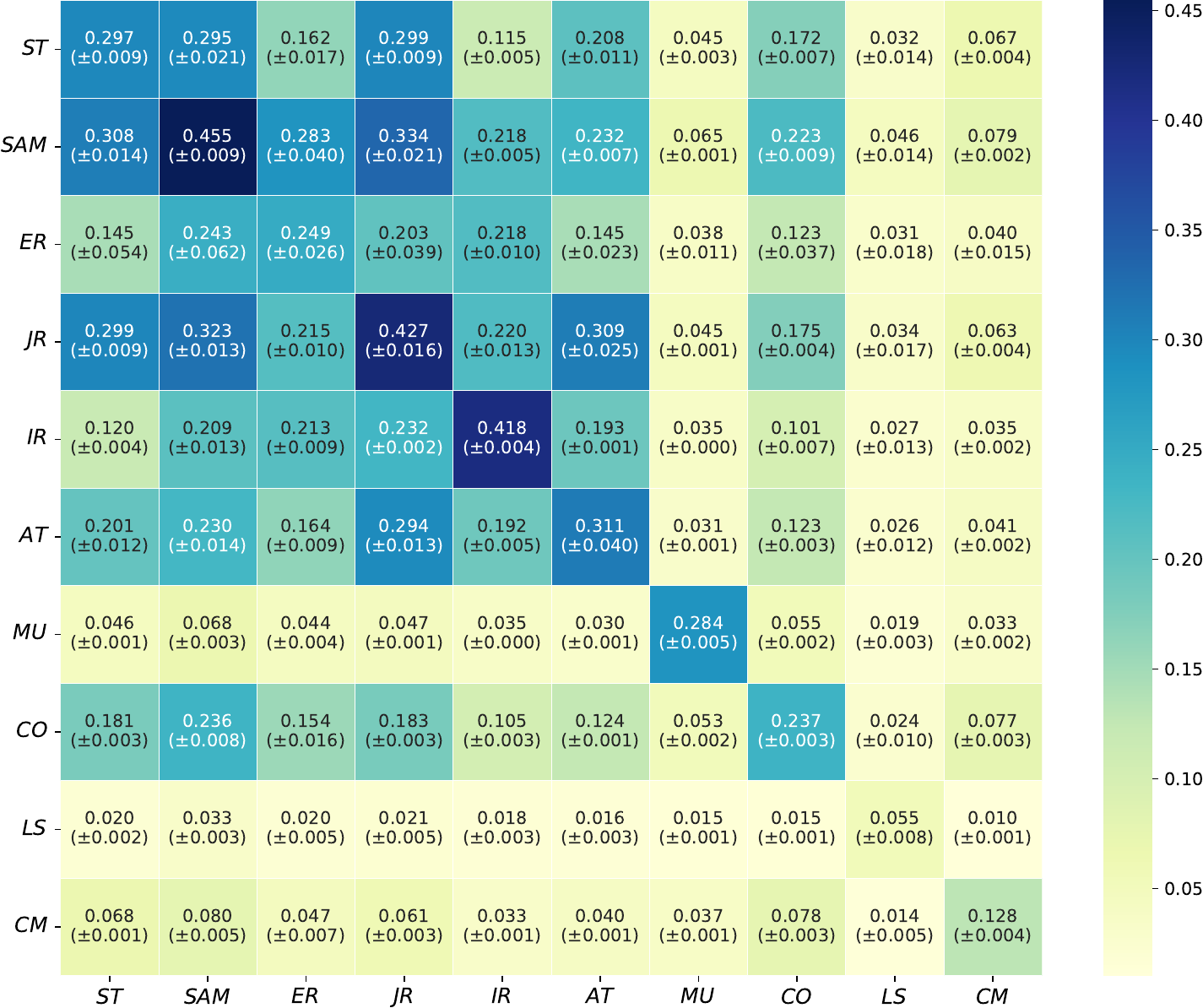}}
\subcaptionbox{ResNet50 trained over ImageNette }{\includegraphics[width=0.45\textwidth]{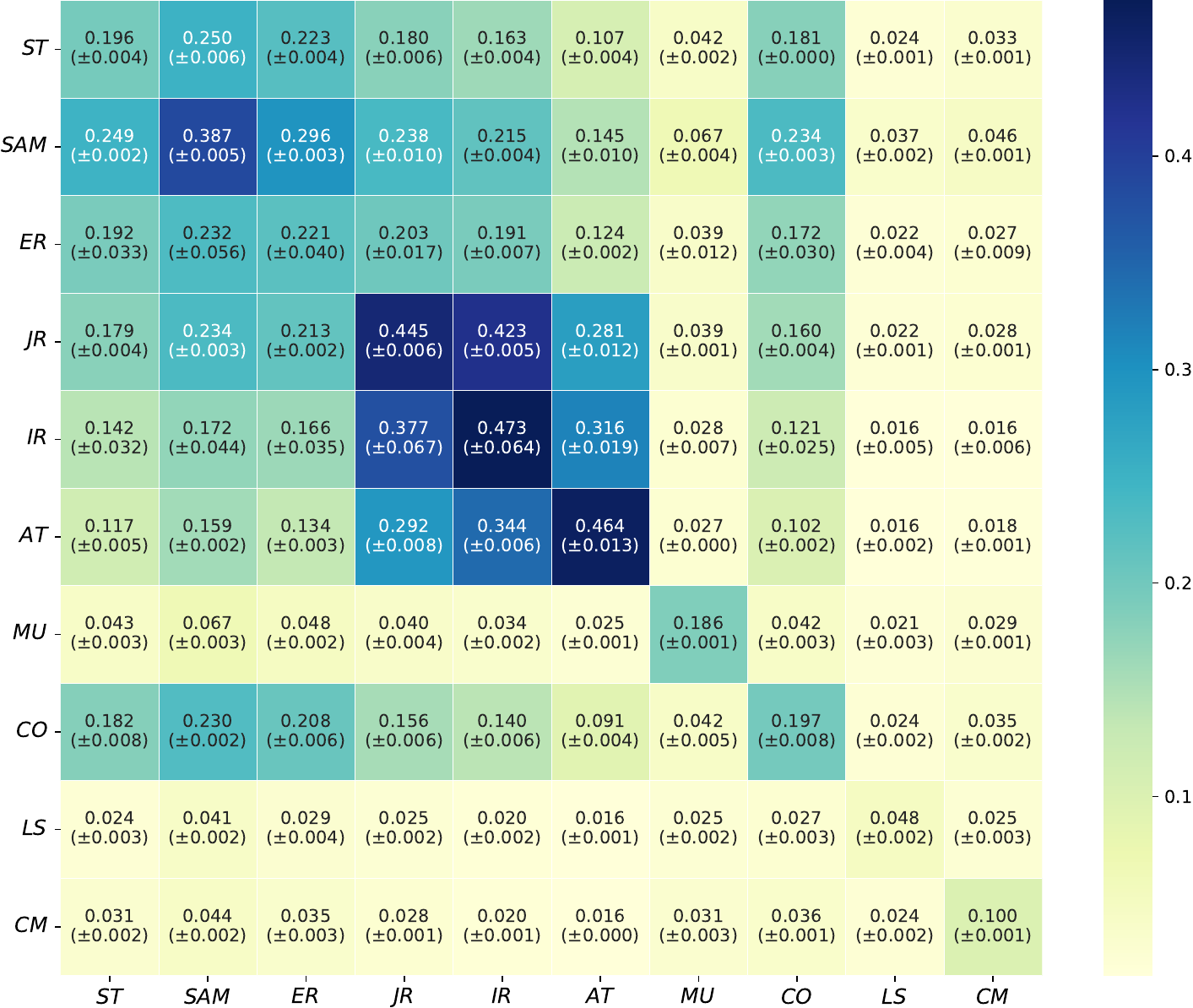}}

\caption{\textbf{Intra-architecture} gradient similarity between different mechanisms. Each cell reports the average result of 3 unique source-target pairs while all pair of models are trained with different random seeds, hence each matrix is asymmetric.}
\label{fig:heatmap-gs}
\vspace{-2mm}
\end{figure*}

\cref{fig:MS-GR-cifar10,fig:MS-GR-imagenette} depict the model smoothness. 
We can observe  gradient regularizations generally result in smoother models (\ie, smaller values) than the baseline where the regularization magnitude is 0 (\ST), and large magnitudes correspond  to better smoothness.
Moreover, input gradient regularizations  yield much better smoothness   than weight gradient regularizations, particularly on CIFAR-10. 
This indicates that  the transfer effect from weight space to input space suggested in \cite{Dherin2022why} does exist, yet is somewhat marginal and cannot outperform   direct regularizations in input space.

\subsection{Trade-off Under Gradient Regularization}
\label{trade-off-GR}

One might presume that input gradient regularization (\JR, \IR) should clearly outperform  weight gradient regularization (\SAM, \ER) in terms of transferability, which agrees with the conclusion that less complexity should exhibit better transferability in \security19, given that smoothness is positively correlated to the  model complexity  (as analyzed in \cref{analysis}, \textbf{Q3}).
However, our experimental results counter this:
\begin{itemize}[leftmargin=11pt]
    \item In \cref{tab:cifar10-ASR} (CIFAR-10), the fact is that \SAM ~is generally better than \IR ~(11 out of 12 cases for untargeted attacks, 12 out of 12 for targeted), and  generally better than \JR ~(11 out of 12 cases for untargeted, 7 out of 12 for targeted).
    \item In \cref{tab:imagenette-ASR} (ImageNette), the fact is that \SAM ~is generally worse than \IR ~(11 out of 12 cases for untargeted 10 out of 12 for targeted), and consistently worse than \JR ~(12 out of 12 cases for both targeted and untargeted).
\end{itemize}
Therefore, \textbf{these inconsistent results do not allow us to conclude that better smoothness always leads to better or worse transferability}.
This finding not only weakens the conclusion in \security19, but also suggests an overall examination of these training mechanisms for superior surrogates, as they may weigh differently on  similarity.




\section{Boosting Transferability with  Surrogates}\label{better-surrogates}
\black{Although we have concluded that gradient regularization can generally improve smoothness, the intangible nature of gradient similarity towards  unknown target models still makes it difficult to guide us in obtaining better surrogates. 
We thus propose a general solution that can simultaneously improve smoothness and similarity, under a strict black-box scenario where   target models are inaccessible. 
Our strategy involves first examining the pros and cons of various training mechanisms \wrt  smoothness and  similarity, and then combining their strengths to train better surrogates.}

While  gradient similarity cannot be directly measured when target models are unknown, one can assess the  similarity between all the training mechanisms under the intra-architecture case instead.  
This helps  to draw general conclusions that can be extended to the cross-architecture case as the angles between gradient directions follow the \textit{triangle inequality}. 
Therefore, we measure the similarity between the surrogates we obtained thus far, reported  in \cref{fig:heatmap-gs}.
\cref{tab:source-target similarities} in \cref{analysis}, \textbf{Q2} can support the consistency posteriorly with  observations in the cross-architecture case.
As a result, we obtain surrogates featuring both superior smoothness and similarity on average, and extensive experiments demonstrate their superiority in transferability.


\subsection{A General Route for  Better Surrogates}
We first elaborate on some general observations about gradient regularization mechanisms as follows:

\begin{table*}[!htbp]\scriptsize
\setlength{\abovecaptionskip}{0cm}
\setlength{\belowcaptionskip}{-0cm}
\centering
\caption{Untargeted and targeted transfer ASRs 
when incorporating  AE generation strategies under the different $L_\infty$ budgets on CIFAR-10. Corresponding results are different from those of  \cref{tab:cifar10-ASR} due to different randomization. In \%.  }
\setlength{\tabcolsep}{1pt}

\resizebox{\textwidth}{!}
{

\begin{tabular}{c|cccccccc|cccccccc}

\toprule              & \multicolumn{8}{c|}{\textbf{Untargeted}}                                                                                                & \multicolumn{8}{c}{\textbf{Targeted}}                                                                                                                                                                                                                                    \\ \hline 
\multicolumn{1}{l|}{} & \multicolumn{4}{c|}{4/255}                                                                                                       & \multicolumn{4}{c|}{8/255}                                                                                                    & \multicolumn{4}{c|}{4/255}                                                                                                       & \multicolumn{4}{c}{8/255}                                                                                                    \\
\multicolumn{1}{l|}{} & \multicolumn{1}{l}{Res-50} & \multicolumn{1}{l}{VGG16} & \multicolumn{1}{l}{Inc-V3} & \multicolumn{1}{l|}{Dense-121}    & \multicolumn{1}{l}{Res-50} & \multicolumn{1}{l}{VGG16} & \multicolumn{1}{l}{Inc-V3} & \multicolumn{1}{l|}{Dense-121} & \multicolumn{1}{l}{Res-50} & \multicolumn{1}{l}{VGG16} & \multicolumn{1}{l}{Inc-V3} & \multicolumn{1}{l|}{Dense-121}    & \multicolumn{1}{l}{Res-50} & \multicolumn{1}{l}{VGG16} & \multicolumn{1}{l}{Inc-V3} & \multicolumn{1}{l}{Dense-121} \\ \hline
\textit{ST}                    &  55.0$_{\pm2.4}$                        & 43.3$_{\pm2.1}$                     & 54.9$_{\pm2.5}$                           & \multicolumn{1}{c|}{71.7$_{\pm1.2}$}          & 79.6$_{\pm2.6}$                        & 71.3$_{\pm3.2}$                     & 80.2$_{\pm3.8}$                           & 93.1$_{\pm1.3}$                            & 19.0$_{\pm1.2}$                        & 13.2$_{\pm1.0}$                     & 21.1$_{\pm1.5}$                           & \multicolumn{1}{c|}{36.4$_{\pm1.8}$}          & 41.1$_{\pm4.7}$                        & 34.3$_{\pm3.7}$                     & 45.8$_{\pm4.2}$                           & 70.8$_{\pm4.6}$                           \\
\textit{SAM}                 & 76.8$_{\pm4.7}$                        & 65.0$_{\pm5.3}$                     & 74.9$_{\pm0.6}$                           & \multicolumn{1}{c|}{87.4$_{\pm1.2}$}          & 97.5$_{\pm1.2}$                        & 94.3$_{\pm2.2}$                     & 97.0$_{\pm0.2}$                           & 99.5$_{\pm0.1}$                            & 37.1$_{\pm6.4}$                        & 27.2$_{\pm5.0}$                     & 38.5$_{\pm2.0}$                           & \multicolumn{1}{c|}{55.8$_{\pm3.7}$}           & 77.2$_{\pm9.2}$                        & 68.3$_{\pm8.4}$                     & 78.4$_{\pm2.7}$                           & 93.4$_{\pm0.9}$                            \\
\textit{SAM\&JR}                & \textbf{81.2$_{\pm0.7}$}               & \textbf{70.2$_{\pm0.7}$}            & \textbf{79.4$_{\pm0.6}$}                  & \multicolumn{1}{c|}{\textbf{91.3$_{\pm0.1}$}} & \textbf{98.7$_{\pm0.3}$}               & \textbf{96.6$_{\pm0.1}$}            & \textbf{98.3$_{\pm0.0}$}                   & \textbf{99.8$_{\pm0.0}$}                   & \textbf{43.6$_{\pm2.1}$}               & \textbf{32.9$_{\pm1.2}$}            & \textbf{45.3$_{\pm2.1}$}                  & \multicolumn{1}{c|}{\textbf{64.5$_{\pm0.9}$}} & \textbf{85.4$_{\pm1.4}$}               & \textbf{77.2$_{\pm9.2}$}            & \textbf{86.8$_{\pm1.9}$}                  & \textbf{97.2$_{\pm0.3}$}                  \\ \hline
\textit{ST}+MI                   & 58.4$_{\pm2.6}$                        & 45.0$_{\pm2.9}$                     & 58.7$_{\pm4.0}$                           & \multicolumn{1}{c|}{76.7$_{\pm2.3}$}           & 82.9$_{\pm3.1}$                        & 75.6$_{\pm4.1}$                     & 83.5$_{\pm3.6}$                           & 94.6$_{\pm1.2}$                            & 22.4$_{\pm1.8}$                        & 15.4$_{\pm1.9}$                     & 25.6$_{\pm2.8}$                           & \multicolumn{1}{c|}{46.0$_{\pm3.9}$}          & 44.8$_{\pm4.1}$                        & 40.2$_{\pm5.9}$                     & 50.4$_{\pm6.3}$                           & 76.1$_{\pm4.5}$                           \\
\textit{SAM}+MI                & 82.3$_{\pm3.9}$                        & 69.6$_{\pm5.3}$                     & 80.5$_{\pm0.1}$                           & \multicolumn{1}{c|}{91.7$_{\pm1.5}$}          & 98.1$_{\pm0.9}$                        & 95.5$_{\pm1.9}$                     & 97.8$_{\pm0.1}$                           & 99.7$_{\pm0.1}$                            & 46.7$_{\pm7.0}$                        & 33.2$_{\pm5.7}$                     & 49.0$_{\pm1.1}$                           & \multicolumn{1}{c|}{69.6$_{\pm3.8}$}          & 84.0$_{\pm6.9}$                        & 76.9$_{\pm5.8}$                     & 84.8$_{\pm1.2}$                           & 96.4$_{\pm0.9}$                           \\
\textit{SAM\&JR}+MI               & \textbf{85.0$_{\pm0.4}$}               & \textbf{73.8$_{\pm0.5}$}            & \textbf{83.3$_{\pm0.5}$}                  & \multicolumn{1}{c|}{\textbf{93.6$_{\pm0.1}$}} & \textbf{98.9$_{\pm0.2}$}               & \textbf{97.2$_{\pm0.2}$}            & \textbf{98.7$_{\pm0.1}$}                  & \textbf{99.8$_{\pm0.1}$}                   & \textbf{52.4$_{\pm2.0}$}               & \textbf{38.4$_{\pm1.5}$}            & \textbf{54.0$_{\pm1.7}$}                  & \multicolumn{1}{c|}{\textbf{74.5$_{\pm0.4}$}} & \textbf{89.6$_{\pm1.6}$}               & \textbf{82.9$_{\pm1.5}$}            & \textbf{90.6$_{\pm2.1}$}                  & \textbf{98.3$_{\pm0.3}$}                  \\ \hline
\textit{ST}+DIM               & 70.2$_{\pm3.8}$                        & 59.2$_{\pm2.7}$                     & 69.2$_{\pm3.5}$                          & \multicolumn{1}{c|}{83.6$_{\pm3.1}$}          & 91.9$_{\pm2.2}$                        & 88.3$_{\pm2.8}$                      & 91.7$_{\pm2.2}$                           & 97.2$_{\pm1.3}$                            & 33.6$_{\pm4.8}$                        & 25.7$_{\pm3.2}$                     & 36.6$_{\pm4.2}$                           & \multicolumn{1}{c|}{54.9$_{\pm6.3}$}          & 65.4$_{\pm6.3}$                        & 62.0$_{\pm6.2}$                     & 69.9$_{\pm5.3}$                           & 87.5$_{\pm5.3}$                           \\
\textit{SAM}+DIM            & \textbf{84.4$_{\pm2.6}$}               & 74.7$_{\pm2.7}$                     & \textbf{82.2$_{\pm1.8}$}                  & \multicolumn{1}{c|}{91.8$_{\pm2.2}$} & 98.9$_{\pm0.4}$                        & 97.7$_{\pm0.7}$                     & \textbf{99.1$_{\pm0.8}$}                  & 99.7$_{\pm0.2}$                   & 50.4$_{\pm4.6}$                        & 38.9$_{\pm3.2}$                     & 51.3$_{\pm4.2}$                           & \multicolumn{1}{c|}{67.8$_{\pm5.8}$} & 89.3$_{\pm3.5}$                        & 84.2$_{\pm2.2}$                      & 89.6$_{\pm3.0}$                           & 96.8$_{\pm1.9}$                           \\
\textit{SAM\&JR}+DIM           & 83.6$_{\pm1.6}$                        & \textbf{74.8$_{\pm1.0}$}            & 81.8$_{\pm1.4}$                           & \multicolumn{1}{c|}{\textbf{92.1$_{\pm0.6}$}}          & \textbf{99.2$_{\pm0.2}$}               & \textbf{98.2$_{\pm0.3}$}            & 99.0$_{\pm0.1}$                           & \textbf{99.8$_{\pm0.1}$}                   & \textbf{51.6$_{\pm3.9}$}               & \textbf{40.7$_{\pm2.4}$}            & \textbf{52.7$_{\pm2.9}$}                  & \multicolumn{1}{c|}{\textbf{70.2$_{\pm2.0}$}}          & \textbf{91.5$_{\pm2.9}$}                & \textbf{86.9$_{\pm2.8}$}            & \textbf{91.8$_{\pm2.6}$}                  & \textbf{98.1$_{\pm0.8}$}                  \\ \bottomrule
\end{tabular}
}

\label{tab:DIM-incorporated-untargeted-targeted}
\vspace{-2mm}
\end{table*}

\begin{itemize}[leftmargin=10pt]
\item \noindent\textbf{Input regularizations stably yield the best smoothness.}
\cref{fig:MS-GR-cifar10,fig:MS-GR-imagenette} show that \IR ~and \JR ~yield superior  smoothness and their   variances are  small. This does not hold for other smoothness-benefiting mechanisms  like \SAM, \ER.
Further, \cref{fig:heatmap-gs} suggests that \IR, \JR, and \AT ~align superiorly well with each other,  this could be attributed to the fact that they all have a direct penalty effect  on  the input space, as analyzed in \cref{sec:TC-AT,sec:input-space}.

\item \noindent\textbf{\SAM ~yields the best gradient alignments when input regularizations are not involved.} 
\cref{fig:heatmap-gs} shows that \SAM ~\textit{exclusively exhibits better alignment towards every mechanism} than \ST.
However,  other mechanisms do not exhibit this advantage.
Furthermore, if we exclude the diagonal elements, we can find \SAM ~yields the best gradient similarity in each row and column except for the cases of \IR, \JR, and \AT.
This suggests that \SAM ~has an underlying benefit on gradient  similarity, which all the other mechanisms do not hold. We defer the detailed analysis on this to \cref{analysis}, \textbf{Q7}. 

\item \textbf{No single mechanism can dominate both datasets.} 
From \cref{tab:cifar10-ASR,tab:imagenette-ASR}, we observe that \SAM ~yields the best ASRs on average on CIFAR-10.
However, it generally performs worse than \IR ~and \JR ~on ImageNette. 
\ER, on the other hand, consistently performs the worst   on two datasets. 
Apparently, there is no one mechanism that can adequately outperform the others on both datasets.

\end{itemize}


The above observations inspire us that \SAM ~and  input regularization  are highly complementary, 
and a straightforward solution to reach a better trade-off between  model smoothness and  gradient similarity is combining  \SAM ~with input regularization, namely \SAM\&\IR ~or \SAM\&\JR.
In other words, we believe \SAM\&\IR ~and \SAM\&\JR ~will bring significant improvement in  transferability.
To validate this hypothesis, we conduct experiments on both CIFAR-10 and ImageNette and report the results on \cref{tab:cifar10-ASR,tab:imagenette-ASR} for comparison.
We can see that \SAM\&\JR~and \SAM\&\IR ~generally perform the best in both datasets.
Notably, we do not limit the route of constructing better surrogates to  \SAM\&\IR ~or \SAM\&\JR, and there may exist other effective combinations or approaches as long as  the trade-off  is well reconciled.

\subsection{Incorporating Surrogate-independent Methods}
\label{surrogate-independent}
We  show that the surrogates obtained from our methods (\eg, \SAM ~and \SAM\&\JR) can be integrated with   surrogate-independent methods to further improve transferability.

\subsubsection{Incorporating AE generation strategies}
Exploiting the generation process of AEs is a popular approach to improve transferability. The most representative methods include \textit{momentum iterative} (MI)~\cite{Dong_2018_CVPR} and \textit{diversity iterative} (DI)~\cite{xie2019improving}, and DI is frequently combined with MI and serves as a stronger baseline---\textit{diversity iterative with momentum} (DIM). 
Generally, these designs are interpreted as constructing adversarial examples that are less likely to fall into a poor local optimum  of the given surrogate. 
Thus they adjust the update direction with momentum to escape it or apply multiple transformations on inputs to mimic an averagely flat landscape.
We conduct experiments to verify that obtaining a better surrogate, which features less and wider local optima (analyzed in \cref{analysis}, \textbf{Q3}) and general adversarial direction, can further enhance these designs by promoting a larger lower bound.

In \cref{tab:DIM-incorporated-untargeted-targeted},  the transfer ASRs under  different generation strategies (PGD, MI, and DIM) show that using better surrogates can always significantly boost transferability. 
This suggests that the optima found in \ST~are inherently poorer than those  in \SAM ~and \SAM\&\JR.
Besides, without MI or DIM, \SAM\&\JR~using standard PGD  alone already outperforms \ST+MI and \ST+DMI, implying that the lower bound of transferability in  a better surrogate is sufficiently high, and  even surpasses the upper bound in   poor models. 

\begin{table*}[!htbp]\scriptsize
\setlength{\abovecaptionskip}{0cm}
\setlength{\belowcaptionskip}{-0cm}
\centering
\caption{Transfer ASRs of AEs crafted against single and LGV ensemble surrogates under the different $L_\infty$ budgets on CIFAR-10. 
Single surrogates without  fine-tuning are used for baselines and compared with LGV  ensembles.
We highlight the best ASRs in bold and those lower than baselines in red for each comparison.  In \%.  } 
\setlength{\tabcolsep}{0.9pt}
\resizebox{\textwidth}{!}
{

\begin{tabular}{c|cccccccc|cccccccc}
\toprule
\multicolumn{1}{l|}{}      & \multicolumn{8}{c|}{\textbf{Untargeted}}                                                                                                                                                                                                                                            & \multicolumn{8}{c}{\textbf{Targeted}}                                                                                                                                                                                                                                               \\ \hline
\multicolumn{1}{l|}{}      & \multicolumn{4}{c|}{4/255}                                                                                                                     & \multicolumn{4}{c|}{8/255}                                                                                                & \multicolumn{4}{c|}{4/255}                                                                                                                     & \multicolumn{4}{c}{8/255}                                                                                                 \\
\multicolumn{1}{l|}{}      & Res-50                       & VGG16                        & Inc-V3                       & \multicolumn{1}{c|}{Dense-121}                    & Res-50                       & VGG16                        & Inc-V3                       & Dense-121                    & Res-50                       & VGG16                        & Inc-V3                       & \multicolumn{1}{c|}{Dense-121}                    & Res-50                       & VGG16                        & Inc-V3                       & Dense-121                    \\ \hline
\ST                         & 55.0$_{\pm2.4}$                        & 43.3$_{\pm2.1}$                        & 54.9$_{\pm2.5}$                        & \multicolumn{1}{c|}{71.7$_{\pm1.2}$}                        & 79.6$_{\pm2.6}$                        & 71.3$_{\pm3.2}$                        & 80.2$_{\pm3.8}$                        & 93.1$_{\pm1.3}$                        & 19.0$_{\pm1.2}$                        & 13.2$_{\pm1.0}$                        & 21.1$_{\pm1.5}$                        & \multicolumn{1}{c|}{36.4$_{\pm1.8}$}                        & 41.1$_{\pm4.7}$                        & 34.3$_{\pm3.7}$                        & 45.8$_{\pm4.2}$                        & 70.8$_{\pm4.6}$                        \\
\textit{ST}+$\rm{LGV_{\text{SGD}}}$          & 70.7$_{\pm2.2}$                        & 57.1$_{\pm1.1}$                        & 70.1$_{\pm0.9}$                        & \multicolumn{1}{c|}{87.1$_{\pm0.6}$}                        & 94.0$_{\pm2.0}$                        & 89.1$_{\pm2.5}$                        & 94.3$_{\pm1.8}$                        & 99.3$_{\pm0.2}$                        & 29.3$_{\pm2.0}$                        & 20.4$_{\pm1.3}$                        & 32.8$_{\pm2.1}$                        & \multicolumn{1}{c|}{53.4$_{\pm2.1}$}                        & 65.6$_{\pm4.2}$                        & 56.8$_{\pm3.2}$                        & 71.3$_{\pm2.5}$                        & 92.0$_{\pm1.2}$                        \\
\ST+$\rm{LGV}_{\SAM}$          & 79.5$_{\pm0.8}$               & 66.8$_{\pm1.1}$               & 79.8$_{\pm1.0}$                        & \multicolumn{1}{c|}{92.3$_{\pm0.2}$}               & 98.6$_{\pm0.2}$               & 96.0$_{\pm0.5}$               & 98.7$_{\pm0.2}$               & \textbf{99.9$_{\pm0.0}$}                        & 39.4$_{\pm1.0}$                        & 27.7$_{\pm0.5}$                        & 44.4$_{\pm1.7}$                        & \multicolumn{1}{c|}{65.7$_{\pm0.9}$}                        & 82.5$_{\pm1.7}$               & 74.2$_{\pm1.0}$               & 87.4$_{\pm0.8}$               & 97.8$_{\pm0.2}$              \\
\ST+$\rm{LGV}_{\SAM\&\JR}$      & \textbf{82.0$_{\pm1.5}$}                        & \textbf{69.5$_{\pm1.7}$}               & \textbf{81.3$_{\pm1.3}$}               & \multicolumn{1}{c|}{\textbf{93.4$_{\pm0.3}$}}                        & \textbf{99.1$_{\pm0.3}$}                        & \textbf{97.1$_{\pm0.7}$}                        & \textbf{99.0$_{\pm0.2}$}                        & \textbf{99.9$_{\pm0.0}$}               & \textbf{42.9$_{\pm2.6}$}               & \textbf{30.3$_{\pm1.5}$}               & \textbf{46.7$_{\pm1.4}$}               & \multicolumn{1}{c|}{\textbf{68.1$_{\pm1.0}$}}               & \textbf{87.1$_{\pm3.5}$}                        & \textbf{79.1$_{\pm3.4}$}                        & \textbf{90.2$_{\pm1.7}$}                        & \textbf{98.5$_{\pm0.3}$}                        \\ \hline
\SAM                        & 76.8$_{\pm4.7}$                        & 65.0$_{\pm5.3}$                     & 74.9$_{\pm0.6}$                        & \multicolumn{1}{c|}{87.4$_{\pm1.2}$}                        & 97.5$_{\pm1.2}$                        & 94.3$_{\pm2.2}$                     & 97.0$_{\pm0.2}$                           & 99.5$_{\pm0.1}$                         & 37.1$_{\pm6.4}$                        & 27.2$_{\pm5.0}$                     & 38.5$_{\pm2.0}$                         & \multicolumn{1}{c|}{55.8$_{\pm3.7}$}                        & 77.2$_{\pm9.2}$                        & 68.3$_{\pm8.4}$                     & 78.4$_{\pm2.7}$                           & 93.4$_{\pm0.9}$                        \\
\textit{SAM}+$\rm{LGV_{\text{SGD}}}$           &  {\color[HTML]{ff0e38}70.8$_{\pm6.1}$} &  {\color[HTML]{ff0e38}57.2$_{\pm6.5}$} & {\color[HTML]{ff0e38} 69.2$_{\pm0.7}$} & \multicolumn{1}{c|}{{\color[HTML]{ff0e38} 84.2$_{\pm0.7}$}} & {\color[HTML]{ff0e38} 93.1$_{\pm3.9}$} & {\color[HTML]{ff0e38} 87.6$_{\pm5.3}$} & {\color[HTML]{ff0e38} 93.2$_{\pm1.4}$} & {\color[HTML]{ff0e38} 98.6$_{\pm0.3}$} & {\color[HTML]{ff0e38} 26.7$_{\pm1.9}$} & {\color[HTML]{ff0e38} 18.7$_{\pm0.1}$}  & {\color[HTML]{ff0e38} 31.5$_{\pm2.2}$} & \multicolumn{1}{c|}{{\color[HTML]{ff0e38} 51.3$_{\pm2.6}$}} & {\color[HTML]{ff0e38} 63.6$_{\pm14.1}$} & {\color[HTML]{ff0e38} 53.6$_{\pm11.8}$} & {\color[HTML]{ff0e38} 66.4$_{\pm5.4}$}  & {\color[HTML]{ff0e38} 87.6$_{\pm2.2}$} \\
\SAM+$\rm{LGV}_{\SAM}$         & 81.5$_{\pm4.0}$                        & 69.5$_{\pm4.8}$                        & 79.8$_{\pm1.3}$               & \multicolumn{1}{c|}{91.6$_{\pm1.3}$}                        & 98.9$_{\pm0.7}$                        & 96.9$_{\pm1.3}$               & 98.7$_{\pm0.1}$                        & \textbf{99.9$_{\pm0.0}$}                        & 42.1$_{\pm7.3}$                        & 30.5$_{\pm5.5}$                        & 44.3$_{\pm1.6}$                        & \multicolumn{1}{c|}{63.4$_{\pm1.9}$}                        & 86.2$_{\pm6.7}$               & 78.4$_{\pm6.6}$               & 87.9$_{\pm1.1}$               & 97.5$_{\pm0.4}$                        \\
\SAM+$\rm{LGV}_{\SAM\&\JR}$     & \textbf{82.7$_{\pm4.2}$}               & \textbf{71.2$_{\pm5.1}$}               & \textbf{80.9$_{\pm0.6}$}                        & \multicolumn{1}{c|}{\textbf{92.1$_{\pm1.2}$}}               & \textbf{99.1$_{\pm0.6}$}               & \textbf{97.4$_{\pm1.3}$}                        & \textbf{99.0$_{\pm0.3}$}               & \textbf{99.9$_{\pm0.0}$}               & \textbf{44.7$_{\pm7.7}$}               & \textbf{32.6$_{\pm5.9}$}               & \textbf{46.6$_{\pm0.6}$}               & \multicolumn{1}{c|}{\textbf{65.6$_{\pm1.7}$}}               & \textbf{87.6$_{\pm6.9}$}                        & \textbf{79.8$_{\pm7.1}$}                        & \textbf{89.1$_{\pm2.0}$}                        & \textbf{97.9$_{\pm0.5}$}               \\ \hline
\textit{SAM\&JR}                    & 81.2$_{\pm0.7}$                       & 70.2$_{\pm0.7}$                      & 79.4$_{\pm0.6}$                     & \multicolumn{1}{c|}{91.3$_{\pm0.1}$}                        & 98.7$_{\pm0.3}$                       & 96.6$_{\pm0.1}$                         & 98.3$_{\pm0.0}$                        & 99.8$_{\pm0.0}$                      & 43.6$_{\pm2.1}$                        & 32.9$_{\pm1.2}$                     & 45.3$_{\pm2.1}$                     & \multicolumn{1}{c|}{64.5$_{\pm0.9}$}                        & 85.4$_{\pm1.4}$                        & 77.2$_{\pm2.1}$                        & 86.8$_{\pm1.9}$                       & 97.2$_{\pm0.3}$                        \\
\textit{SAM\&JR}+$\rm{LGV_{\text{SGD}}}$       &  {\color[HTML]{ff0e38}73.8$_{\pm1.8}$} &  {\color[HTML]{ff0e38}61.1$_{\pm1.6}$} & {\color[HTML]{ff0e38}73.1$_{\pm1.8}$} & \multicolumn{1}{c|}{ {\color[HTML]{ff0e38}88.3$_{\pm0.5}$}} &  {\color[HTML]{ff0e38}95.8$_{\pm1.1}$} &  {\color[HTML]{ff0e38}90.9$_{\pm1.9}$} &  {\color[HTML]{ff0e38}95.6$_{\pm1.1}$} &  {\color[HTML]{ff0e38}99.4$_{\pm0.1}$} &  {\color[HTML]{ff0e38}35.6$_{\pm2.9}$} &  {\color[HTML]{ff0e38}25.1$_{\pm2.4}$} &  {\color[HTML]{ff0e38}39.0$_{\pm2.7}$} & \multicolumn{1}{c|}{{\color[HTML]{ff0e38}60.2$_{\pm1.7}$}} &  {\color[HTML]{ff0e38}70.5$_{\pm5.0}$} & {\color[HTML]{ff0e38}61.1$_{\pm5.9}$} &  {\color[HTML]{ff0e38}74.5$_{\pm5.7}$} & {\color[HTML]{ff0e38}93.6$_{\pm1.5}$} \\
\textit{SAM\&JR}+$\rm{LGV}_{\SAM}$     & 82.5$_{\pm0.7}$                        & 71.2$_{\pm1.0}$                        & 81.0$_{\pm0.4}$                        & \multicolumn{1}{c|}{93.0$_{\pm0.2}$}                        & 99.1$_{\pm0.1}$                        & 97.4$_{\pm0.2}$                        & 98.9$_{\pm0.1}$                        & \textbf{99.9$_{\pm0.1}$}                        & 44.0$_{\pm1.5}$ &  {\color[HTML]{ff0e38}32.2$_{\pm1.1}$}  & 46.8$_{\pm1.4}$                        & \multicolumn{1}{c|}{67.6$_{\pm0.8}$}                        & 87.6$_{\pm1.2}$                        & 80.2$_{\pm1.7}$                        & 89.7$_{\pm1.2}$                        & 98.4$_{\pm0.1}$                        \\
\textit{SAM\&JR}+$\rm{LGV}_{\SAM\&\JR}$ & \textbf{83.8$_{\pm0.8}$}               & \textbf{72.9$_{\pm1.2}$}               & \textbf{82.5$_{\pm0.9}$}               & \multicolumn{1}{c|}{\textbf{93.7$_{\pm0.1}$}}               & \textbf{99.3$_{\pm0.1}$}               & \textbf{97.8$_{\pm0.3}$}               & \textbf{99.3$_{\pm0.1}$}               & \textbf{99.9$_{\pm0.0}$}               & \textbf{46.9$_{\pm1.4}$}               & \textbf{34.8$_{\pm1.2}$}               & \textbf{49.5$_{\pm1.0}$}               & \multicolumn{1}{c|}{\textbf{70.0$_{\pm0.5}$}}               & \textbf{90.0$_{\pm0.9}$}               & \textbf{83.2$_{\pm0.7}$}               & \textbf{91.8$_{\pm0.7}$}               & \textbf{98.8$_{\pm0.1}$}               \\ \bottomrule
\end{tabular}
}

\label{tab:LGV-untargeted-targeted}
\vspace{-2mm}
\end{table*}

\subsubsection{Incorporating ensemble strategies}
Using   multiple surrogates  is  widely believed to update adversarial examples in a more general direction that benefits transferability.
Recently, \cite{gubriLGV2022} proposed  constructing the ensemble from a \textit{large geometric vicinity} (LGV), \ie, first fine-tuning a normally trained model with a high constant learning rate to obtain a set of intermediate models that belongs to a large vicinity, then optimizing  AEs iteratively against all these models. 
We conduct experiments  to verify  whether  \SAM~and \SAM\&\JR~can further enhance this ensemble  and analyze how their intrinsic mechanisms  correlate with LGV.
Specifically, we fine-tune models well-trained with  SGD\footnote{We refer SGD here to training models without any  augmentation or regularization mechanisms, not merely the particular optimizer itself.}, \SAM, and \SAM\&\JR ~on CIFAR-10 using  SGD, \SAM, and \SAM\&\JR ~with a constant learning rate of 0.05.

In \cref{tab:LGV-untargeted-targeted}, the transfer results show that LGV with \SAM, \SAM\&\JR~can always yield superior ASRs than standard SGD, indicating that the underlying property of \SAM ~and \JR~allows them to find solutions that benefit transferability. 
More intriguingly, fine-tuning models obtained by \SAM~and \SAM\&\JR~with  SGD always brings non-negligible negative impacts on transferability. 
These suggest solutions well-trained with \SAM~and \SAM\&\JR~lie in the \textit{basin of attractions} featuring better  transferability, 
and a relatively high constant learning rate with SGD will arbitrarily minimize the training loss while leaving the basin of attractions.
Note that a relatively high constant learning rate is necessary for a large vicinity \cite{gubriLGV2022}, thus we can infer that a large vicinity may not be a determining factor for transferability since a single surrogate  can outperform an ensemble of poor surrogates (see \SAM\&\JR ~and \textit{ST}+$\rm{LGV_{\text{SGD}}}$).
Instead, a dense subspace  whose geometry is intrinsically connected to transferability  is more desirable.

\subsubsection{Incorporating other strategies}
We  take  AE generation and ensemble strategies, which are the most well-explored  methods, to exemplify the universality of better surrogates. 
We are aware of other factors that could also contribute to transferability, such as loss object modification~\cite{zhao2021on}, architecture selection~\cite{Wu2020Skip}, $l_2$-norm consideration~\cite{maoP2022TransferAttacks2022},  unrestricted generation~\cite{hu2022protecting,zhou2023advclip}, patch-based AEs~\cite{hu2021advhash}.
We leave it to our future work to investigate whether they can be mildly integrated with our design.

\section{Analyses and Discussions}\label{analysis}


\noindent\textbf{Q1: Does the trade-off between  smoothness and  similarity really ``significantly indicate" transferability \wrt surrogates?}

Our theoretical analysis and experimental validation have demonstrated the significant role of smoothness and similarity in regulating transferability among surrogate-side factors. 
To investigate whether their impact on transferability is significant, we calculate the respective Pearson correlation coefficients between the surrogate-related factors and ASR and run  \textit{ordinary least squares} (OLS) regressions to measure their joint effect on ASR. 
\cref{tab:statistical-results} reports the results under each fixed target model and adversarial budget,
which indicates that both smoothness and similarity are highly correlated to ASR, with a clear superiority over other factors.
Besides, the joint effect of smoothness and similarity  ($R^2_1$) is significant in each case and very close to that of the 4 factors combined ($R^2_2$).
Further, it does not exhibit a strong correlation with $\epsilon$ under each target model, indicating the joint effect is  indeed strong and robust.

\black{Interestingly, we can observe that as  $\epsilon$ increases, the relevance  \wrt smoothness alone increases, and the relevance  \wrt similarity alone decreases.
This agrees with the intuition: smoothness depicts the gradient invariance around  origin, while similarity captures the alignment of two gradients in different directions.
As $\epsilon$ approaches 0, smoothness becomes less important for the gradient may not vary too much in a small region, and the legitimate direction becomes more crucial since there is limited scope for success.
As $\epsilon$ gets bigger, the alignment at the origin requires better smoothness to be further propagated into a bigger region.
These suggest an adaptive similarity-smoothness trade-off surrogate could be  preferable given different $\epsilon$.\\
}

\begin{table}[t!]
\renewcommand\arraystretch{0.95}
\setlength{\abovecaptionskip}{0.05cm}
\centering
\setlength{\tabcolsep}{0.9pt}
\caption{The Pearson correlation coefficients ($r$)    between ASR  and  accuracy ($Acc$), fooling probability ($FP$),  model smoothness ($MS$),  and gradient similarity ($GS$). For the coefficient of determination ($R^2$) in OLS, $R_1^2$ depicts the joint effect of $MS$ and $GS$ on ASR and $R_2^2$ depicts the joint effect of these 4 factors.}
\resizebox{0.48\textwidth}{!}
{
\begin{tabular}{ccccccccccccc}
\toprule
\multicolumn{13}{c}{CIFAR-10}                                                                                                                                                                                      \\ \hline
\multicolumn{1}{c|}{\multirow{2}{*}{}} & \multicolumn{3}{c|}{ResNet50}               & \multicolumn{3}{c|}{VGG16}                  & \multicolumn{3}{c|}{InceptionV3}            & \multicolumn{3}{c}{DenseNet121} \\
\multicolumn{1}{c|}{}                  & $\frac{\textrm{4}}{\textrm{255}}$ & $\frac{\textrm{8}}{\textrm{255}}$ & \multicolumn{1}{c|}{$\frac{\textrm{16}}{\textrm{255}}$} & $\frac{\textrm{4}}{\textrm{255}}$ & $\frac{\textrm{8}}{\textrm{255}}$ & \multicolumn{1}{c|}{$\frac{\textrm{16}}{\textrm{255}}$} & $\frac{\textrm{4}}{\textrm{255}}$ & $\frac{\textrm{8}}{\textrm{255}}$ & \multicolumn{1}{c|}{$\frac{\textrm{16}}{\textrm{255}}$} & $\frac{\textrm{4}}{\textrm{255}}$ & $\frac{\textrm{8}}{\textrm{255}}$ & $\frac{\textrm{16}}{\textrm{255}}$   \\ \hline
\multicolumn{1}{c|}{$r(Acc, ASR)$}        & 0.121 & -0.049 & \multicolumn{1}{c|}{-0.255}  & 0.032 & -0.095 & \multicolumn{1}{c|}{-0.259}  & 0.211 & 0.028 & \multicolumn{1}{c|}{-0.187}  & 0.286    & 0.130    & -0.107     \\
\multicolumn{1}{c|}{$r(FP, ASR)$}        & 0.118 & 0.106 & \multicolumn{1}{c|}{0.055}  & 0.197 & 0.164 & \multicolumn{1}{c|}{0.098}  & 0.105 & 0.091 & \multicolumn{1}{c|}{0.040}  & 0.051    & 0.008    & -0.033     \\
\multicolumn{1}{c|}{$r(MS, ASR)$}        & 0.540 & 0.662 & \multicolumn{1}{c|}{0.753}  & 0.614 & 0.700 & \multicolumn{1}{c|}{0.774}  & 0.478 & 0.630 & \multicolumn{1}{c|}{0.732}  & 0.418    & 0.581    & 0.701     \\
\multicolumn{1}{c|}{$r(GS, ASR)$}        & 0.912 & 0.837 & \multicolumn{1}{c|}{0.759}  & 0.924 & 0.878 & \multicolumn{1}{c|}{0.816}  & 0.864 & 0.743 & \multicolumn{1}{c|}{0.653}  & 0.875    & 0.710    & 0.585     \\ \hline
\multicolumn{1}{c|}{$R^{2}_{1}$}                & 0.909 & 0.885 & \multicolumn{1}{c|}{0.879}  & 0.917 & 0.906 & \multicolumn{1}{c|}{0.893}  & 0.870 & 0.827 & \multicolumn{1}{c|}{0.838}  & 0.880    & 0.773    & 0.765     \\
\multicolumn{1}{c|}{$R^{2}_{2}$}                & 0.943 & 0.892 & \multicolumn{1}{c|}{0.894}  & 0.931 & 0.907 & \multicolumn{1}{c|}{0.913}  & 0.911 & 0.843 & \multicolumn{1}{c|}{0.852}  & 0.926    & 0.823    & 0.799     \\ \hline

\multicolumn{13}{c}{ImageNette}                                                                                                                                                                                    \\ \hline
\multicolumn{1}{c|}{\multirow{2}{*}{}} & \multicolumn{3}{c|}{VGG16}                  & \multicolumn{3}{c|}{DenseNet121}            & \multicolumn{3}{c|}{MobileNetV2}            & \multicolumn{3}{c}{Xception}    \\
\multicolumn{1}{c|}{}                  & $\frac{\textrm{4}}{\textrm{255}}$ & $\frac{\textrm{8}}{\textrm{255}}$ & \multicolumn{1}{c|}{$\frac{\textrm{16}}{\textrm{255}}$} & $\frac{\textrm{4}}{\textrm{255}}$ & $\frac{\textrm{8}}{\textrm{255}}$ & \multicolumn{1}{c|}{$\frac{\textrm{16}}{\textrm{255}}$} & $\frac{\textrm{4}}{\textrm{255}}$ & $\frac{\textrm{8}}{\textrm{255}}$ & \multicolumn{1}{c|}{$\frac{\textrm{16}}{\textrm{255}}$} & $\frac{\textrm{4}}{\textrm{255}}$ & $\frac{\textrm{8}}{\textrm{255}}$ & $\frac{\textrm{16}}{\textrm{255}}$    \\ \hline
\multicolumn{1}{c|}{$r(Acc, ASR)$}       & 0.218 & 0.128 & \multicolumn{1}{c|}{-0.047}  & 0.133 & 0.023 & \multicolumn{1}{c|}{-0.165}  & 0.093 & -0.014 & \multicolumn{1}{c|}{-0.196}  & 0.072    & -0.028    & -0.177     \\
\multicolumn{1}{c|}{$r(FP, ASR)$}        & 0.121 & 0.120 & \multicolumn{1}{c|}{-0.055}  & 0.146 & 0.117 & \multicolumn{1}{c|}{-0.081}  & -0.034 & 0.003 & \multicolumn{1}{c|}{-0.151}  & 0.062    & -0.012    & -0.156     \\
\multicolumn{1}{c|}{$r(MS, ASR)$}        & 0.410 & 0.472 & \multicolumn{1}{c|}{0.516}  & 0.458 & 0.503 & \multicolumn{1}{c|}{0.492}  & 0.495 & 0.576 & \multicolumn{1}{c|}{0.625}  & 0.499    & 0.571    & 0.624     \\
\multicolumn{1}{c|}{$r(GS, ASR)$}        & 0.864 & 0.829 & \multicolumn{1}{c|}{0.764}  & 0.859 & 0.855 & \multicolumn{1}{c|}{0.853}  & 0.789 & 0.766 & \multicolumn{1}{c|}{0.780}  & 0.753    & 0.722    & 0.721     \\ \hline
\multicolumn{1}{c|}{$R^{2}_{1}$}                & 0.854 & 0.841 & \multicolumn{1}{c|}{0.780}  & 0.821 & 0.843 & \multicolumn{1}{c|}{0.832}  & 0.715 & 0.742 & \multicolumn{1}{c|}{0.804}  & 0.688    & 0.704    & 0.752     \\
\multicolumn{1}{c|}{$R^{2}_{2}$}                & 0.871 & 0.847 & \multicolumn{1}{c|}{0.798}  & 0.854 & 0.850 & \multicolumn{1}{c|}{0.867}  & 0.750 & 0.752 & \multicolumn{1}{c|}{0.833}  & 0.712    & 0.707    & 0.781     \\ 
\bottomrule
\end{tabular}
}
\label{tab:statistical-results}
\vspace{-2mm}
\end{table}

\noindent\textbf{Q2: What if the target model is not trained on $\mathcal{D}$?} 
So far, we presume the target models are trained within the normal data distribution $\mathcal{D}$ to lay out our analysis and conclusions, which is also set by default in \security19.
\black{To verify the generality of the obtained conclusion, we deliver the results here when the data distribution restrictions of target model $\mathcal{G}$ is relaxed.
First, we generalize \cref{hypothesis1} as:}
\begin{hypothesis}[\textbf{$\mathcal{D}$-surrogate  aligns general target models better than $\mathcal{D}'$-surrogate}]
\label{hypothesis2}
Suppose surrogates $\mathcal{F}_{\mathcal{D}}$ and $\mathcal{F}_{\mathcal{D}'}$ are trained on $\mathcal{D}$ and  $\mathcal{D}'$  respectively. 
Denote the data distribution of target model $\mathcal{G}_{\mathcal{D}_t}$ as $\mathcal{D}_t \in \mathcal{P}_{\mathcal{X}\times\mathcal{Y}}$. 
They  share a joint training loss $\ell$.
Then $\mathbb{E}_{\mathcal{D}_t\sim P_{\mathcal{X}\times\mathcal{Y}}} [\tilde{\mathcal{S}}_\mathcal{D}\left(\ell_{\mathcal{F}_{\mathcal{D}'}}, \ell_{\mathcal{G}_{\mathcal{D}_t}}\right)
- \tilde{\mathcal{S}}\left(\ell_{\mathcal{F}_{\mathcal{D}}}, \ell_{\mathcal{G}_{\mathcal{D}_t}}\right)]  < 0$ generally holds.
\end{hypothesis}

This hypothesis considers a more practical case that the adversary has no knowledge of  $\mathcal{D}_t$, and tries to attack $\mathcal{G}_{\mathcal{D}_t}$ with AEs constructed on  normal samples $(x, y) \in \mathcal{D}$.
The adversary hopes its surrogate aligns well with any $\mathcal{G}_{\mathcal{D}_t}$ regardless of what $\mathcal{D}_t$ is.
We show that $\mathcal{D}$-surrogate $\mathcal{F}_\mathcal{D}$ is still better than $\mathcal{D}'$-surrogate $\mathcal{F}_\mathcal{D}'$ based on the expectation over various $\mathcal{D}_t$.

\cref{tab:source-target similarities} reports the general surrogate-target gradient similarities. 
It's obvious that $\mathcal{D}'$ surrogates (above \AT) generally worsen the gradient alignments as their increment scores are all negative, except for \AT ~with a small budget $\epsilon=0.01$. 
Conversely, the  scores of $\mathcal{D}$ surrogates (below \AT)   are all non-negative and most of them improve the general alignment.
Note that when $\mathcal{D}'$ happens to be $\mathcal{D}_t$, the gradient similarity between $\mathcal{F}_\mathcal{D}'$ and $\mathcal{G}_{\mathcal{D}_t}$   is high, especially for \CM ~and \MU.
However, these surrogates  lower the similarity towards all other targets.
In general, they are not good surrogates since $\mathcal{D}_t$ is unknown to the adversary.

The same observations apply to the transfer ASRs. 
\cref{tab:source-target-transferability} reports the transfer ASRs against various $\mathcal{G}_{\mathcal{D}_t}$ that are not trained on $\mathcal{D}$.
Even when surrogates have the same training mechanisms as the corresponding target models, the best surrogate is still either \SAM\&\IR ~or \SAM\&\JR, ~regardless of what $\mathcal{D}_t$ is. This finding is in line with the results in \cref{tab:cifar10-ASR}.

\begin{table}[t!]
\setlength{\abovecaptionskip}{0.05cm}
\centering
\setlength{\tabcolsep}{2pt}
\caption{\black{\textbf{Cross-architecture} gradient similarity between  surrogate-target pairs (ResNet18-VGG16) on CIFAR-10. 
The first row reports the results of \ST ~surrogate, below which report the difference to  \ST.
The last column rates the score of the surrogates based on the general alignment toward multiple targets. (Accumulate +1 for positive, -1 for negative values.)
} }
\resizebox{0.5\textwidth}{!}
{
\begin{tabular}{cc|cccccc|c}
\toprule
\multicolumn{2}{c|}{\diagbox{Surrogate  $\mathcal{F}$}{Target  $\mathcal{G}_{\mathcal{D}_t}$}}    & \ST             & \CO             & \CM             & \MU             & \LS             
& \begin{tabular}[c]{@{}c@{}}\AT  \end{tabular} & \begin{tabular}[c]{@{}c@{}}Increment\\ Score\end{tabular} \\ \hline
\rowcolor[HTML]{F5F5F5}
\multicolumn{1}{c|}{$\,\,\mathcal{D}\,\,$}   & \ST                 & {0.141}    & 0.135          & 0.056          & 0.038          & 0.014          & 0.131            & {0}               \\ \hline
\multicolumn{1}{c|}{\multirow{5}{*}{$\,\,\mathcal{D^{\prime}}\,\,$}}  & \CO                 & -0.038         & {+0.005}    & -0.001         & +0.004          & -0.001         & -0.041           & -2              \\
\multicolumn{1}{c|}{} & \CM                 & -0.105         & -0.084         & \textbf{+0.008} & -0.021         & -0.006         & -0.098           & -4              \\
\multicolumn{1}{c|}{} & \MU                 & -0.106         & -0.092         & -0.030         & \textbf{+0.122} & -0.003         & -0.102           & -4              \\
\multicolumn{1}{c|}{} & \LS                 & -0.094         & -0.098         & -0.033         & -0.013         & {\ul +0.008} & -0.090           & -4              \\ \cline{2-9}
\multicolumn{1}{c|}{} & \AT & -0.002         & -0.005         & -0.013         & +0.001          & +0.001          & { +0.047}      & {0}               \\ \hline
\multicolumn{1}{c|}{\multirow{6}{*}{$\,\,\mathcal{D}\,\,$}}  & \IR                 & {-0.002}    & {-0.014}    & {-0.020}    & {+0.008}    & {\ul +0.007}    & {+0.043}   & 0     \\ 
\multicolumn{1}{c|}{}  & \JR                 & {\ul +0.044}    & {\ul +0.028}    & {\ul +0.004}    & {+0.012}    & {+0.004}    & {\ul +0.072}   & \textbf{+6}     \\ 
\multicolumn{1}{c|}{}  & \ER                 & {+0.012}    & {+0.006}    & {-0.013}    & {+0.011}    & {+0.005}    & {+0.031}   & {\ul +4}     \\ 
\multicolumn{1}{c|}{}  & \SAM                & {\ul +0.055} & {\ul +0.057} & {\ul +0.007}    & {\ul +0.029}    & {+0.006}    & { +0.058}      & \textbf{+6}      \\
\cline{2-9}
\multicolumn{1}{c|}{}  & \SAM\&\IR        & {+0.012}    & {+0.001}    & {-0.021}    & {+0.016}    & \textbf{+0.010}    & {\ul +0.065}                                            & {\ul +4}    \\
\multicolumn{1}{c|}{}  & \SAM\&\JR        & \textbf{+0.078}    & \textbf{+0.072}    & \textbf{+0.008}    & {\ul +0.035}    & \textbf{+0.010}    & \textbf{+0.093}                                            & \textbf{+6}    \\
\bottomrule
\end{tabular}
}
\label{tab:source-target similarities}
\vspace{-2mm}
\end{table}

\vspace{3mm}
\noindent\textbf{Q3: What is the relationship between model smoothness and model complexity?}
Model complexity  has been investigated in \security19  to figure out whether it correlates with transferability, which is defined as the variability of the loss landscape, deriving from  the number of local optima of the surrogates.
Here we correlate the number of local optima with  model smoothness. It is obvious that a high-complexity model featuring  multiple local optima will cause a larger variance to the loss landscape of AE optimization.
Formally, for a feature $x$, the fewer optima in its $\epsilon$-neighbor, the smaller the variance would be. 
Thus, supposing the loss function used to optimize adversarial points is continuously differentiable, by promoting the smoothness, we can expect fewer optima \black{in the finite $\epsilon$-neighbor space} since all the local optima are wider. 
Consequently, model smoothness negatively correlates with model complexity, \ie, a smoother model indicates lower complexity.

\begin{table}[t!]
\setlength{\tabcolsep}{2pt}
\centering
\caption{Transfer ASRs of AEs crafted from various ResNet18 surrogates against VGG16 target models  trained with \AT ~and augmentations   on CIFAR-10. $\bigstar$ denotes that surrogates  have the same training methods as  the corresponding target models. `T' represents `Targeted' attack and `U' represents `Untargeted' attack.}
\resizebox{0.5\textwidth}{!}
{
\begin{tabular}{c|cc|cc|cc|cc|cc|cc}
\toprule
 & \multicolumn{2}{c|}{\AT, $\epsilon=0.01$} & \multicolumn{2}{c|}{\AT, $\epsilon=0.05$} & \multicolumn{2}{c|}{\CM}  & \multicolumn{2}{c|}{\CO}     & \multicolumn{2}{c|}{\LS}               & \multicolumn{2}{c}{\MU}        \\
                  & T          & U        & T          & U        & T       & U   & T          & U        & T          & U        & T       & U   \\ \hline
\ST                & 20.18             & 51.33             & 4.83              & 22.49             & 42.42          & 79.49       & 38.80             & 76.33             & 34.62             & 72.34             & 37.46          & 69.25    \\
\rowcolor[HTML]{F5F5F5}
$\bigstar$               & 43.01             & 76.96             & {\ul 35.54}       & 69.56             & 14.20          & 54.57     & 32.91             & 79.11             & 11.90             & 46.66             & 15.81          & 54.43       \\
\IR                & {\ul 60.08}       & {\ul 91.23}       & 31.84             & {\ul 73.98}       & 59.80          & 91.82      & 60.78             & 92.14             & 63.20             & 93.52             & 60.79          & 90.06      \\
\JR                & 57.24             & 83.54             & 14.57             & 41.77             & 73.44          & 92.47      & {\ul 74.77}       & 93.03             & 72.27             & 92.30             & 68.34          & 86.90     \\
\ER                & 32.61             & 77.14             & 9.55              & 40.41             & 41.94          & 86.92       & 43.04             & 87.39             & 44.51             & 87.77             & 42.12          & 82.66   \\
\SAM               & 40.35             & 79.31             & 8.41              & 35.05             & {\ul 75.34}    & {\ul 96.59}   & 67.72             & {\ul 94.84}       & 62.68             & 93.57             & 65.37          & 90.43 \\
\SAM\&\IR            & \textbf{71.88}    & \textbf{94.01}    & \textbf{48.38}    & \textbf{80.88}    & 70.81          & 93.75        & 73.05             & 94.31             & {\ul 74.02}       & {\ul 95.47}       & {\ul 70.81}    & {\ul 92.68}  \\
\SAM\&\JR            & 58.81             & 89.23             & 15.11             & 47.11             & \textbf{83.56} & \textbf{97.57}  & \textbf{80.19}    & \textbf{97.02}    & \textbf{77.95}    & \textbf{96.92}    & \textbf{75.06} & \textbf{94.49}\\
\bottomrule
\end{tabular}
}

\label{tab:source-target-transferability}

\vspace{-2mm}
\end{table}

\vspace{3mm}
\noindent\textbf{Q4: Does a smaller model generalization gap imply higher adversarial transferability?} 
It is natural to raise this question as the idea of adversarial transferability in a way resembles the concept of  model generalization. The former captures the effectiveness of AEs against unseen models, while the latter evaluates the model's ability to adapt properly to previously unseen data. 
A concurrent work \cite{wang2022role} attempts to show the positive relationship between model generalization  and adversarial transferability, by applying a spectral normalization method constraining the $L_2$-operator norm of the surrogate's weight matrix to empirically demonstrate the increase of both transferability and generalization.

Nevertheless, we emphasize that we cannot over-generalize the alignment between generalization and  transferability. 
It is easy to find some cases that boost generalization yet degrades  transferability (\eg, through damaging  similarity). 
For example, the results in \cref{tab:cifar10-ASR,tab:imagenette-ASR,tab:autoattack,tab:API-results} have shown  augmentations that boost generalization  could sacrifice transferability. 
More typically, the strong data augmentations that yield the smallest generalization gaps (see \cref{clean-acc}) perform the worst in transfer attacks.
In fact, the product of spectral norm  of  model's layer-wise weight matrices is a loose upper bound of  $\|\nabla_x f_\theta(x)\|_2$~\cite{NEURIPS2020_ab731488,yoshida2017spectral}. 
When $\|\nabla_x f_\theta(x)\|_2 \rightarrow 0$, we have $\|\nabla_x f_\theta(x)\|_F\rightarrow 0$. Utilizing \cref{proposition1}, we also have $\|\nabla_x \ell(f_\theta(x),y)\| \rightarrow 0$. This suggests that spectral normalization actually benefits transferability by promoting smoothness.

\vspace{3mm}
\noindent\textbf{Q5: What is restraining the adversarial transferability of ViT?}
Recently, ViT \cite{dosovitskiy2021an} receives notable attention for its outstanding performance  on various tasks.
Meanwhile, some works \cite{shao2022on,Mahmood_2021_ICCV} have demonstrated that ViT models exhibit lower adversarial transferability than CNNs when used as surrogates, and a recent work \cite{naseer2022on} argues that this is because the current  attacks are not strong enough to fully exploit the transferability potential of ViT.
Complementarily,   we provide a different viewpoint that \textit{the transferability of ViT may have been restrained by its own  training paradigm}.
ViT is known to have poor trainability~\cite{Chen_2021_ICCV,chenICLR2022When2022}, and thus generally requires large-scale pre-training and fine-tuning with strong data augmentations.  
This training paradigm is used by default since  its appearance.
However, as we discussed in this paper, changing the data distribution could have negative impacts on  similarity.
Using data augmentation is plausible for accuracy performance, yet not ideal for transfer attacks.

\begin{table}[!t]
\renewcommand\arraystretch{0.95}
\setlength{\abovecaptionskip}{0.05cm}
\centering
\caption{Untargeted transfer ASRs of AEs crafted against ViT ImageNet surrogates on target models. In \%.  }
\setlength{\tabcolsep}{1pt}
\resizebox{0.48\textwidth}{!}
{
\begin{tabular}{c|c|cccc|cc}
\toprule
\multirow{2}{*}{Source} &\multirow{2}{*}{Acc} & \multicolumn{4}{c|}{Normally Trained}            & \multicolumn{2}{c}{Adversarially Trained}  \\
                      &  & Res101 & Dense121 & VGG16 & IncV3 & IncV3 $_{\AT}$ & IncResV3$_{\AT}$  \\ \hline
ViT-B/16  & \textbf{84.5}               & 23.8      & 37.6        & 50.0    & 45.9        & 48.6           & 27.6                       \\
ViT-B/16-\SAM   & 80.2          & \textbf{48.4}      & \textbf{69.9}        & \textbf{75.2}  & \textbf{70.2}        & \textbf{59.5}          & \textbf{40.5}                     \\ \hline
ViT-B/32       & \textbf{80.7}          & 19.6      & 35.7        & 50.0    & 49.1        & 51.9           & 29.7                         \\
ViT-B/32-\SAM  & 73.7           & \textbf{34.0}     & \textbf{55.8}        & \textbf{68.9}  & \textbf{64.3}        & \textbf{63.1}           & \textbf{43.9}                     \\ \hline
ViT-B/16+DIM  & \textbf{84.5}   & 60.7      & 75.3        & 79.9  & 80.0          & 61.4           & 43.8     \\
ViT-B/16-\SAM+DIM & 80.2 & \textbf{81.6}      & \textbf{95.6}        & \textbf{93.9}  & \textbf{94.6}        & \textbf{82.4}           & \textbf{79.9}  \\ \hline
ViT-B/32+DIM    & \textbf{80.7} & 48.4      & 66.4        & 75.2  & 81.6        & 68.6           & 54.4        \\
ViT-B/32-\SAM+DIM & 73.7 & \textbf{48.9}      & \textbf{78.1}        & \textbf{78.4}  & \textbf{82.5}        & \textbf{79.3}           & \textbf{72.3}          \\ \bottomrule
\end{tabular}
}
\label{tab:ViT-transfer}
\vspace{-2mm}
\end{table}

To validate this standpoint, we show that with a different training method such as  \SAM, the transferability of ViT can be significantly improved.
Our experiments are conducted on well-trained ImageNet classifiers, all publicly available from PyTorch image model library timm~\cite{rw2019timm}. 
We consider standard ViT surrogates ViT-B/16 and ViT-B/32, which are both first pre-trained on a large-scale dataset and then fine-tuned on ImageNet using a combination of \MU~and RandAugment~\cite{Cubuk_2020_CVPR_Workshops}. 
In contrast, their \SAM~versions ViT-B/16-\SAM ~and ViT-B/32-\SAM ~are merely trained from scratch.
We use both normally and adversarially trained CNN  classifiers as target models.  
We craft AEs on these 4 surrogates using both standard PGD and DIM  under $\epsilon=16/255$ on the  ImageNet-compatible dataset as in~\cite{qin2022boosting}.
The accuracy of these surrogates (referring to~\cite{rw2019timm}) and ASRs  are reported in  \cref{tab:ViT-transfer}.
The results indicate that standard ViT models have much weaker transferability than their \SAM~versions, despite the fact that they have higher accuracy.

\vspace{3mm}
\noindent\textbf{Q6: Will the transferability benefit of better surrogate in ``lab''  generalize to ``real-world'' environment?}
Recently, \SP22 suggests that many conclusions about adversarial  transferability done in the simplistic controlled ``lab" environments may not hold in the real world. 
We thus conduct  experiments   to   verify the generality of our conclusions using  AEs crafted against our  ImageNette surrogates.




Specifically, aligning with \SP22,  we conduct transfer attacks on three commercial MLaaS platforms---AWS Rekognition, Alibaba Cloud, and Baidu Cloud, using their code, and report the results in \cref{tab:API-results}. 
The results  are generally consistent with that in \cref{tab:imagenette-ASR}.
Particularly, we highlight here that our experimental results on MLaaS platforms  \textbf{contradict} the observation in \SP22 that ``adversarial training and data augmentation do not show  strong correlations to the transfer attack''.
\cref{tab:API-results} indicates that strong data augmentations like \CM ~and \MU ~generally exhibit lower ASRs than \ST ~for both targeted and untargeted attacks. 
In addition, we  can  observe that ``little robustness'' (see \AT, $\epsilon=1$) improves transferability while large adversarial magnitude (see \AT, $\epsilon=5$) degrades it.

\vspace{3mm}
\noindent\textbf{Q7: How does \SAM~yield superior alignment towards every training solution?}
Given its superior performance on transferability and gradient alignment, it is natural to wonder why \SAM ~exhibits this  property.
Recent works~\cite{pmlr-v180-gubri22a,li2023making} empirically demonstrate that attacking an ensemble of models in the distribution obtained via \textit{Bayesian} learning substantially improves the transferability.

\begin{table}[t!]\small
\setlength{\abovecaptionskip}{0.05cm}
\renewcommand\arraystretch{0.95}
\setlength{\belowcaptionskip}{-0.1cm}
\setlength{\tabcolsep}{2.5pt}
\centering
\caption{Transfer ASRs under the different $L_\infty$ budgets on 3 MLaaS platforms using AEs reported in \cref{tab:imagenette-ASR}. }
\resizebox{0.48\textwidth}{!}{

\begin{tabular}{c|cccccc|cccccc}
\toprule
\multicolumn{1}{l|}{}   & \multicolumn{6}{c|}{Untargeted}                                                                                                                                                                                                   & \multicolumn{6}{c}{Targeted}                                                                                                                                                                                                     \\ \hline
                        & \multicolumn{2}{c|}{AWS}                                                         & \multicolumn{2}{c|}{Baidu}                                                       & \multicolumn{2}{c|}{Aliyun}                                 & \multicolumn{2}{c|}{AWS}                                                         & \multicolumn{2}{c|}{Baidu}                                                       & \multicolumn{2}{c}{Aliyun}                                 \\
\multirow{-2}{*}{Model} & $\frac{\textrm{8}}{\textrm{255}}$                        & \multicolumn{1}{c|}{$\frac{\textrm{16}}{\textrm{255}}$}                       & $\frac{\textrm{8}}{\textrm{255}}$                        & \multicolumn{1}{c|}{$\frac{\textrm{16}}{\textrm{255}}$}                       & $\frac{\textrm{8}}{\textrm{255}}$                         & $\frac{\textrm{16}}{\textrm{255}}$                       & $\frac{\textrm{8}}{\textrm{255}}$                         & \multicolumn{1}{c|}{$\frac{\textrm{16}}{\textrm{255}}$}                       & $\frac{\textrm{8}}{\textrm{255}}$                         & \multicolumn{1}{c|}{$\frac{\textrm{16}}{\textrm{255}}$}                       & $\frac{\textrm{8}}{\textrm{255}}$                        & $\frac{\textrm{16}}{\textrm{255}}$                       \\ \hline
\textit{ST}             & 9.4                          & \multicolumn{1}{c|}{13.4}                         & 30.0                         & \multicolumn{1}{c|}{53.8}                         & 18.6                         & 52.2                         & 23.0                         & \multicolumn{1}{c|}{23.8}                         & 11.6                         & \multicolumn{1}{c|}{16.6}                         & 2.6                         & 6.4                          \\ \hline
\textit{MU}             & \cellcolor[HTML]{F2F2F2}8.6  & \multicolumn{1}{c|}{\cellcolor[HTML]{F2F2F2}10.8} & \cellcolor[HTML]{F2F2F2}20.4 & \multicolumn{1}{c|}{\cellcolor[HTML]{F2F2F2}39.6} & \cellcolor[HTML]{F2F2F2}12.4 & \cellcolor[HTML]{F2F2F2}33.8 & \cellcolor[HTML]{F2F2F2}20.4 & \multicolumn{1}{c|}{23.8}                         & \cellcolor[HTML]{F2F2F2}9.6  & \multicolumn{1}{c|}{\cellcolor[HTML]{F2F2F2}15.0} & \cellcolor[HTML]{F2F2F2}1.4 & \cellcolor[HTML]{F2F2F2}2.8  \\
\textit{CM}             & \cellcolor[HTML]{F2F2F2}8.0  & \multicolumn{1}{c|}{\cellcolor[HTML]{F2F2F2}10.8} & \cellcolor[HTML]{F2F2F2}19.4 & \multicolumn{1}{c|}{\cellcolor[HTML]{F2F2F2}35.0} & \cellcolor[HTML]{F2F2F2}12.4 & \cellcolor[HTML]{F2F2F2}29.2 & \cellcolor[HTML]{F2F2F2}22.2 & \multicolumn{1}{c|}{\cellcolor[HTML]{F2F2F2}23.0} & \cellcolor[HTML]{F2F2F2}10.0 & \multicolumn{1}{c|}{\cellcolor[HTML]{F2F2F2}11.2} & \cellcolor[HTML]{F2F2F2}1.2 & \cellcolor[HTML]{F2F2F2}3.0  \\
\textit{CO}             & 9.6                          & \multicolumn{1}{c|}{\cellcolor[HTML]{F2F2F2}13.0} & \cellcolor[HTML]{F2F2F2}27.2 & \multicolumn{1}{c|}{\cellcolor[HTML]{F2F2F2}48.4} & 18.6                         & \cellcolor[HTML]{F2F2F2}51.8 & 23.2                         & \multicolumn{1}{c|}{\cellcolor[HTML]{FCE4D6}26.4} & \cellcolor[HTML]{FCE4D6}14.0 & \multicolumn{1}{c|}{16.8}                         & \cellcolor[HTML]{F2F2F2}2.0 & \cellcolor[HTML]{FCE4D6}6.8  \\
\textit{LS}             & 9.6                          & \multicolumn{1}{c|}{\cellcolor[HTML]{F2F2F2}12.0} & \cellcolor[HTML]{F2F2F2}23.4 & \multicolumn{1}{c|}{\cellcolor[HTML]{F2F2F2}43.0} & \cellcolor[HTML]{F2F2F2}12.2 & \cellcolor[HTML]{F2F2F2}33.6 & \cellcolor[HTML]{F2F2F2}22.2 & \multicolumn{1}{c|}{\cellcolor[HTML]{F2F2F2}23.0} & \cellcolor[HTML]{FCE4D6}12.6 & \multicolumn{1}{c|}{\cellcolor[HTML]{F2F2F2}14.4} & \cellcolor[HTML]{F2F2F2}1.4 & \cellcolor[HTML]{F2F2F2}2.0  \\ \hline
\textit{AT}, $\epsilon = 1$         & 9.6                          & \multicolumn{1}{c|}{\cellcolor[HTML]{F8CBAD}24.0} & \cellcolor[HTML]{FCE4D6}38.2 & \multicolumn{1}{c|}{\cellcolor[HTML]{FCE4D6}60.4} & \cellcolor[HTML]{FCE4D6}35.0 & \cellcolor[HTML]{FCE4D6}82.6 & \cellcolor[HTML]{FCE4D6}24.8 & \multicolumn{1}{c|}{\cellcolor[HTML]{FCE4D6}33.2} & \cellcolor[HTML]{FCE4D6}14.8 & \multicolumn{1}{c|}{\cellcolor[HTML]{F8CBAD}27.6} & \cellcolor[HTML]{FCE4D6}3.6 & \cellcolor[HTML]{FCE4D6}13.8 \\
\textit{AT}, $\epsilon = 5$         & \cellcolor[HTML]{F2F2F2}7.8  & \multicolumn{1}{c|}{\cellcolor[HTML]{FCE4D6}16.0} & \cellcolor[HTML]{F2F2F2}22.2 & \multicolumn{1}{c|}{\cellcolor[HTML]{F2F2F2}53.4} & \cellcolor[HTML]{F2F2F2}15.6 & \cellcolor[HTML]{FCE4D6}60.0 & \cellcolor[HTML]{F2F2F2}22.0 & \multicolumn{1}{c|}{\cellcolor[HTML]{FCE4D6}27.6} & \cellcolor[HTML]{F2F2F2}11.0 & \multicolumn{1}{c|}{\cellcolor[HTML]{FCE4D6}17.0} & \cellcolor[HTML]{F2F2F2}1.2 & \cellcolor[HTML]{F2F2F2}5.0  \\ \hline
\textit{IR}             & \cellcolor[HTML]{FCE4D6}10.4 & \multicolumn{1}{c|}{\cellcolor[HTML]{FCE4D6}23.6} & \cellcolor[HTML]{FCE4D6}45.8 & \multicolumn{1}{c|}{\cellcolor[HTML]{FCE4D6}60.0} & \cellcolor[HTML]{F8CBAD}48.8 & \cellcolor[HTML]{F8CBAD}86.0 & \cellcolor[HTML]{FCE4D6}25.0 & \multicolumn{1}{c|}{\cellcolor[HTML]{F8CBAD}34.4} & \cellcolor[HTML]{F8CBAD}17.2 & \multicolumn{1}{c|}{\cellcolor[HTML]{FCE4D6}26.4} & \cellcolor[HTML]{FCE4D6}8.0 & \cellcolor[HTML]{F8CBAD}15.4 \\
\textit{JR}             & \cellcolor[HTML]{FCE4D6}9.8  & \multicolumn{1}{c|}{\cellcolor[HTML]{FCE4D6}22.6} & \cellcolor[HTML]{F8CBAD}46.4 & \multicolumn{1}{c|}{\cellcolor[HTML]{FCE4D6}60.6} & \cellcolor[HTML]{FCE4D6}48.2 & \cellcolor[HTML]{FCE4D6}79.6 & \cellcolor[HTML]{FCE4D6}26.8 & \multicolumn{1}{c|}{\cellcolor[HTML]{FCE4D6}32.0} & \cellcolor[HTML]{F8CBAD}17.2 & \multicolumn{1}{c|}{\cellcolor[HTML]{FCE4D6}25.0} & \cellcolor[HTML]{FCE4D6}7.8 & \cellcolor[HTML]{FCE4D6}12.4 \\
\textit{ER}             & \cellcolor[HTML]{FCE4D6}9.8  & \multicolumn{1}{c|}{13.4}                         & \cellcolor[HTML]{F2F2F2}28.2 & \multicolumn{1}{c|}{\cellcolor[HTML]{FCE4D6}56.2} & \cellcolor[HTML]{FCE4D6}24.0 & \cellcolor[HTML]{FCE4D6}56.6 & \cellcolor[HTML]{FCE4D6}23.8 & \multicolumn{1}{c|}{\cellcolor[HTML]{FCE4D6}26.4} & \cellcolor[HTML]{FCE4D6}14.2 & \multicolumn{1}{c|}{\cellcolor[HTML]{FCE4D6}20.0} & \cellcolor[HTML]{FCE4D6}3.2 & \cellcolor[HTML]{FCE4D6}8.0  \\
\textit{SAM}            & \cellcolor[HTML]{FCE4D6}10.2 & \multicolumn{1}{c|}{\cellcolor[HTML]{FCE4D6}16.4} & \cellcolor[HTML]{FCE4D6}37.2 & \multicolumn{1}{c|}{\cellcolor[HTML]{FCE4D6}59.2} & \cellcolor[HTML]{FCE4D6}36.2 & \cellcolor[HTML]{FCE4D6}68.2 & \cellcolor[HTML]{FCE4D6}23.4 & \multicolumn{1}{c|}{\cellcolor[HTML]{FCE4D6}30.0} & \cellcolor[HTML]{FCE4D6}13.8 & \multicolumn{1}{c|}{\cellcolor[HTML]{FCE4D6}19.2} & \cellcolor[HTML]{FCE4D6}4.6 & \cellcolor[HTML]{FCE4D6}9.8  \\ \hline
\textit{SAM}\&\textit{IR}         & \cellcolor[HTML]{F8CBAD}11.0 & \multicolumn{1}{c|}{\cellcolor[HTML]{F4B084}28.4} & \cellcolor[HTML]{F4B084}54.2 & \multicolumn{1}{c|}{\cellcolor[HTML]{F4B084}63.8} & \cellcolor[HTML]{F4B084}55.8 & \cellcolor[HTML]{F4B084}93.4 & \cellcolor[HTML]{F4B084}29.2 & \multicolumn{1}{c|}{\cellcolor[HTML]{F4B084}36.2} & \cellcolor[HTML]{FCE4D6}16.8 & \multicolumn{1}{c|}{\cellcolor[HTML]{F4B084}28.0} & \cellcolor[HTML]{F4B084}9.8 & \cellcolor[HTML]{F4B084}18.0 \\ 
\textit{SAM}\&\textit{JR}         & \cellcolor[HTML]{F4B084}11.4 & \multicolumn{1}{c|}{\cellcolor[HTML]{FCE4D6}16.2} & \cellcolor[HTML]{FCE4D6}45.8 & \multicolumn{1}{c|}{\cellcolor[HTML]{F8CBAD}61.2} & \cellcolor[HTML]{FCE4D6}48.0 & \cellcolor[HTML]{FCE4D6}79.0 & \cellcolor[HTML]{F8CBAD}25.4 & \multicolumn{1}{c|}{\cellcolor[HTML]{FCE4D6}33.2} & \cellcolor[HTML]{F4B084}18.8 & \multicolumn{1}{c|}{\cellcolor[HTML]{FCE4D6}21.8} & \cellcolor[HTML]{F8CBAD}8.2 & \cellcolor[HTML]{FCE4D6}11.8 \\
\bottomrule
\end{tabular}
}
\label{tab:API-results}
\vspace{-2mm}
\end{table}

Another concurrent work~\cite{mollenhoff2023sam} theoretically establishes \SAM ~as an optimal relaxation of Bayes object.
Consequently, this suggests that the transferability benefit of  \SAM ~may owe to the fact it optimizes a single model that represents an ensemble (expectation) of  models, \ie, attacking a \SAM ~solution is  somewhat equivalent to attacking an ensemble of solutions. 
We  conjecture that the strong input gradient alignment towards every  training solution may reflect this  ``expectation'' effect of \SAM ~(Bayes).
Notably, our interpretation here is merely intuitive and speculative, however, \SAM ~still remains poorly understood to date.
We believe that uncovering the exact reason is far beyond the scope of our study, while our experimental observations provide  valuable insights and may contribute to  a better understanding  of \SAM ~for future study. 



\section{Related Work}

In this work, we aim to provide a comprehensive overview of adversarial transferability, \ie, the transferabi-lity of adversarial examples, through the lens of surrogate model training. 
Here we connect our works to research that studies transferability beyond the adversarial perspective and approaches that improve adversarial transferability beyond surrogate training.


\subsection{Beyond  Adversarial Transferability}
In recent research, adversarial transferability, the property that depicts the machine learning similarities from the adversarial perspective, has been connected to another form of transferability, namely, \textit{domain transferability} (or knowledge transferability).
Domain transferability describes the ability of a pre-trained model to transfer knowledge from the source domain to the target domain, allowing it to perform well on various downstream tasks.

Typically, Liang \etal ~\cite{liang2021uncovering} have shed light on the bidirectional relationship between these two forms of transferability.
They presented theoretical and empirical results showing that \textit{\textbf{domain transferability implies adversarial transferability, and vice versa}}. This finding suggests that approaches aimed at enhancing domain transferability during model training may also have a positive impact on adversarial transferability.

Within the line of research \cite{yosinski2014transferable,wang2019characterizing,zhang2021quantifying,wang2022generalizing} investigating factors influencing domain transferability, several works~\cite{yi2021improved,salman2020adversarially,utrera2021adversariallytrained} have focused on the impact of adversarial training.
These works consistently showed that adversarially trained models transfer better than non-robust models on tasks in different target domains, \ie, \textit{\textbf{adversarial training improves domain transferability},} thus asserting a connection between adversarial robustness and domain transferability.

However, recently, Xu \etal ~\cite{xuICML2022Adversarially2022} 
challenged this connection and contended that robustness alone is neither necessary nor sufficient for domain transferability.
Through comprehensive experiments, they demonstrated that models with better robustness do not necessarily guarantee enhanced domain transferability. 
Instead, the regularization effect of adversarial training offers a fundamental explanation for the relationship between adversarial training and domain transferability, namely, \textit{\textbf{adversarial training improves domain transferability through its regularization effect}}.

Taken together, the conclusions of these researches suggest that adversarial training is expected to improve both two kinds of transferability with its regularization effect.
However, in this work,  though we find the regularization effect induced by adversarial training does play a crucial role in improving model smoothness, leading to enhanced adversarial transferability, the situation is much more complex: adversarially trained models with small budgets yield better adversarial transferability, large budgets perform worse~\cite{springer2021little,springer2021uncovering}, and yet, these same robust models does uniformly transfer  better in downstream tasks~\cite{salman2020adversarially}.

Our work dissects this nuanced relationship and establishes a connection between the drawbacks of adversarial training in adversarial transferability and the data distribution shift.
From the  distribution perspective, the key distinction between these two kinds of transferability is that domain transferability does not necessarily require the model to match the source domain perfectly since the target domain is independent of it, yet successful untargeted (targeted) adversarial attacks owe to pulling (pushing) the generated AEs out of (into) the original (target) sample distribution~\cite{zhu2022toward}, thus demanding the models fitting the data domain well.

\subsection{Beyond Surrogate Training} 
Intriguingly, besides the perspective of surrogate training, we notice that some concurrent works adopt similar design philosophies as ours from other perspectives to improve adversarial transferability. 
Wu \etal ~\cite{wu2023gnp} and Ge \etal ~\cite{ge2023boosting} both suggest a regularization term that encourages small loss gradient norms in the input space during AE generation, while the latter also performs gradient averaging on multiple randomly sampled neighbors. 
On the other hand, Chen \etal ~\cite{chen2023rethinking} study the model ensemble perspective and propose a common weakness attack, which adopts \SAM ~and a cosine similarity encourager to optimize  AEs, aiming to converge to points close to the flat local optimum of each model.
These works also demonstrate the effectiveness of optimizing smoothness and similarity for improving adversarial transferability.

Other approaches that study from the surrogate perspective transform the model at the test time (\ie, after training, when performing the attack).
Ghost network~\cite{li2020learning} attacks a set of ghost networks generated by densely applying dropout at intermediate features.  
SGM~\cite{Wu2020Skip} incorporates more gradients from the skip connections of ResNet to generate AEs, while LinBP~\cite{NEURIPS2020_00e26af6} and BPA~\cite{wang2023rethinking} both restore the truncated gradient caused by non-linear layers. 

Our work complements them all since we view adversarial transferability as an inherent property of the surrogate model itself and propose methods to construct surrogate models specialized for transfer attacks at the training time. 
We leave the exploration and discussion regarding a unified framework~\cite{wang2021a} that can reason adversarial transferability from all perspectives for future work.

\section{Conclusion}

In this paper, we have conducted a comprehensive analysis of the transferability of adversarial examples from the surrogate perspective.
To the best of our knowledge, this is the first in-depth study on obtaining better surrogates for transfer attacks. 
Specifically, we investigate the impact of different training mechanisms on two key factors: model smoothness and gradient similarity. 
By exploring these factors, we seek to gain insights into the dynamics of  transferability and its relationship with various training approaches.

Through the dissection of adversarial training, we propose and then verify two working hypotheses. 
The first one is that the trade-off between model smoothness and gradient similarity largely dictates adversarial transferability from the perspective of surrogate models.
The other one is that the data distribution shift will impair gradient similarity, thereby producing worse surrogates on average.
As a result, the practical guide for more effective transfer attacks is to handle these model smoothness and similarity well simultaneously, validated by the superiority of our proposed methods, \ie, the combination of input gradient regularization (\ER ~and \JR) and sharpness-aware minimization (\SAM).

On the other front,  our study also exemplifies the conflicting aspects of current research and further provides plausible explanations for some intriguing puzzles while revealing other unsolved issues in the study of  transferability.

\section*{Acknowledgements}
\vspace{-2mm}
We sincerely thank the anonymous reviewers and the shepherd for their valuable feedback. 
Special thanks to Yuhang Zhou for pointing out an error in the original proof of \cref{theorem1}, which we've fixed in this version.
Shengshan's work is supported in part by the National Natural Science Foundation of China (Grant No.U20A20177) and Hubei Province Key R\&D Technology Special Innovation Project under Grant No.2021BAA032. 
Minghui's work is supported in part by the National Natural Science Foundation of China (Grant No. 62202186).
Shengshan Hu is the corresponding author.


{
\small
\bibliographystyle{plain}
\bibliography{egbib}
}

\appendix \label{Appendix}

\subsection{Experimental Settings}
\vspace{-2mm}
\cref{tab:training-settings} reports the detailed training settings. 
Code and models are available at \url{https://github.com/CGCL-codes/TransferAttackSurrogates}.

\begin{table}[hb!]
\centering
\setlength{\tabcolsep}{0.05pt}
\caption{Hyperparameters for training on CIFAR-10 and ImageNette. \MU, \CM, \CO\; and \LS\;refer to the X-axis parameters of \cref{fig:GS-DA-cifar10,fig:MS-DA-cifar10,fig:MS-DA-imagenette,fig:GS-DA-imagenette}.  \AT\; refer to \;\cref{fig:adv-transfer,fig:adv-ms,fig:adv-gs}. \IR, \JR, \ER \;and \SAM\;refer to \cref{fig:MS-GR-cifar10,fig:MS-GR-imagenette}, respectively.
The bold items in parameter lists correspond to \cref{tab:cifar10-ASR,tab:imagenette-ASR}.}
\resizebox{0.48\textwidth}{!}
{
\begin{tabular}{l|cc}
\toprule
                    & \textbf{CIFAR-10}                        & \textbf{ImageNette}                    \\ \hline
Surrogate Architecture    & ResNet18                        & ResNet50                    \\ 
Input size          & $32\times32$                          & $224\times224$                     \\
Batch size           & 128                            & 128                         \\
Epoch               & 200                            & 50                          \\
Warmup epoch        & 10                             & 5                           \\
Finetune            & False                          & True                        \\
Peak learning rate             & 0.1                            & 1                           \\
Learning rate decay & \multicolumn{2}{c}{cosine}                                  \\
Optimizer           & \multicolumn{2}{c}{SGD}                                      \\ \hline
(\MU)\;$\mathbf{p}$     & \multicolumn{2}{c}{\{\textbf{0.1}, 0.3, 0.5, 0.7, \textbf{0.9}\}}              \\
(\CM)\;$\mathbf{p}$    & \multicolumn{2}{c}{\{\textbf{0.1}, 0.3, 0.5, 0.7, \textbf{0.9}\}}              \\
(\CO)\;$M$   & \{\textbf{8}, 12, 16, 20, \textbf{24}\}          & \{\textbf{80}, 100, 120, 140, \textbf{160}\}   \\
(\LS)\;p       & \multicolumn{2}{c}{\{\textbf{0.1}, 0.2, 0.3, 0.4, \textbf{0.5}\}}  \\ \hline
(\AT)\;Step            & 10                             & 5                           \\
(\AT)\;Step\_size            &  $0.25\times \epsilon$                               & $0.4\times\epsilon$                           \\
(\AT)\;$\epsilon$         & \{0.01, \textbf{0.03}, 0.05, 0.1, 0.2, 0.5, 1\}                           & \{0.05, 0.1, 0.2, 0.5, \textbf{1}, 3, 5\}                           \\
\hline
(\IR)\;$\lambda_{ir}$        & \{0.1, 0.3, \textbf{0.5}, 0.7, 1\}                            & \{0.1, 0.5, \textbf{1}, 1.5, 2\}                           \\
(\JR)\;$\lambda_{jr}$        & \{1e-5, \textbf{1e-4}, 0.01, 0.05, 0.1\}                         & \{0.001, 0.01, \textbf{0.1}, 0.5, 1\}                          \\
(\ER)\;$\lambda_{er}$        & \{1e-3, 5e-3, 0.01, \textbf{0.05}, 0.1\}                           & \{\textbf{0.05}, 0.3, 0.5, 0.7, 1\}                        \\
(\SAM)\;$\rho$            & \{0.01, 0.05, \textbf{0.1}, 0.15, 0.2\}                            & \{\textbf{0.2}, 0.5, 1, 2, 3\}                         \\ \bottomrule
\end{tabular}
}
\label{tab:training-settings}
\vspace{-2mm}
\end{table}

\begin{table}[ht!]
\centering
\setlength{\tabcolsep}{0.05pt}
\caption{Clean accuracy of models in \cref{tab:cifar10-ASR,tab:imagenette-ASR,tab:API-results}. We report the results of a fixed random seed here.}
\setlength{\tabcolsep}{0.8pt}
\resizebox{0.48\textwidth}{!}
{
\begin{tabular}{c|cc|cc}
\toprule
          & \multicolumn{2}{c|}{CIFAR10}              & \multicolumn{2}{c}{ImageNette}            \\ \hline
          & \begin{tabular}[c]{@{}c@{}}Clean Acc\\ Train/Test\end{tabular} & \begin{tabular}[c]{@{}c@{}}Generalization\\ Gap ($\downarrow$)\end{tabular} & \begin{tabular}[c]{@{}c@{}}Clean Acc\\ Train/Test\end{tabular} & \begin{tabular}[c]{@{}c@{}}Generalization\\ Gap ($\downarrow$)\end{tabular} \\ \hline
ST        & 100.00/94.40         & 5.60                & 97.75/96.74          & 1.01               \\ \hline
\MU, $\tau=1$\; & 100.00/95.04         & 4.96               & 95.53/97.10          & -1.57              \\
\MU, $\tau=5$\; & \textbf{100.00/96.13}         & \textbf{3.87}               & \textbf{84.23/97.58}          & \textbf{-13.35}             \\
\CM, $\tau=1$\; & 100.00/95.45         & 4.55               & 94.84/97.10          & -2.26              \\
\CM, $\tau=5$\; & \textbf{99.99/95.93}          & \textbf{4.06}               & \textbf{82.46/97.73}          & \textbf{-15.27}             \\
\CO, $\tau=1$\; & 100.00/95.09         & 4.91               & 97.40/96.94          & 0.46               \\
\CO, $\tau=5$\; & 100.00/95.79         & 4.21               & 95.41/96.97          & -1.56              \\
\LS, $\tau=1$\; & 100.00/94.84         & 5.16               & 98.09/97.48          & 0.61               \\
\LS, $\tau=5$\; &100.00/94.35         & 5.65               & 98.38/98.04          & 0.34               \\ \hline
\AT        & 100.00/94.30         & 5.70                & 96.15/95.99          & 0.16               \\ \hline
\IR        & 100.00/92.76         & 7.24               & 97.24/96.13          & 1.11               \\
\JR        & 100.00/94.46         & 5.54               & 97.51/96.28          & 1.23               \\
\ER        & 100.00/92.83         & 7.17               & 97.86/97.12          & 0.74               \\
\SAM       & 100.00/95.20         & 4.80                & 96.54/97.53          & -0.99              \\ \hline
\SAM\&\IR    & 99.87/92.72          & 7.15               & 95.49/96.51          & -1.02              \\
\SAM\&\JR    & 100.00/95.16         & 4.84               & 95.62/96.79          & -1.17              \\ \bottomrule
\end{tabular}
}
\label{clean-acc}
\end{table}

\begin{table*}[tbp!]
\renewcommand\arraystretch{0.80}
\setlength{\abovecaptionskip}{0cm}
\setlength{\belowcaptionskip}{-0cm}
\centering
\caption{Transfer ASRs of AEs crafted against different surrogates on \textbf{CIFAR-10} and \textbf{ImageNette} using \textbf{AutoAttack}. We plot this table in the same way as \cref{tab:cifar10-ASR}. In \%.}
\tabcolsep=0.05cm
  \resizebox{\textwidth}{!}
  {
\begin{tabular}{c|cccc|cccc|cccc}
\toprule    \multicolumn{13}{c}{\textbf{\scriptsize CIFAR-10}} \\
\hline
                   & \multicolumn{4}{c|}{\scriptsize 4/255}                                                                                                                 & \multicolumn{4}{c|}{\scriptsize 8/255}                                                                                                                    & \multicolumn{4}{c}{\scriptsize 16/255}                                                                                                                  \\
\multirow{-2}{*}{} & \scriptsize ResNet50                         & \scriptsize VGG16                            & \scriptsize InceptionV3                      & \scriptsize DenseNet121                       & \scriptsize ResNet50                          & \scriptsize VGG16                             & \scriptsize InceptionV3                       & \scriptsize DenseNet121                       & \scriptsize ResNet50                          & \scriptsize VGG16                             & \scriptsize InceptionV3                      & \scriptsize DenseNet121                      \\ \hline
\scriptsize \ST                 & \scriptsize 41.2$_{\pm 5.6}$                        & \scriptsize 30.1$_{\pm3.5}$                         & \scriptsize 41.7$_{\pm5.3}$                         & \scriptsize 58.6$_{\pm7.1}$                          & \scriptsize 62.9$_{\pm7.9}$                          & \scriptsize 54.5$_{\pm6.5}$                          & \scriptsize 64.8$_{\pm6.0}$                          & \scriptsize 80.5$_{\pm6.8}$                          & \scriptsize 83.2$_{\pm5.2}$                          & \scriptsize 81.0$_{\pm5.5}$                          & \scriptsize 84.2$_{\pm3.8}$                         & \scriptsize 91.0$_{\pm3.3}$                         \\ \hline
\scriptsize \MU, $\tau=1$         & \scriptsize \cellcolor[HTML]{F2F2F2}27.9$_{\pm5.1}$& \scriptsize \cellcolor[HTML]{F2F2F2}20.9$_{\pm2.5}$& \scriptsize \cellcolor[HTML]{F2F2F2}29.0$_{\pm4.8}$& \scriptsize \cellcolor[HTML]{F2F2F2}40.2$_{\pm7.0}$ & \scriptsize \cellcolor[HTML]{F2F2F2}48.5$_{\pm6.9}$ & \scriptsize \cellcolor[HTML]{F2F2F2}40.6$_{\pm3.7}$ & \scriptsize \cellcolor[HTML]{F2F2F2}51.4$_{\pm5.9}$ & \scriptsize \cellcolor[HTML]{F2F2F2}65.3$_{\pm7.5}$ & \scriptsize \cellcolor[HTML]{F2F2F2}73.1$_{\pm6.4}$ & \scriptsize \cellcolor[HTML]{F2F2F2}71.3$_{\pm4.5}$ & \scriptsize \cellcolor[HTML]{F2F2F2}76.5$_{\pm4.7}$& \scriptsize \cellcolor[HTML]{F2F2F2}85.2$_{\pm5.1}$ \\
\scriptsize \MU,  $\tau=3$        & \scriptsize \cellcolor[HTML]{F2F2F2}20.4$_{\pm0.6}$& \scriptsize \cellcolor[HTML]{F2F2F2}16.4$_{\pm0.5}$& \scriptsize \cellcolor[HTML]{F2F2F2}23.3$_{\pm1.2}$& \scriptsize \cellcolor[HTML]{F2F2F2}27.8$_{\pm1.6}$ & \scriptsize \cellcolor[HTML]{F2F2F2}37.7$_{\pm0.9}$ & \scriptsize \cellcolor[HTML]{F2F2F2}31.2$_{\pm1.1}$ & \scriptsize \cellcolor[HTML]{F2F2F2}42.0$_{\pm1.4}$ & \scriptsize \cellcolor[HTML]{F2F2F2}51.1$_{\pm1.3}$ & \scriptsize \cellcolor[HTML]{F2F2F2}68.0$_{\pm1.9}$ & \scriptsize \cellcolor[HTML]{F2F2F2}64.7$_{\pm1.7}$ & \scriptsize \cellcolor[HTML]{F2F2F2}72.9$_{\pm1.2}$& \scriptsize \cellcolor[HTML]{F2F2F2}80.4$_{\pm1.2}$ \\
\scriptsize \MU,  $\tau=5$         & \scriptsize \cellcolor[HTML]{F2F2F2}18.8$_{\pm0.3}$& \scriptsize \cellcolor[HTML]{F2F2F2}15.7$_{\pm0.3}$& \scriptsize \cellcolor[HTML]{F2F2F2}20.8$_{\pm1.0}$& \scriptsize \cellcolor[HTML]{F2F2F2}24.5$_{\pm0.8}$ & \scriptsize \cellcolor[HTML]{F2F2F2}33.8$_{\pm1.7}$ & \scriptsize \cellcolor[HTML]{F2F2F2}28.6$_{\pm0.9}$ & \scriptsize \cellcolor[HTML]{F2F2F2}37.5$_{\pm1.0}$ & \scriptsize \cellcolor[HTML]{F2F2F2}44.9$_{\pm1.6}$ & \scriptsize \cellcolor[HTML]{F2F2F2}64.1$_{\pm1.5}$ & \scriptsize \cellcolor[HTML]{F2F2F2}60.5$_{\pm1.9}$ & \scriptsize \cellcolor[HTML]{F2F2F2}68.3$_{\pm0.4}$& \scriptsize \cellcolor[HTML]{F2F2F2}76.4$_{\pm1.5}$ \\
\scriptsize \CM, $\tau=1$         & \scriptsize \cellcolor[HTML]{F2F2F2}22.4$_{\pm0.9}$& \scriptsize \cellcolor[HTML]{F2F2F2}16.7$_{\pm1.3}$& \scriptsize \cellcolor[HTML]{F2F2F2}21.8$_{\pm0.6}$& \scriptsize \cellcolor[HTML]{F2F2F2}30.3$_{\pm1.3}$ & \scriptsize \cellcolor[HTML]{F2F2F2}39.6$_{\pm1.7}$ & \scriptsize \cellcolor[HTML]{F2F2F2}32.42$_{\pm.5}$ & \scriptsize \cellcolor[HTML]{F2F2F2}39.0$_{\pm1.0}$ & \scriptsize \cellcolor[HTML]{F2F2F2}53.4$_{\pm2.5}$ & \scriptsize \cellcolor[HTML]{F2F2F2}64.5$_{\pm2.0}$ & \scriptsize \cellcolor[HTML]{F2F2F2}62.4$_{\pm3.2}$ & \scriptsize \cellcolor[HTML]{F2F2F2}64.9$_{\pm1.7}$& \scriptsize \cellcolor[HTML]{F2F2F2}77.8$_{\pm3.2}$ \\
\scriptsize \CM, $\tau=3$         & \scriptsize \cellcolor[HTML]{F2F2F2}14.3$_{\pm1.2}$& \scriptsize \cellcolor[HTML]{F2F2F2}11.5$_{\pm0.4}$& \scriptsize \cellcolor[HTML]{F2F2F2}14.2$_{\pm1.1}$& \scriptsize \cellcolor[HTML]{F2F2F2}17.2$_{\pm2.0}$ & \scriptsize \cellcolor[HTML]{F2F2F2}25.8$_{\pm3.0}$ & \scriptsize \cellcolor[HTML]{F2F2F2}20.2$_{\pm1.8}$ & \scriptsize \cellcolor[HTML]{F2F2F2}25.1$_{\pm2.6}$ & \scriptsize \cellcolor[HTML]{F2F2F2}32.1$_{\pm4.4}$ & \scriptsize \cellcolor[HTML]{F2F2F2}48.9$_{\pm4.1}$ & \scriptsize \cellcolor[HTML]{F2F2F2}47.1$_{\pm2.7}$ & \scriptsize \cellcolor[HTML]{F2F2F2}49.8$_{\pm3.0}$& \scriptsize \cellcolor[HTML]{F2F2F2}60.6$_{\pm4.6}$ \\
\scriptsize \CM, $\tau=5$         & \scriptsize \cellcolor[HTML]{F2F2F2}12.3$_{\pm0.5}$& \scriptsize \cellcolor[HTML]{F2F2F2}10.4$_{\pm0.2}$& \scriptsize \cellcolor[HTML]{F2F2F2}12.0$_{\pm1.1}$& \scriptsize \cellcolor[HTML]{F2F2F2}13.7$_{\pm1.2}$ & \scriptsize \cellcolor[HTML]{F2F2F2}21.9$_{\pm1.3}$& \scriptsize \cellcolor[HTML]{F2F2F2}17.5$_{\pm0.6}$ & \scriptsize \cellcolor[HTML]{F2F2F2}21.3$_{\pm0.9}$ & \scriptsize \cellcolor[HTML]{F2F2F2}26.5$_{\pm1.8}$ & \scriptsize \cellcolor[HTML]{F2F2F2}44.4$_{\pm1.6}$ & \scriptsize \cellcolor[HTML]{F2F2F2}41.5$_{\pm1.6}$& \scriptsize \cellcolor[HTML]{F2F2F2}44.3$_{\pm1.9}$& \scriptsize \cellcolor[HTML]{F2F2F2}53.4$_{\pm1.8}$ \\
\scriptsize \CO, $\tau=1$         & \scriptsize \cellcolor[HTML]{F2F2F2}40.7$_{\pm7.0}$& \scriptsize \cellcolor[HTML]{F2F2F2}31.1$_{\pm7.0}$& \scriptsize \cellcolor[HTML]{F2F2F2}39.8$_{\pm5.9}$& \scriptsize \cellcolor[HTML]{F2F2F2}55.9$_{\pm8.1}$ & \scriptsize \cellcolor[HTML]{F2F2F2}62.4$_{\pm6.9}$ & \scriptsize \cellcolor[HTML]{F2F2F2}55.3$_{\pm8.1}$ & \scriptsize \cellcolor[HTML]{F2F2F2}63.1$_{\pm6.2}$ & \scriptsize \cellcolor[HTML]{F2F2F2}78.2$_{\pm5.9}$ & \scriptsize \cellcolor[HTML]{F2F2F2}82.5$_{\pm5.6}$ & \scriptsize \cellcolor[HTML]{F2F2F2}79.9$_{\pm6.3}$ & \scriptsize \cellcolor[HTML]{F2F2F2}82.8$_{\pm5.3}$& \scriptsize \cellcolor[HTML]{F2F2F2}89.9$_{\pm3.7}$ \\
\scriptsize \CO, $\tau=3$         & \scriptsize \cellcolor[HTML]{F2F2F2}34.6$_{\pm5.2}$& \scriptsize \cellcolor[HTML]{F2F2F2}24.9$_{\pm2.8}$& \scriptsize \cellcolor[HTML]{F2F2F2}33.7$_{\pm5.1}$& \scriptsize \cellcolor[HTML]{F2F2F2}47.1$_{\pm8.0}$ & \scriptsize \cellcolor[HTML]{F2F2F2}55.5$_{\pm5.8}$ & \scriptsize \cellcolor[HTML]{F2F2F2}46.0$_{\pm4.7}$ & \scriptsize \cellcolor[HTML]{F2F2F2}55.8$_{\pm6.7}$ & \scriptsize \cellcolor[HTML]{F2F2F2}70.7$_{\pm8.0}$ & \scriptsize \cellcolor[HTML]{F2F2F2}79.6$_{\pm5.4}$ & \scriptsize \cellcolor[HTML]{F2F2F2}74.4$_{\pm5.1}$ & \scriptsize \cellcolor[HTML]{F2F2F2}79.2$_{\pm5.6}$& \scriptsize \cellcolor[HTML]{F2F2F2}86.8$_{\pm5.0}$ \\
\scriptsize \CO, $\tau=5$         & \scriptsize \cellcolor[HTML]{F2F2F2}30.8$_{\pm2.5}$& \scriptsize \cellcolor[HTML]{F2F2F2}22.3$_{\pm2.0}$& \scriptsize \cellcolor[HTML]{F2F2F2}31.0$_{\pm3.3}$& \scriptsize \cellcolor[HTML]{F2F2F2}42.1$_{\pm6.7}$ & \scriptsize \cellcolor[HTML]{F2F2F2}49.5$_{\pm5.0}$ & \scriptsize \cellcolor[HTML]{F2F2F2}41.3$_{\pm4.8}$ & \scriptsize \cellcolor[HTML]{F2F2F2}51.3$_{\pm5.8}$ & \scriptsize \cellcolor[HTML]{F2F2F2}63.9$_{\pm7.8}$ & \scriptsize \cellcolor[HTML]{F2F2F2}73.2$_{\pm5.6}$ & \scriptsize \cellcolor[HTML]{F2F2F2}69.3$_{\pm6.3}$ & \scriptsize \cellcolor[HTML]{F2F2F2}74.5$_{\pm6.4}$& \scriptsize \cellcolor[HTML]{F2F2F2}81.7$_{\pm7.0}$ \\
\scriptsize \LS, $\tau=1$         & \scriptsize \cellcolor[HTML]{F2F2F2}35.5$_{\pm4.8}$& \scriptsize \cellcolor[HTML]{F2F2F2}28.2$_{\pm5.7}$& \scriptsize \cellcolor[HTML]{F2F2F2}35.5$_{\pm3.3}$& \scriptsize \cellcolor[HTML]{F2F2F2}48.4$_{\pm5.7}$ & \scriptsize \cellcolor[HTML]{F2F2F2}54.3$_{\pm9.4}$ & \scriptsize \cellcolor[HTML]{F2F2F2}49.7$_{\pm10.7}$& \scriptsize \cellcolor[HTML]{F2F2F2}56.0$_{\pm6.4}$ & \scriptsize \cellcolor[HTML]{F2F2F2}69.4$_{\pm8.4}$ & \scriptsize \cellcolor[HTML]{F2F2F2}76.1$_{\pm7.8}$ & \scriptsize \cellcolor[HTML]{F2F2F2}76.9$_{\pm8.3}$ & \scriptsize \cellcolor[HTML]{F2F2F2}78.3$_{\pm6.0}$& \scriptsize \cellcolor[HTML]{F2F2F2}85.9$_{\pm5.1}$ \\
\scriptsize \LS, $\tau=3$         & \scriptsize \cellcolor[HTML]{F2F2F2}34.2$_{\pm7.5}$& \scriptsize \cellcolor[HTML]{F2F2F2}27.3$_{\pm4.6}$& \scriptsize \cellcolor[HTML]{F2F2F2}33.9$_{\pm5.6}$& \scriptsize \cellcolor[HTML]{F2F2F2}44.1$_{\pm2.8}$ & \scriptsize \cellcolor[HTML]{F2F2F2}55.0$_{\pm13.6}$& \scriptsize \cellcolor[HTML]{F2F2F2}50.0$_{\pm10.0}$& \scriptsize \cellcolor[HTML]{F2F2F2}55.9$_{\pm11.0}$& \scriptsize \cellcolor[HTML]{F2F2F2}68.2$_{\pm5.9}$ & \scriptsize \cellcolor[HTML]{F2F2F2}76.8$_{\pm12.1}$& \scriptsize \cellcolor[HTML]{F2F2F2}76.6$_{\pm10.0}$& \scriptsize \cellcolor[HTML]{F2F2F2}78.3$_{\pm9.3}$& \scriptsize \cellcolor[HTML]{F2F2F2}85.5$_{\pm5.3}$ \\
\scriptsize \LS, $\tau=5$         & \scriptsize \cellcolor[HTML]{F2F2F2}31.3$_{\pm3.5}$& \scriptsize \cellcolor[HTML]{F2F2F2}26.6$_{\pm3.5}$& \scriptsize \cellcolor[HTML]{F2F2F2}30.2$_{\pm3.1}$& \scriptsize \cellcolor[HTML]{F2F2F2}41.2$_{\pm2.9}$ & \scriptsize \cellcolor[HTML]{F2F2F2}51.1$_{\pm8.2}$ & \scriptsize \cellcolor[HTML]{F2F2F2}48.6$_{\pm8.0}$ & \scriptsize \cellcolor[HTML]{F2F2F2}51.2$_{\pm6.9}$ & \scriptsize \cellcolor[HTML]{F2F2F2}65.6$_{\pm7.0}$ & \scriptsize \cellcolor[HTML]{F2F2F2}72.2$_{\pm7.7}$ & \scriptsize \cellcolor[HTML]{F2F2F2}74.5$_{\pm8.2}$ & \scriptsize \cellcolor[HTML]{F2F2F2}73.5$_{\pm6.2}$& \scriptsize \cellcolor[HTML]{F2F2F2}83.3$_{\pm5.4}$ \\ \hline
\scriptsize \AT                & \scriptsize \cellcolor[HTML]{FCE4D6}54.4$_{\pm2.3}$& \scriptsize \cellcolor[HTML]{FCE4D6}45.3$_{\pm1.5}$& \scriptsize \cellcolor[HTML]{FCE4D6}49.5$_{\pm3.6}$& \scriptsize \cellcolor[HTML]{FCE4D6}61.0$_{\pm3.5}$ & \scriptsize \cellcolor[HTML]{FCE4D6}80.2$_{\pm5.8}$ & \scriptsize \cellcolor[HTML]{FCE4D6}73.7$_{\pm5.9}$ & \scriptsize \cellcolor[HTML]{FCE4D6}77.5$_{\pm4.1}$ & \scriptsize \cellcolor[HTML]{FCE4D6}86.4$_{\pm4.2}$ & \scriptsize \cellcolor[HTML]{FCE4D6}90.8$_{\pm3.0}$ & \scriptsize \cellcolor[HTML]{FCE4D6}89.5$_{\pm3.2}$ & \scriptsize \cellcolor[HTML]{FCE4D6}89.8$_{\pm2.8}$& \scriptsize \cellcolor[HTML]{FCE4D6}92.9$_{\pm2.3}$ \\ \hline
\scriptsize \IR                & \scriptsize \cellcolor[HTML]{FCE4D6}51.9$_{\pm1.4}$& \scriptsize \cellcolor[HTML]{FCE4D6}45.3$_{\pm1.1}$& \scriptsize \cellcolor[HTML]{FCE4D6}48.0$_{\pm0.6}$& \scriptsize \cellcolor[HTML]{F2F2F2}54.4$_{\pm1.1}$ & \scriptsize \cellcolor[HTML]{FCE4D6}85.9$_{\pm3.4}$ & \scriptsize \cellcolor[HTML]{FCE4D6}82.4$_{\pm3.9}$ & \scriptsize \cellcolor[HTML]{FCE4D6}83.8$_{\pm3.3}$ & \scriptsize \cellcolor[HTML]{FCE4D6}87.4$_{\pm3.2}$ & \scriptsize \cellcolor[HTML]{FCE4D6}92.4$_{\pm2.1}$ & \scriptsize \cellcolor[HTML]{FCE4D6}92.2$_{\pm2.3}$ & \scriptsize \cellcolor[HTML]{FCE4D6}91.8$_{\pm2.1}$& \scriptsize \cellcolor[HTML]{FCE4D6}92.0$_{\pm2.0}$ \\
\scriptsize \JR                & \scriptsize \cellcolor[HTML]{FCE4D6}63.5$_{\pm8.7}$& \scriptsize \cellcolor[HTML]{FCE4D6}53.1$_{\pm7.1}$& \scriptsize \cellcolor[HTML]{FCE4D6}62.4$_{\pm5.7}$& \scriptsize \cellcolor[HTML]{FCE4D6}75.1$_{\pm4.5}$ & \scriptsize \cellcolor[HTML]{FCE4D6}78.9$_{\pm6.8}$ & \scriptsize \cellcolor[HTML]{FCE4D6}73.7$_{\pm6.1}$ & \scriptsize \cellcolor[HTML]{FCE4D6}79.5$_{\pm4.1}$ & \scriptsize \cellcolor[HTML]{FCE4D6}88.7$_{\pm2.6}$ & \scriptsize \cellcolor[HTML]{FCE4D6}87.6$_{\pm4.3}$& \scriptsize \cellcolor[HTML]{FCE4D6}87.2$_{\pm3.9}$& \scriptsize \cellcolor[HTML]{FCE4D6}88.1$_{\pm2.4}$& \scriptsize \cellcolor[HTML]{FCE4D6}92.1$_{\pm2.1}$ \\
\scriptsize \ER                & \scriptsize \cellcolor[HTML]{FCE4D6}55.0$_{\pm5.5}$& \scriptsize \cellcolor[HTML]{FCE4D6}42.9$_{\pm5.6}$& \scriptsize \cellcolor[HTML]{FCE4D6}46.2$_{\pm7.4}$& \scriptsize \cellcolor[HTML]{F2F2F2}55.2$_{\pm13.8}$& \scriptsize \cellcolor[HTML]{FCE4D6}82.3$_{\pm5.5}$ & \scriptsize \cellcolor[HTML]{FCE4D6}74.1$_{\pm7.9}$ & \scriptsize \cellcolor[HTML]{FCE4D6}76.1$_{\pm8.8}$ & \scriptsize \cellcolor[HTML]{FCE4D6}81.1$_{\pm12.2}$& \scriptsize \cellcolor[HTML]{FCE4D6}90.6$_{\pm4.0}$ & \scriptsize \cellcolor[HTML]{FCE4D6}89.3$_{\pm4.6}$ & \scriptsize \cellcolor[HTML]{FCE4D6}88.9$_{\pm4.0}$& \scriptsize \cellcolor[HTML]{FCE4D6}89.8$_{\pm4.4}$ \\
\scriptsize \SAM               & \scriptsize \cellcolor[HTML]{F8CBAD}66.1$_{\pm8.6}$& \scriptsize \cellcolor[HTML]{F8CBAD}53.4$_{\pm5.8}$& \scriptsize \cellcolor[HTML]{F8CBAD}66.0$_{\pm8.8}$& \scriptsize \cellcolor[HTML]{F8CBAD}81.1$_{\pm6.9}$ & \scriptsize \cellcolor[HTML]{F8CBAD}88.4$_{\pm4.7}$ & \scriptsize \cellcolor[HTML]{FCE4D6}83.4$_{\pm4.5}$& \scriptsize \cellcolor[HTML]{F8CBAD}88.7$_{\pm4.8}$ & \scriptsize \cellcolor[HTML]{F8CBAD}94.1$_{\pm2.4}$ & \scriptsize \cellcolor[HTML]{F8CBAD}94.3$_{\pm3.0}$ & \scriptsize \cellcolor[HTML]{F8CBAD}94.0$_{\pm3.0}$ & \scriptsize \cellcolor[HTML]{F8CBAD}94.2$_{\pm2.9}$& \scriptsize \cellcolor[HTML]{F8CBAD}94.8$_{\pm2.4}$ \\ \hline
\scriptsize \SAM\&\IR           & \scriptsize \cellcolor[HTML]{FCE4D6}52.9$_{\pm0.6}$& \scriptsize \cellcolor[HTML]{FCE4D6}47.0$_{\pm1.0}$& \scriptsize \cellcolor[HTML]{FCE4D6}48.2$_{\pm1.2}$& \scriptsize \cellcolor[HTML]{F2F2F2}54.1$_{\pm1.7}$ & \scriptsize \cellcolor[HTML]{FCE4D6}88.3$_{\pm2.3}$ & \scriptsize \cellcolor[HTML]{F8CBAD}85.5$_{\pm2.7}$ & \scriptsize \cellcolor[HTML]{FCE4D6}85.7$_{\pm2.4}$ & \scriptsize \cellcolor[HTML]{FCE4D6}88.7$_{\pm2.0}$ & \scriptsize \cellcolor[HTML]{FCE4D6}92.8$_{\pm2.0}$ & \scriptsize \cellcolor[HTML]{FCE4D6}92.9$_{\pm2.0}$ & \scriptsize \cellcolor[HTML]{FCE4D6}92.2$_{\pm2.0}$& \scriptsize \cellcolor[HTML]{FCE4D6}92.4$_{\pm1.9}$ \\
\scriptsize \SAM\&\JR           & \scriptsize \cellcolor[HTML]{F4B084}69.8$_{\pm3.0}$& \scriptsize \cellcolor[HTML]{F4B084}57.2$_{\pm2.9}$& \scriptsize \cellcolor[HTML]{F4B084}69.2$_{\pm2.5}$& \scriptsize \cellcolor[HTML]{F4B084}82.7$_{\pm2.1}$ & \scriptsize \cellcolor[HTML]{F4B084}90.4$_{\pm3.9}$ & \scriptsize \cellcolor[HTML]{F4B084}86.2$_{\pm4.7}$ & \scriptsize \cellcolor[HTML]{F4B084}90.3$_{\pm3.5}$ & \scriptsize \cellcolor[HTML]{F4B084}94.6$_{\pm1.8}$ & \scriptsize \cellcolor[HTML]{F4B084}94.8$_{\pm1.9}$ & \scriptsize \cellcolor[HTML]{F4B084}94.6$_{\pm2.1}$ & \scriptsize \cellcolor[HTML]{F4B084}94.6$_{\pm1.7}$& \scriptsize \cellcolor[HTML]{F4B084}94.9$_{\pm1.6}$ \\  

\toprule    \multicolumn{13}{c}{\textbf{\scriptsize ImageNette}} \\
\hline
                   & \multicolumn{4}{c|}{\scriptsize 4/255}                                                                                                                & \multicolumn{4}{c|}{\scriptsize 8/255}                                                                                                                   & \multicolumn{4}{c}{\scriptsize 16/255}                                                                                                                   \\
\multirow{-2}{*}{} & \scriptsize VGG16                            & \scriptsize DenseNet121                      & \scriptsize MobileNetV2                      & \scriptsize Xception                         & \scriptsize VGG16                             & \scriptsize DenseNet121                      & \scriptsize MobileNetV2                       & \scriptsize Xception                          & \scriptsize VGG16                             & \scriptsize DenseNet121                      & \scriptsize MobileNetV2                       & \scriptsize Xception                          \\ \hline
\scriptsize \ST                & \scriptsize 10.1$_{\pm0.7}$                        & \scriptsize 16.0$_{\pm1.0}$                        & \scriptsize 7.1$_{\pm0.5}$                         & \scriptsize 6.7$_{\pm0.1}$                         & \scriptsize 27.4$_{\pm3.0}$                         & \scriptsize 41.0$_{\pm2.8}$                        & \scriptsize 17.7$_{\pm0.7}$                         & \scriptsize 16.9$_{\pm1.1}$                         & \scriptsize 61.5$_{\pm6.5}$                         & \scriptsize 83.3$_{\pm4.9}$                        & \scriptsize 51.5$_{\pm2.7}$                         & \scriptsize 44.8$_{\pm4.6}$                          \\ \hline
\scriptsize \MU, $\tau=1$         & \scriptsize \cellcolor[HTML]{F2F2F2}7.9$_{\pm1.5}$ & \scriptsize \cellcolor[HTML]{F2F2F2}11.1$_{\pm1.1}$& \scriptsize \cellcolor[HTML]{F2F2F2}5.9$_{\pm0.5}$ & \scriptsize \cellcolor[HTML]{F2F2F2}6.1$_{\pm0.5}$ & \scriptsize \cellcolor[HTML]{F2F2F2}20.3$_{\pm1.5}$ & \scriptsize \cellcolor[HTML]{F2F2F2}23.6$_{\pm2.2}$& \scriptsize \cellcolor[HTML]{F2F2F2}13.5$_{\pm1.3}$ & \scriptsize \cellcolor[HTML]{F2F2F2}11.1$_{\pm0.3}$ & \scriptsize \cellcolor[HTML]{F2F2F2}43.8$_{\pm0.8}$ & \scriptsize \cellcolor[HTML]{F2F2F2}60.5$_{\pm4.5}$& \scriptsize \cellcolor[HTML]{F2F2F2}39.1$_{\pm2.3}$ & \scriptsize \cellcolor[HTML]{F2F2F2}29.5$_{\pm1.9}$  \\
\scriptsize \MU, $\tau=3$         & \scriptsize \cellcolor[HTML]{F2F2F2}6.8$_{\pm0.4}$ & \scriptsize \cellcolor[HTML]{F2F2F2}7.8$_{\pm0.6}$ & \scriptsize \cellcolor[HTML]{F2F2F2}5.1$_{\pm0.5}$ & \scriptsize \cellcolor[HTML]{F2F2F2}5.5$_{\pm0.5}$ & \scriptsize \cellcolor[HTML]{F2F2F2}13.3$_{\pm0.9}$ & \scriptsize \cellcolor[HTML]{F2F2F2}14.3$_{\pm0.9}$& \scriptsize \cellcolor[HTML]{F2F2F2}8.7$_{\pm0.9}$ & \scriptsize \cellcolor[HTML]{F2F2F2}9.0$_{\pm0.8}$  & \scriptsize \cellcolor[HTML]{F2F2F2}30.8$_{\pm2.8}$ & \scriptsize \cellcolor[HTML]{F2F2F2}35.6$_{\pm1.6}$& \scriptsize \cellcolor[HTML]{F2F2F2}26.4$_{\pm2.2}$ & \scriptsize \cellcolor[HTML]{F2F2F2}19.3$_{\pm1.7}$  \\
\scriptsize \MU, $\tau=5$         & \scriptsize \cellcolor[HTML]{F2F2F2}5.9$_{\pm0.5}$ & \scriptsize \cellcolor[HTML]{F2F2F2}7.0$_{\pm0.2}$ & \scriptsize \cellcolor[HTML]{F2F2F2}4.2$_{\pm0.4}$ & \scriptsize \cellcolor[HTML]{F2F2F2}5.3$_{\pm0.3}$ & \scriptsize \cellcolor[HTML]{F2F2F2}10.4$_{\pm1.0}$ & \scriptsize \cellcolor[HTML]{F2F2F2}11.6$_{\pm0.6}$& \scriptsize \cellcolor[HTML]{F2F2F2}8.2$_{\pm0.6}$  & \scriptsize \cellcolor[HTML]{F2F2F2}7.5$_{\pm0.5}$  & \scriptsize \cellcolor[HTML]{F2F2F2}24.9$_{\pm0.9}$ & \scriptsize \cellcolor[HTML]{F2F2F2}27.3$_{\pm1.1}$& \scriptsize \cellcolor[HTML]{F2F2F2}20.7$_{\pm1.3}$ & \scriptsize \cellcolor[HTML]{F2F2F2}17.2$_{\pm0.8}$  \\
\scriptsize \CM, $\tau=1$         & \scriptsize \cellcolor[HTML]{F2F2F2}7.3$_{\pm0.7}$ & \scriptsize \cellcolor[HTML]{F2F2F2}9.1$_{\pm1.1}$ & \scriptsize \cellcolor[HTML]{F2F2F2}5.3$_{\pm0.3}$ & \scriptsize \cellcolor[HTML]{F2F2F2}5.7$_{\pm0.3}$ & \scriptsize \cellcolor[HTML]{F2F2F2}15.3$_{\pm0.7}$ & \scriptsize \cellcolor[HTML]{F2F2F2}17.9$_{\pm1.5}$& \scriptsize \cellcolor[HTML]{F2F2F2}9.3$_{\pm0.7}$  & \scriptsize \cellcolor[HTML]{F2F2F2}8.6$_{\pm0.4}$  & \scriptsize \cellcolor[HTML]{F2F2F2}33.1$_{\pm1.7}$ & \scriptsize \cellcolor[HTML]{F2F2F2}44.0$_{\pm3.6}$& \scriptsize \cellcolor[HTML]{F2F2F2}25.6$_{\pm2.0}$ & \scriptsize \cellcolor[HTML]{F2F2F2}18.7$_{\pm0.9}$  \\
\scriptsize \CM, $\tau=3$         & \scriptsize \cellcolor[HTML]{F2F2F2}5.8$_{\pm0.4}$ & \scriptsize \cellcolor[HTML]{F2F2F2}6.4$_{\pm0.6}$ & \scriptsize \cellcolor[HTML]{F2F2F2}4.2$_{\pm0.6}$ & \scriptsize \cellcolor[HTML]{F2F2F2}5.3$_{\pm0.3}$ & \scriptsize \cellcolor[HTML]{F2F2F2}9.3$_{\pm0.9}$  & \scriptsize \cellcolor[HTML]{F2F2F2}11.4$_{\pm1.2}$& \scriptsize \cellcolor[HTML]{F2F2F2}7.1$_{\pm1.1}$  & \scriptsize \cellcolor[HTML]{F2F2F2}6.9$_{\pm0.5}$  & \scriptsize \cellcolor[HTML]{F2F2F2}20.2$_{\pm0.8}$ & \scriptsize \cellcolor[HTML]{F2F2F2}24.7$_{\pm2.3}$& \scriptsize \cellcolor[HTML]{F2F2F2}17.2$_{\pm0.6}$ & \scriptsize \cellcolor[HTML]{F2F2F2}12.7$_{\pm1.7}$  \\
\scriptsize \CM, $\tau=5$         & \scriptsize \cellcolor[HTML]{F2F2F2}4.5$_{\pm0.3}$ & \scriptsize \cellcolor[HTML]{F2F2F2}5.7$_{\pm0.5}$ & \scriptsize \cellcolor[HTML]{F2F2F2}3.8$_{\pm0.4}$ & \scriptsize \cellcolor[HTML]{F2F2F2}4.7$_{\pm0.3}$ & \scriptsize \cellcolor[HTML]{F2F2F2}7.9$_{\pm0.5}$  & \scriptsize \cellcolor[HTML]{F2F2F2}8.5$_{\pm0.7}$ & \scriptsize \cellcolor[HTML]{F2F2F2}5.7$_{\pm0.3}$  & \scriptsize \cellcolor[HTML]{F2F2F2}6.3$_{\pm0.7}$  & \scriptsize \cellcolor[HTML]{F2F2F2}16.6$_{\pm1.0}$ & \scriptsize \cellcolor[HTML]{F2F2F2}16.9$_{\pm0.9}$& \scriptsize \cellcolor[HTML]{F2F2F2}14.7$_{\pm1.3}$ & \scriptsize \cellcolor[HTML]{F2F2F2}9.9$_{\pm0.7}$   \\
\scriptsize \CO, $\tau=1$         & \scriptsize \cellcolor[HTML]{FCE4D6}10.9$_{\pm1.1}$& \scriptsize \cellcolor[HTML]{F2F2F2}15.3$_{\pm0.7}$& \scriptsize \cellcolor[HTML]{F2F2F2}6.9$_{\pm0.7}$ & \scriptsize \cellcolor[HTML]{FCE4D6}7.2$_{\pm0.0}$& \scriptsize \cellcolor[HTML]{F2F2F2}26.7$_{\pm2.5}$ & \scriptsize \cellcolor[HTML]{FCE4D6}41.3$_{\pm3.5}$& \scriptsize 17.9$_{\pm0.7}$                         & \scriptsize \cellcolor[HTML]{F2F2F2}16.3$_{\pm1.1}$ & \scriptsize \cellcolor[HTML]{F2F2F2}60.2$_{\pm4.6}$ & \scriptsize \cellcolor[HTML]{FCE4D6}83.9$_{\pm4.1}$& \scriptsize \cellcolor[HTML]{F2F2F2}50.6$_{\pm3.4}$ & \scriptsize \cellcolor[HTML]{F2F2F2}43.1$_{\pm3.1}$  \\
\scriptsize \CO, $\tau=3$        & \scriptsize 10.3$_{\pm1.5}$                        & \scriptsize \cellcolor[HTML]{F2F2F2}14.6$_{\pm0.2}$& \scriptsize 7.2$_{\pm0.4}$                         & \scriptsize \cellcolor[HTML]{FCE4D6}7.6$_{\pm0.4}$ & \scriptsize \cellcolor[HTML]{F2F2F2}26.3$_{\pm1.9}$ & \scriptsize \cellcolor[HTML]{F2F2F2}40.3$_{\pm2.3}$& \scriptsize \cellcolor[HTML]{F2F2F2}18.9$_{\pm0.9}$ & \scriptsize \cellcolor[HTML]{F2F2F2}15.7$_{\pm3.3}$ & \scriptsize \cellcolor[HTML]{F2F2F2}59.6$_{\pm3.4}$ & \scriptsize \cellcolor[HTML]{F2F2F2}82.4$_{\pm1.6}$& \scriptsize \cellcolor[HTML]{F2F2F2}51.2$_{\pm5.4}$ & \scriptsize \cellcolor[HTML]{F2F2F2}42.1$_{\pm4.3}$  \\
\scriptsize \CO, $\tau=5$         & \scriptsize \cellcolor[HTML]{F2F2F2}9.7$_{\pm1.3}$ & \scriptsize \cellcolor[HTML]{F2F2F2}13.2$_{\pm0.4}$& \scriptsize \cellcolor[HTML]{F2F2F2}6.2$_{\pm1.0}$ & \scriptsize 6.6$_{\pm0.4}$                         & \scriptsize \cellcolor[HTML]{F2F2F2}24.4$_{\pm3.2}$ & \scriptsize \cellcolor[HTML]{F2F2F2}35.4$_{\pm4.6}$& \scriptsize \cellcolor[HTML]{F2F2F2}14.8$_{\pm0.8}$ & \scriptsize \cellcolor[HTML]{F2F2F2}15.0$_{\pm0.4}$ & \scriptsize \cellcolor[HTML]{F2F2F2}58.2$_{\pm5.4}$ & \scriptsize \cellcolor[HTML]{F2F2F2}79.9$_{\pm2.1}$& \scriptsize \cellcolor[HTML]{F2F2F2}47.1$_{\pm3.9}$ & \scriptsize \cellcolor[HTML]{F2F2F2}39.8$_{\pm1.6}$  \\
\scriptsize \LS, $\tau=1$         & \scriptsize \cellcolor[HTML]{F2F2F2}5.9$_{\pm0.7}$ & \scriptsize \cellcolor[HTML]{F2F2F2}7.7$_{\pm0.3}$ & \scriptsize \cellcolor[HTML]{F2F2F2}4.7$_{\pm0.9}$ & \scriptsize \cellcolor[HTML]{F2F2F2}5.2$_{\pm0.2}$ & \scriptsize \cellcolor[HTML]{F2F2F2}11.1$_{\pm1.1}$ & \scriptsize \cellcolor[HTML]{F2F2F2}14.2$_{\pm1.0}$& \scriptsize \cellcolor[HTML]{F2F2F2}8.1$_{\pm0.5}$  & \scriptsize \cellcolor[HTML]{F2F2F2}7.9$_{\pm0.3}$  & \scriptsize \cellcolor[HTML]{F2F2F2}24.8$_{\pm0.8}$ & \scriptsize \cellcolor[HTML]{F2F2F2}32.8$_{\pm0.6}$& \scriptsize \cellcolor[HTML]{F2F2F2}22.8$_{\pm1.6}$ & \scriptsize \cellcolor[HTML]{F2F2F2}15.5$_{\pm0.9}$  \\
\scriptsize \LS, $\tau=3$         & \scriptsize \cellcolor[HTML]{F2F2F2}4.9$_{\pm0.3}$ & \scriptsize \cellcolor[HTML]{F2F2F2}6.4$_{\pm0.2}$ & \scriptsize \cellcolor[HTML]{F2F2F2}4.1$_{\pm0.7}$ & \scriptsize \cellcolor[HTML]{F2F2F2}5.2$_{\pm0.4}$ & \scriptsize \cellcolor[HTML]{F2F2F2}9.5$_{\pm0.3}$  & \scriptsize \cellcolor[HTML]{F2F2F2}11.7$_{\pm0.7}$& \scriptsize \cellcolor[HTML]{F2F2F2}7.3$_{\pm0.3}$  & \scriptsize \cellcolor[HTML]{F2F2F2}7.1$_{\pm0.3}$  & \scriptsize \cellcolor[HTML]{F2F2F2}22.0$_{\pm0.6}$ & \scriptsize \cellcolor[HTML]{F2F2F2}26.6$_{\pm1.0}$& \scriptsize \cellcolor[HTML]{F2F2F2}19.4$_{\pm0.8}$ & \scriptsize \cellcolor[HTML]{F2F2F2}14.3$_{\pm1.9}$  \\
\scriptsize \LS, $\tau=5$         & \scriptsize \cellcolor[HTML]{F2F2F2}5.1$_{\pm0.5}$ & \scriptsize \cellcolor[HTML]{F2F2F2}5.7$_{\pm0.5}$ & \scriptsize \cellcolor[HTML]{F2F2F2}3.9$_{\pm0.1}$ & \scriptsize \cellcolor[HTML]{F2F2F2}4.8$_{\pm0.0}$ & \scriptsize \cellcolor[HTML]{F2F2F2}8.6$_{\pm1.8}$  & \scriptsize \cellcolor[HTML]{F2F2F2}9.6$_{\pm1.2}$ & \scriptsize \cellcolor[HTML]{F2F2F2}6.5$_{\pm0.7}$ & \scriptsize \cellcolor[HTML]{F2F2F2}6.9$_{\pm0.9}$  & \scriptsize \cellcolor[HTML]{F2F2F2}18.1$_{\pm2.5}$ & \scriptsize \cellcolor[HTML]{F2F2F2}22.8$_{\pm0.2}$& \scriptsize \cellcolor[HTML]{F2F2F2}17.2$_{\pm2.2}$ & \scriptsize \cellcolor[HTML]{F2F2F2}14.1$_{\pm0.5}$  \\ \hline
\scriptsize \AT                & \scriptsize \cellcolor[HTML]{FCE4D6}11.3$_{\pm0.9}$& \scriptsize \cellcolor[HTML]{F2F2F2}15.7$_{\pm0.9}$& \scriptsize \cellcolor[HTML]{FCE4D6}11.4$_{\pm0.6}$& \scriptsize \cellcolor[HTML]{FCE4D6}11.1$_{\pm0.3}$& \scriptsize \cellcolor[HTML]{FCE4D6}50.3$_{\pm5.9}$ & \scriptsize \cellcolor[HTML]{FCE4D6}68.9$_{\pm2.5}$& \scriptsize \cellcolor[HTML]{FCE4D6}57.8$_{\pm5.0}$ & \scriptsize \cellcolor[HTML]{FCE4D6}51.6$_{\pm3.6}$ & \scriptsize \cellcolor[HTML]{F8CBAD}94.9$_{\pm0.7}$ & \scriptsize \cellcolor[HTML]{FCE4D6}97.1$_{\pm0.3}$& \scriptsize \cellcolor[HTML]{F4B084}95.5$_{\pm1.1}$ & \scriptsize \cellcolor[HTML]{F8CBAD}93.8$_{\pm0.8}$  \\ \hline
\scriptsize \IR                & \scriptsize \cellcolor[HTML]{FCE4D6}19.1$_{\pm6.1}$& \scriptsize \cellcolor[HTML]{FCE4D6}29.9$_{\pm6.3}$& \scriptsize \cellcolor[HTML]{FCE4D6}16.7$_{\pm5.1}$& \scriptsize \cellcolor[HTML]{FCE4D6}16.7$_{\pm4.7}$& \scriptsize \cellcolor[HTML]{FCE4D6}60.7$_{\pm14.7}$& \scriptsize \cellcolor[HTML]{FCE4D6}82.7$_{\pm9.5}$& \scriptsize \cellcolor[HTML]{FCE4D6}64.1$_{\pm16.7}$& \scriptsize \cellcolor[HTML]{FCE4D6}55.9$_{\pm13.5}$& \scriptsize \cellcolor[HTML]{FCE4D6}91.2$_{\pm7.6}$ & \scriptsize \cellcolor[HTML]{FCE4D6}96.2$_{\pm0.8}$& \scriptsize \cellcolor[HTML]{FCE4D6}92.0$_{\pm5.6}$ & \scriptsize \cellcolor[HTML]{FCE4D6}90.3$_{\pm6.5}$  \\
\scriptsize \JR                & \scriptsize \cellcolor[HTML]{FCE4D6}23.4$_{\pm2.2}$    & \scriptsize \cellcolor[HTML]{FCE4D6}37.3$_{\pm2.3}$    & \scriptsize \cellcolor[HTML]{FCE4D6}20.0$_{\pm1.8}$    & \scriptsize \cellcolor[HTML]{FCE4D6}19.1$_{\pm1.1}$    & \scriptsize \cellcolor[HTML]{FCE4D6}67.3$_{\pm3.5}$     & \scriptsize \cellcolor[HTML]{FCE4D6}88.2$_{\pm0.5}$    & \scriptsize \cellcolor[HTML]{FCE4D6}66.6$_{\pm1.8}$     & \scriptsize \cellcolor[HTML]{FCE4D6}57.7$_{\pm3.5}$     & \scriptsize \cellcolor[HTML]{FCE4D6}93.1$_{\pm1.5}$     & \scriptsize \cellcolor[HTML]{F8CBAD}97.4$_{\pm0.4}$    & \scriptsize \cellcolor[HTML]{FCE4D6}93.5$_{\pm0.3}$     & \scriptsize \cellcolor[HTML]{FCE4D6}91.9$_{\pm1.1}$      \\
\scriptsize \ER                & \scriptsize \cellcolor[HTML]{FCE4D6}12.1$_{\pm2.7}$& \scriptsize \cellcolor[HTML]{FCE4D6}19.6$_{\pm2.4}$& \scriptsize \cellcolor[HTML]{FCE4D6}7.5$_{\pm0.5}$ & \scriptsize \cellcolor[HTML]{FCE4D6}8.1$_{\pm1.5}$ & \scriptsize \cellcolor[HTML]{FCE4D6}32.5$_{\pm9.3}$ & \scriptsize \cellcolor[HTML]{FCE4D6}51.9$_{\pm9.5}$& \scriptsize \cellcolor[HTML]{FCE4D6}22.1$_{\pm4.5}$ & \scriptsize \cellcolor[HTML]{FCE4D6}21.5$_{\pm5.5}$ & \scriptsize \cellcolor[HTML]{FCE4D6}65.7$_{\pm17.9}$& \scriptsize \cellcolor[HTML]{FCE4D6}89.2$_{\pm5.2}$& \scriptsize \cellcolor[HTML]{FCE4D6}55.0$_{\pm8.4}$ & \scriptsize \cellcolor[HTML]{FCE4D6}50.1$_{\pm11.9}$ \\
\scriptsize \SAM               & \scriptsize \cellcolor[HTML]{FCE4D6}22.6$_{\pm2.6}$& \scriptsize \cellcolor[HTML]{FCE4D6}30.1$_{\pm0.5}$& \scriptsize \cellcolor[HTML]{FCE4D6}12.7$_{\pm0.5}$& \scriptsize \cellcolor[HTML]{FCE4D6}11.5$_{\pm0.1}$& \scriptsize \cellcolor[HTML]{FCE4D6}55.5$_{\pm1.9}$ & \scriptsize \cellcolor[HTML]{FCE4D6}72.5$_{\pm3.1}$& \scriptsize \cellcolor[HTML]{FCE4D6}39.4$_{\pm3.4}$ & \scriptsize \cellcolor[HTML]{FCE4D6}34.8$_{\pm3.0}$ & \scriptsize \cellcolor[HTML]{FCE4D6}89.8$_{\pm1.0}$ & \scriptsize \cellcolor[HTML]{FCE4D6}97.1$_{\pm0.7}$& \scriptsize \cellcolor[HTML]{FCE4D6}86.1$_{\pm10.3}$& \scriptsize \cellcolor[HTML]{FCE4D6}82.0$_{\pm13.8}$ \\ \hline
\scriptsize \SAM\&\IR           & \scriptsize \cellcolor[HTML]{F8CBAD}23.7$_{\pm4.1}$& \scriptsize \cellcolor[HTML]{F8CBAD}37.5$_{\pm2.3}$& \scriptsize \cellcolor[HTML]{F8CBAD}22.0$_{\pm2.8}$& \scriptsize \cellcolor[HTML]{F8CBAD}19.7$_{\pm1.5}$& \scriptsize \cellcolor[HTML]{F8CBAD}68.9$_{\pm8.3}$ & \scriptsize \cellcolor[HTML]{F8CBAD}91.1$_{\pm2.3}$& \scriptsize \cellcolor[HTML]{F4B084}73.5$_{\pm6.9}$ & \scriptsize \cellcolor[HTML]{F8CBAD}64.3$_{\pm7.5}$ & \scriptsize \cellcolor[HTML]{FCE4D6}94.6$_{\pm1.3}$ & \scriptsize \cellcolor[HTML]{FCE4D6}96.9$_{\pm0.5}$& \scriptsize \cellcolor[HTML]{F8CBAD}95.4$_{\pm1.2}$ & \scriptsize \cellcolor[HTML]{F4B084}94.1$_{\pm1.3}$  \\
\scriptsize \SAM\&\JR           & \scriptsize \cellcolor[HTML]{F4B084}32.7$_{\pm4.5}$& \scriptsize \cellcolor[HTML]{F4B084}47.3$_{\pm3.1}$& \scriptsize \cellcolor[HTML]{F4B084}27.1$_{\pm3.3}$& \scriptsize \cellcolor[HTML]{F4B084}24.9$_{\pm3.3}$& \scriptsize \cellcolor[HTML]{F4B084}74.3$_{\pm7.3}$ & \scriptsize \cellcolor[HTML]{F4B084}92.7$_{\pm1.5}$& \scriptsize \cellcolor[HTML]{F8CBAD}72.8$_{\pm7.0}$ & \scriptsize \cellcolor[HTML]{F4B084}65.3$_{\pm8.7}$ & \scriptsize \cellcolor[HTML]{F4B084}95.1$_{\pm3.1}$ & \scriptsize \cellcolor[HTML]{F4B084}98.4$_{\pm0.6}$& \scriptsize \cellcolor[HTML]{FCE4D6}95.0$_{\pm3.0}$ & \scriptsize \cellcolor[HTML]{FCE4D6}93.5$_{\pm2.7}$  \\ \hline
\end{tabular}
}
\label{tab:autoattack}
\end{table*}

\subsection{Proof of Transferability Lower Bound}
\vspace{-2mm}
Here we present the proof of \cref{theorem1}.
The following three lemmas are used in the proof.

\begin{lemma}\label{lemma1}
For arbitrary vector $\delta, x, y$, suppose $\|\delta\|_{2} \leq \epsilon, x$ and y are unit vectors, i.e., $\|x\|_{2}=$ $\|y\|_{2}=1$. Let $\cos \langle x, y\rangle=\frac{x \cdot y}{\|x\|_{2} \cdot\|y\|_{2}}$. Let $c$ denote any real number. Then

\begin{align}
\delta \cdot y<c-\epsilon \sqrt{2-2 \cos \langle x, y\rangle} \Rightarrow \delta \cdot x<c.
\end{align}

\end{lemma}
\begin{proof}
$\delta \cdot x=\delta \cdot y+\delta \cdot(x-y)<c-\epsilon \sqrt{2-2 \cos \langle x, y\rangle}+\delta \cdot(x-y)$. By law of cosines, $\delta \cdot(x-y) \leq \|\delta\|_2 \cdot \|(x-y)\|_2 \leq \epsilon \sqrt{2-2 \cos \langle x, y\rangle}$. Hence, $\delta \cdot x<c$
\end{proof}

\begin{lemma}\label{lemma2}
For arbitrary events $A$ and $B$, we have $\operatorname{Pr}(\neg A\cap \neg B)\ge 1-\operatorname{Pr}(A)-\operatorname{Pr}(B)$, where $\operatorname{Pr}(\cdot)$ denotes the probability of a event.
\end{lemma}
\begin{proof}
For events $A$ and $B$, we have $\operatorname{Pr}(A\cup B) + \operatorname{Pr}(\neg(A\cup B))=1$. 
And it's true that $\operatorname{Pr}(A) + \operatorname{Pr}(B) \ge \operatorname{Pr}(A\cup B)$ and $\neg (A\cup B)= \neg A \cap \neg B$. Then we have $\operatorname{Pr}(A) + \operatorname{Pr}(B) + \operatorname{Pr}(\neg A \cap \neg B ) \ge \operatorname{Pr}(A\cup 
 B) + \operatorname{Pr}(\neg(A\cup B))=1$, thus $\operatorname{Pr}(\neg A\cap \neg B)\ge 1-\operatorname{Pr}(A)-\operatorname{Pr}(B)$.
\end{proof}

\noindent Note that this is a generalizable property. Given any number of events $A_1, A_2, ..., A_n$, we can also have $\operatorname{Pr}(A_1)+\operatorname{Pr}(A_2)+...+ \operatorname{Pr}(A_n) + \operatorname{Pr}(\neg A_1 \cap \neg A_2 ... \cap \neg A_n)\ge 1$, thus  $\operatorname{Pr}(\neg A_1\cap \neg A_2 ... \neg A_n )\ge 1-\operatorname{Pr}(A_1)-\operatorname{Pr}(A_2)-...-\operatorname{Pr}(A_n)$.\\

\begin{lemma}\label{lemma3}
Let $Z$ be a random variable with support contained in $(-\infty, \epsilon]$. 
Let $c$ be a real constant such that $c < \epsilon$. 
The probability of the event $Z < c$ is upper-bounded by the expectation of the linear function $h(Z)$ as follows:
\[
\operatorname{Pr}(Z < c) \leq \mathbb{E}[h(Z)]
\]
where $h(Z) = \frac{\epsilon - Z}{\epsilon - c}$.
\end{lemma}

\begin{proof}
The probability of an event can be expressed as the expectation of its indicator function. Thus,
\[
\operatorname{Pr}(Z < c) =  \mathbb{E}\big[\mathbb{I}[Z < c]\big]
\]
where $\mathbb{I}[Z < c]$ is the indicator function that has the value of 1 if $Z < c$ and 0 otherwise. We now show that  $\mathbb{I}[Z < c] \leq h(Z)$ always holds.   
Consider two cases:

\begin{enumerate}
    \item \textbf{Case 1: $Z < c$}\\
    In this case, $\mathbb{I}\{Z < c\} = 1$. Since $Z < c$ and $c < \epsilon$, we have $\epsilon - Z > \epsilon - c > 0$. Thus, we have:
    \[ h(Z) = \frac{\epsilon - Z}{\epsilon - c} > 1 \]
    Thus, $\mathbb{I}_{Z < c} \leq h(Z)$ holds.

    \item \textbf{Case 2: $Z \geq c$}\\
    In this case, $\mathbb{I}\{Z < c\} = 0$. Since $Z \in (-\infty, \epsilon]$, so we have  $\epsilon - Z \geq 0$. The $\epsilon - c$ is positive. Therefore:
    \[ h(Z) = \frac{\epsilon - Z}{\epsilon - c} \geq 0 \]
    Thus, $\mathbb{I}[Z < c] \leq h(Z)$ holds.
\end{enumerate}

Since $\mathbb{I}\{Z < c\} \leq h(Z)$ holds pointwise for all values in the support of $Z$, we can apply the monotonicity property of expectation:
\[
\mathbb{E}[\mathbb{I}\{Z < c\}] \leq \mathbb{E}[h(Z)]
\]
Substituting back the definition of probability, we have:
\[
\operatorname{Pr}(Z < c) \leq \mathbb{E}[h(Z)]
\]
\end{proof}

\noindent \textbf{Proof for \cref{theorem1}}

The proof builds upon the derivations in \nips21, with the primary modifications focusing on the definition of smoothness. We omit some intermediate steps to save space.

\begin{proof}
Define auxiliary functions $f$, $g : \mathcal{M} \times \mathcal{L} \rightarrow \mathbb{R}$ such that
\begin{align}
f(x, y) &= \frac{\min_{y' \in \mathcal{L} : y' \neq y} \ell_{\mathcal{F}}(x + \delta, y') - \ell_{\mathcal{F}}(x, y) - \overline{\sigma}_\mathcal{F} \epsilon^{2} / 2}{\left\|\nabla_{x} \ell_{\mathcal{F}}(x, y)\right\|_{2}}, \nonumber \\
g(x, y) &= \frac{\min_{y' \in \mathcal{L} : y' \neq y} \ell_{\mathcal{G}}(x + \delta, y') - \ell_{\mathcal{G}}(x, y) + \overline{\sigma}_\mathcal{G} \epsilon^{2} / 2}{\left\|\nabla_{x} \ell_{\mathcal{G}}(x, y)\right\|_{2}}. \label{eq:f&g-definition}
\end{align}
The functions $f$ and $g$ are orthogonal to the confidence score functions of models $\mathcal{F}$ and $\mathcal{G}$, respectively. We reuse the notation $f$ here in the Appendix.
Note that $c_{\mathcal{F}} = \min_{(x, y) \sim \mathcal{D}} f(x, y)$, $c_{\mathcal{G}} = \max_{(x, y) \sim \mathcal{D}} g(x, y)$.

Observing \cref{def:transferability}, the transferability concerns a successful transfer requiring four events, i.e., $\mathcal{F}(x) = y$, $\mathcal{G}(x) = y$, $\mathcal{F}(x + \delta) \neq y$, and $\mathcal{G}(x + \delta) \neq y$. 
In other words, both $\mathcal{F}$ and $\mathcal{G}$ give the correct prediction for $x$, and the wrong prediction for $x + \delta$. We have:
\begin{align}
&\operatorname{Pr}\left(T_{r}(\mathcal{F}, \mathcal{G}, x, y) = 1\right) \nonumber \\
&= \operatorname{Pr}\left(\mathcal{F}(x) = y \cap \mathcal{G}(x) = y \cap \mathcal{F}(x + \delta) \neq y \cap \mathcal{G}(x + \delta) \neq y\right) \nonumber \\
&\geq 1 - \operatorname{Pr}(\mathcal{F}(x) \neq y) - \operatorname{Pr}(\mathcal{G}(x) \neq y) - \operatorname{Pr}\left(\mathcal{F}(x + \delta) = y\right) \nonumber \\
&\quad - \operatorname{Pr}\left(\mathcal{G}(x + \delta) = y\right) \nonumber \\
&\geq 1 - \gamma_{\mathcal{F}} - \gamma_{\mathcal{G}} - \alpha - \operatorname{Pr}\left(\mathcal{G}(x + \delta) = y\right). \label{eq:transfer-bound}
\end{align}
We explain \cref{eq:transfer-bound} as follows:
\begin{itemize}
    \item In the first inequality of \cref{eq:transfer-bound},  we used the general version for four events of \cref{lemma2}.
Note that $\mathcal{F}(x) \neq y$, $\mathcal{G}(x) \neq y$, $\mathcal{F}(x + \delta) = y$, and $\mathcal{G}(x + \delta) = y$ are the corresponding complementary events of $\mathcal{F}(x) = y$, $\mathcal{G}(x) = y$, $\mathcal{F}(x + \delta) \neq y$, and $\mathcal{G}(x + \delta) \neq y$.

\item In the second inequality of \cref{eq:transfer-bound}, we use the definitions of the model risks $\gamma_{\mathcal{F}}$, $\gamma_{\mathcal{G}}$, and $\alpha$ ($\operatorname{Pr}\left(\mathcal{F}(x + \delta) = y\right) < \alpha$ holds, see \cref{theorem1}).

\end{itemize}

Now consider the probability of two events, $\mathcal{F}(x + \delta$, and $\mathcal{G}(x + \delta) = y $, $\operatorname{Pr}\left(\mathcal{F}(x + \delta) = y\right)$ and $\operatorname{Pr}\left(\mathcal{G}(x + \delta) = y\right)$. 
Given that a model predicts the label for which the training loss is minimized, we have
\begin{equation}
      \mathcal{F}(x + \delta) = y \Longleftrightarrow \ell_{\mathcal{F}}(x + \delta, y) < \min_{y' \in \mathcal{L} : y' \neq y} \ell_{\mathcal{F}}(x + \delta, y'). \label{eq:f_loss_def}
\end{equation}
Similarly, we also have  
\begin{equation}
\mathcal{G}(x + \delta) = y \Longleftrightarrow \ell_{\mathcal{G}}(x + \delta, y) < \min_{y' \in \mathcal{L} : y' \neq y} \ell_{\mathcal{G}}(x + \delta, y').
\label{eq:g_loss_def}
\end{equation}
Following Taylor's Theorem with Lagrange remainder, we have:
\begin{align}
\ell_{\mathcal{F}}(x + \delta, y) &= \ell_{\mathcal{F}}(x, y) + \delta \nabla_{x} \ell_{\mathcal{F}}(x, y) + \frac{1}{2} \xi^{\top} \mathbf{H}_{\mathcal{F}} \xi, \label{eq:Taylor1} \\
\ell_{\mathcal{G}}(x + \delta, y) &= \ell_{\mathcal{G}}(x, y) + \delta \nabla_{x} \ell_{\mathcal{G}}(x, y) + \frac{1}{2} \xi^{\top} \mathbf{H}_{\mathcal{G}} \xi. \label{eq:Taylor2}
\end{align}
In \cref{eq:Taylor1} and \cref{eq:Taylor2}, $\|\xi\| < \|\delta\|$. $\mathbf{H}_{\mathcal{F}}$ and $\mathbf{H}_{\mathcal{G}}$ are Hessian matrices of $\ell_{\mathcal{F}}$ and $\ell_{\mathcal{G}}$, respectively.

Since $\ell_{\mathcal{F}}(x + \delta, y)$ and $\ell_{\mathcal{G}}(x + \delta, y)$ have upper smoothness $\overline{\sigma}_\mathcal{F}$ and $\overline{\sigma}_\mathcal{G}$, the maximum eigenvalues of $\mathbf{H}_{\mathcal{F}}$ and $\mathbf{H}_{\mathcal{G}}$ are bounded. As a result, $\left|\xi^{\top} \mathbf{H}_{\mathcal{F}} \xi\right| \leq \overline{\sigma}_{\mathcal{F}} \cdot \|\xi\|_{2}^{2} \leq \overline{\sigma}_{\mathcal{F}} \epsilon^{2}$. 

\noindent Applying this to \cref{eq:Taylor1} and \cref{eq:Taylor2}, we thus have:
{\small
\begin{align}
\ell_{\mathcal{F}} + \delta \nabla_{x} \ell_{\mathcal{F}} - \frac{1}{2} \overline{\sigma}_\mathcal{F} \epsilon^{2} &\leq \ell_{\mathcal{F}}(x + \delta, y) \leq \ell_{\mathcal{F}} + \delta \nabla_{x} \ell_{\mathcal{F}} + \frac{1}{2} \overline{\sigma}_\mathcal{F} \epsilon^{2}, \label{eq:Taylor1-app} \\ 
\ell_{\mathcal{G}} + \delta \nabla_{x} \ell_{\mathcal{G}} - \frac{1}{2} \overline{\sigma}_\mathcal{G} \epsilon^{2} &\leq \ell_{\mathcal{G}}(x + \delta, y) \leq \ell_{\mathcal{G}} + \delta \nabla_{x} \ell_{\mathcal{G}} + \frac{1}{2} \overline{\sigma}_\mathcal{G} \epsilon^{2}. \label{eq:Taylor2-app}
\end{align}
}\normalsize
Here, we have abbreviated the notations $\nabla_{x} \ell_{\mathcal{G}}(x, y)$ and $\nabla_{x} \ell_{\mathcal{F}}(x, y)$ as $\nabla_{x} \ell_{\mathcal{G}}$ and $\nabla_{x} \ell_{\mathcal{F}}$, respectively, to save space, as often needed hereinafter.\\

Apply the right-hand side of \cref{eq:Taylor1-app}  and use the \cref{eq:f_loss_def}:
\begin{align}
&\operatorname{Pr}\left(\mathcal{F}(x + \delta) = y\right) \nonumber \\
&= \operatorname{Pr}\left(\ell_{\mathcal{F}}(x + \delta, y) < \min_{y' \in \mathcal{L} : y' \neq y} \ell_{\mathcal{F}}(x + \delta, y')\right) \nonumber \\
&\geq \operatorname{Pr}\left(\ell_{\mathcal{F}} + \delta \nabla_{x} \ell_{\mathcal{F}} + \frac{1}{2} \overline{\sigma}_\mathcal{F} \epsilon^{2} < \min_{y' \in \mathcal{L} : y' \neq y} \ell_{\mathcal{F}}(x + \delta, y')\right) \nonumber \\
&= \operatorname{Pr}\left(\delta \cdot \frac{\nabla_{x} \ell_{\mathcal{F}}}{\left\|\nabla_{x} \ell_{\mathcal{F}}\right\|_{2}} < f(x, y)\right), \nonumber \\
&\Rightarrow \operatorname{Pr}\left(\delta \cdot \frac{\nabla_{x} \ell_{\mathcal{F}}}{\left\|\nabla_{x} \ell_{\mathcal{F}}\right\|_{2}} < f(x, y)\right) \leq \alpha. \label{eq:F-leq-alpha}
\end{align}

Apply the left-hand side of \cref{eq:Taylor2-app} and use \cref{eq:g_loss_def}:
\begin{align}
&\operatorname{Pr}\left(\mathcal{G}(x + \delta) = y\right) \nonumber \\
&= \operatorname{Pr}\left(\ell_{\mathcal{G}}(x + \delta, y) < \min_{y' \in \mathcal{L} : y' \neq y} \ell_{\mathcal{G}}(x + \delta, y')\right) \nonumber \\
&\leq \operatorname{Pr}\left(\ell_{\mathcal{G}} + \delta \nabla_{x} \ell_{\mathcal{G}} - \frac{1}{2} \overline{\sigma}_\mathcal{G} \epsilon^{2} < \min_{y' \in \mathcal{L} : y' \neq y} \ell_{\mathcal{G}}(x + \delta, y')\right) \nonumber \\
&= \operatorname{Pr}\left(\delta \cdot \frac{\nabla_{x} \ell_{\mathcal{G}}}{\left\|\nabla_{x} \ell_{\mathcal{G}}\right\|_{2}} < g(x, y)\right) \nonumber \\
&\leq \operatorname{Pr}\left(\delta \cdot \frac{\nabla_{x} \ell_{\mathcal{G}}}{\left\|\nabla_{x} \ell_{\mathcal{G}}\right\|_{2}} < c_{\mathcal{G}}\right). \label{eq:G-upper}
\end{align}
\cref{eq:G-upper} used the definition of $g$  and the definition of $c_{\mathcal{G}}$.

Since $\|\delta\|_{2} \leq \epsilon$, from \cref{lemma1}, we can deduce:
\begin{align}
&\delta \cdot \frac{\nabla_{x} \ell_{\mathcal{G}}}{\left\|\nabla_{x} \ell_{\mathcal{G}}\right\|_{2}} < f(x, y) - \epsilon \sqrt{2 - 2 \underline{\mathcal{S}}_\mathcal{D}\left(\ell_{\mathcal{F}}, \ell_{\mathcal{G}}\right)}, \label{eq:S_F&G1} \\
&\Rightarrow \delta \cdot \frac{\nabla_{x} \ell_{\mathcal{G}}}{\left\|\nabla_{x} \ell_{\mathcal{G}}\right\|_{2}} < f(x, y) - \epsilon \sqrt{2 - 2 \cos \left\langle \nabla_{x} \ell_{\mathcal{F}}, \nabla_{x} \ell_{\mathcal{G}} \right\rangle}, \label{eq:S_F&G2} \\
&\Rightarrow \delta \cdot \frac{\nabla_{x} \ell_{\mathcal{F}}}{\left\|\nabla_{x} \ell_{\mathcal{F}}\right\|_{2}} < f(x, y). \label{eq:S_F&G3}
\end{align}
We explain the deduction as follows:
\begin{itemize}
    \item From \cref{eq:S_F&G1} to \cref{eq:S_F&G2}, we used the definition of $\underline{\mathcal{S}}_\mathcal{D}$, since it indicates that $\underline{\mathcal{S}}_\mathcal{D}\left(\ell_{\mathcal{F}}, \ell_{\mathcal{G}}\right) \leq \cos \left\langle \nabla_{x} \ell_{\mathcal{F}}(x, y), \nabla_{x} \ell_{\mathcal{G}}(x, y) \right\rangle$.
    \item From \cref{eq:S_F&G2} to \cref{eq:S_F&G3}, we applied \cref{lemma1}.
\end{itemize}

 As a result, using \cref{eq:F-leq-alpha}, we have:
\begin{align*}
&\operatorname{Pr}\left(\delta \cdot \frac{\nabla_{x} \ell_{\mathcal{G}}}{\left\|\nabla_{x} \ell_{\mathcal{G}}\right\|_{2}} < f(x, y) - \epsilon \sqrt{2 - 2 \underline{\mathcal{S}}_\mathcal{D}\left(\ell_{\mathcal{F}}, \ell_{\mathcal{G}}\right)}\right) \\
&\leq \operatorname{Pr}\left(\delta \cdot \frac{\nabla_{x} \ell_{\mathcal{F}}}{\left\|\nabla_{x} \ell_{\mathcal{F}}\right\|_{2}} < f(x, y)\right) \leq \alpha.
\end{align*}
Since $c_\mathcal{F} \leq f(x, y)$, we also have:
\begin{equation}
    \operatorname{Pr}\left(\delta \cdot \frac{\nabla_{x} \ell_{\mathcal{G}}}{\left\|\nabla_{x} \ell_{\mathcal{G}}\right\|_{2}} < c_{\mathcal{F}} - \epsilon \sqrt{2 - 2 \underline{\mathcal{S}}_\mathcal{D}\left(\ell_{\mathcal{F}}, \ell_{\mathcal{G}}\right)}\right) \leq \alpha. \label{eq:Z_lower}
\end{equation}

We now denote $Z = \delta \cdot \frac{\nabla_{x} \ell_{\mathcal{G}}}{\left\|\nabla_{x} \ell_{\mathcal{G}}\right\|_{2}}$, so \cref{eq:Z_lower} can be re-written as 
\begin{equation}
    \operatorname{Pr}\left(Z < c_{\mathcal{F}} - \epsilon \sqrt{2 - 2 \underline{\mathcal{S}}_\mathcal{D}\left(\ell_{\mathcal{F}}, \ell_{\mathcal{G}}\right)}\right) \leq \alpha. \label{eq:Z_lower_new}
\end{equation}

Since $\|\delta\|_2 \leq \epsilon$, we have $Z \in [-\epsilon, \epsilon]$,  meaning $Z$ has a lower-bound of $-\epsilon$. 
Thus, using \cref{eq:Z_lower_new}, the expectation of $Z$ can be lower-bounded:
\begin{equation}
\mathbb{E}\left[Z\right] \geq -\epsilon \alpha + \left(c_{\mathcal{F}} - \epsilon \sqrt{2 - 2 \underline{\mathcal{S}}_\mathcal{D}\left(\ell_{\mathcal{F}}, \ell_{\mathcal{G}}\right)}\right)(1 - \alpha). \label{eq:expectation}
\end{equation}

On the other hand, since $Z \in [-\epsilon, \epsilon] \subset (-\infty, \epsilon]$ and $c_{\mathcal{G}} < \epsilon$, so we can apply \cref{lemma3} to obtain:
\[
\operatorname{Pr}(Z < c_{\mathcal{G}}) \leq \mathbb{E}[h(Z)] = \mathbb{E}\left[ \frac{\epsilon - Z}{\epsilon - c_{\mathcal{G}}} \right] = \frac{\epsilon - \mathbb{E}[Z]}{\epsilon - c_{\mathcal{G}}}.
\]
Using the lower bound on $\mathbb{E}[Z]$ from \eqref{eq:expectation},
we have:
{\small
\begin{align*}
\operatorname{Pr}(Z < c_{\mathcal{G}}) &\leq \frac{\epsilon - \left[ -\epsilon \alpha + \left( c_{\mathcal{F}} - \epsilon \sqrt{2 - 2 \underline{\mathcal{S}}_\mathcal{D}\left(\ell_{\mathcal{F}}, \ell_{\mathcal{G}}\right)} \right) (1 - \alpha) \right]}{\epsilon - c_{\mathcal{G}}} \\
&= \frac{\epsilon (1 + \alpha) - \left( c_{\mathcal{F}} - \epsilon \sqrt{2 - 2 \underline{\mathcal{S}}_\mathcal{D}\left(\ell_{\mathcal{F}}, \ell_{\mathcal{G}}\right)} \right) (1 - \alpha)}{\epsilon - c_{\mathcal{G}}}.
\end{align*}
}\normalsize
Thus, from \eqref{eq:G-upper}, we have:

\begin{equation}
\begin{aligned}
&\operatorname{Pr}\left(\mathcal{G}(x + \delta) = y\right) \\
&\leq \operatorname{Pr}(Z < c_{\mathcal{G}}) \\
&\leq \frac{\epsilon (1 + \alpha) - \left( c_{\mathcal{F}} - \epsilon \sqrt{2 - 2 \underline{\mathcal{S}}_\mathcal{D}\left(\ell_{\mathcal{F}}, \ell_{\mathcal{G}}\right)} \right) (1 - \alpha)}{\epsilon - c_{\mathcal{G}}}.
\end{aligned} \label{eq:bound-classification}
\end{equation}
We conclude our proof by plugging \eqref{eq:bound-classification} into \eqref{eq:transfer-bound}.
\end{proof}

\subsection{Proof of connection between \IR~and \JR}
\begin{proposition} \label{proposition1}
Let a neural network parameterized by $\theta$, and $f_\theta$ represents its logit network. Given a sample $(x, y)$, if $\|\nabla_x f_\theta(x)\|_F \rightarrow 0$, $\|\nabla_x \ell(f_\theta(x),y)\| \rightarrow 0$, where $\ell$ denotes the cross-entropy loss function.
\end{proposition}
\begin{proof}
	Define the logit output of $f_\theta$ given input feature $x$ with respect to each class as $z_1, z_2,...,z_k$, including the according logit of target class $z_y$. Accordingly, $\nabla_x f_\theta(x) = (z_1, z_2,...,z_k)^T$. Thus, we have
\begin{align}
  \|\nabla_x f_\theta(x)\|_F &= \sum_i^k \|\nabla_x z_i\| \label{eq:F-norm}
\end{align} 
The cross-entropy loss of $(x,y)$ is computed as follows:
\begin{align}
\ell(f_\theta(x),y)=-\log (\frac{e^{z_{y}}}{\sum_i^k e^{z_{i}}}) = -z_{y}+\log (\sum_i^k e^{z_{i}}).\nonumber
\end{align}
As a result, the gradient \wrt $x$ $\nabla_x \ell(f_\theta(x),y)$ is
\begin{align}
\nabla_x \ell(f_\theta(x),y) = - \nabla_x z_y + \frac{\sum_i^k e^{z_i} \nabla_x z_i }{\sum_i^k e^{z_i}} \nonumber
\end{align}
Thus, 
\begin{align}
  \|\nabla_x \ell(f_\theta(x),y)\| &=  \|\frac{1}{\sum_i^k e^{z_i}} \cdot \sum_i^k e^{z_i} \nabla_x z_i  - \nabla_x z_y\|  \nonumber \\
  								   & \leq \frac{1}{\sum_i^k e^{z_i}} \cdot \|\sum_i^k e^{z_i} \nabla_x z_i\| + \|\nabla_x z_y\|  \nonumber \\
  								   & \leq \frac{1}{\sum_i^k e^{z_i}} \cdot \sum_i^k e^{z_i} \|\nabla_x z_i\| + \|\nabla_x z_y\|  \nonumber \\
  								   & < \sum_i^k \|\nabla_x z_i\| + \|\nabla_x z_y\| \label{eq:cross-entropy-norm}
\end{align}
Note that if $\sum_i^k \|\nabla_x z_i\|  \rightarrow 0, \|\nabla_x z_y\|\rightarrow 0$.
Observing \cref{eq:F-norm,eq:cross-entropy-norm}, we can conclude that if $\|\nabla_x f_\theta(x)\|_F \rightarrow 0$, $\|\nabla_x \ell(f_\theta(x),y)\| \rightarrow 0$.
\end{proof}

\end{document}